\documentclass{article}

\usepackage[preprint]{neurips_2024}

\usepackage[utf8]{inputenc} 
\usepackage[T1]{fontenc}    
\usepackage{hyperref}       
\usepackage{url}            
\usepackage{booktabs}       
\usepackage{amsfonts}       
\usepackage{nicefrac}       
\usepackage{microtype}      
\usepackage{xcolor}         

\usepackage{subfiles}
\usepackage{enumitem}

\usepackage{algorithm}
\usepackage{algorithmic}
\usepackage{multicol}
\usepackage{tikz}

\usepackage{booktabs} 
\usepackage{multirow} 
\usepackage{makecell}
\usepackage{xspace}

\usepackage{color}
\definecolor{aqua}{rgb}{0.0, 1.0, 1.0}
\definecolor{aquamarine}{rgb}{0.5, 1.0, 0.83}
\definecolor{alizarin}{rgb}{0.82, 0.1, 0.26}
\definecolor{carmine}{rgb}{0.59, 0.0, 0.09}
\definecolor{tred}{rgb}{0.8, 0.0, 0.13}
\definecolor{blue(pigment)}{rgb}{0.2, 0.2, 0.6}
\definecolor{darkgreen}{rgb}{0.0, 0.39, 0.0}
\definecolor{Crimson}{HTML}{DC143C}
\definecolor{Maroon}{HTML}{800000}
\definecolor{MidnightBlue}{HTML}{191970}
\definecolor{Peru}{HTML}{CD853F}
\definecolor{Teal}{HTML}{008080}
\definecolor{FloralWhite}{HTML}{FFFAF0}
\definecolor{LightYellow}{HTML}{FFFFE0}
\definecolor{Burgundy}{HTML}{9C001A}

\hypersetup
{
	colorlinks=true,
	linkcolor=blue(pigment),
	filecolor=magenta,
	urlcolor=Burgundy,
	citecolor=darkgreen,
}

\usepackage[toc,page,header]{appendix}
\usepackage{minitoc}

\usepackage{amsthm, amsmath, amssymb}
\usepackage[capitalize]{cleveref}

\newtheorem{theorem}{Theorem}
\newtheorem{corollary}[theorem]{Corollary}
\newtheorem{lemma}[theorem]{Lemma}
\newtheorem{proposition}[theorem]{Proposition}

\theoremstyle{definition}
\newtheorem{definition}{Definition}
\newtheorem{assumption}{Assumption}

\usepackage{dsfont}
\usepackage{mathrsfs, mathtools}
\usepackage{txfonts}

\newcommand{\N}{\mathbf N}

\newcommand{\R}{\mathbf R}

\newcommand{\E}{\mathbf E}
\renewcommand{\Pr}{\mathbf P}
\newcommand{\Var}{\mathbf{V}}
\def\sp#1{{\mathrm{sp}\parens{#1}}}

\newcommand{\set}[1]{\left\{#1\right\}}
\newcommand{\abs}[1]{\left|#1\right|}
\newcommand{\norm}[1]{\left\|#1\right\|}
\newcommand{\parens}[1]{\left(#1\right)}
\newcommand{\brackets}[1]{\left[#1\right]}

\newcommand{\indicator}[1]{\mathbf{1}\parens{#1}}

\def\toto{\leftrightarrow}

\DeclareMathOperator{\Reg}{{\normalfont Reg}}
\DeclareMathOperator{\Fix}{{\normalfont Fix}}
\DeclareMathOperator{\KL}{{\normalfont\rm KL}}

\DeclareMathOperator*{\oh}{{\normalfont\rm o}}
\DeclareMathOperator*{\OH}{{\normalfont\rm O}}

\usepackage{pgffor}

\def\genUpperShortcuts#1#2{
    \foreach \x in {A,...,Z}{%
        \expandafter\xdef \csname\x#1\endcsname{\noexpand\ensuremath{\noexpand#2{\x}}}
    }
}

\def\genLowerShortcuts#1#2{
    \foreach \x in {a,...,z}{%
        \expandafter\xdef \csname\x#1\endcsname{\noexpand\ensuremath{\noexpand#2{\x}}}
    }
}

\genUpperShortcuts{c}{\mathcal}
\genUpperShortcuts{bb}{\mathbb}
\genUpperShortcuts{bf}{\mathbf}
\genUpperShortcuts{rm}{\mathrm}
\genUpperShortcuts{fk}{\mathfrak}
\genLowerShortcuts{bf}{\mathbf}

\title{Achieving Tractable Minimax Optimal Regret in Average Reward MDPs}

\author{%
  Victor Boone \\
  \texttt{victor.boone@univ-grenoble-alpes.fr} \\
  Univ.~Grenoble Alpes, Inria, CNRS, Grenoble INP, LIG, 38000 Grenoble, France
  \And
  Zihan Zhang \\
  \texttt{zz5478@princeton.edu} \\
  Princeton University
}

\newcommand{\ALGORITHM}{\texttt{PMEVI-DT}\xspace}
\def\sumnl{\sum\nolimits}

\def\STEP#1{(\textbf{STEP~#1})}

\begin{document}
\doparttoc 
\faketableofcontents 

\maketitle

\begin{abstract}
In recent years, significant attention has been directed towards learning average-reward Markov Decision Processes (MDPs).
However, existing algorithms either suffer from sub-optimal regret guarantees or computational inefficiencies.
In this paper, we present the first \emph{tractable} algorithm  with minimax optimal regret of $\widetilde{\OH}\left(\!\!\sqrt{\sp{h^*} S A T}\right)$,\footnote{$\widetilde{\OH}(\cdot)$ hides logarithmic factors of $(S,A,T)$.} where $\sp{h^*}$ is the span of the optimal bias function $h^*$, $S\times A$  is the size of the state-action space and $T$ the number of learning steps. 
Remarkably, our algorithm does not require prior information on $\sp{h^*}$. 

Our algorithm relies on a novel subroutine, \textbf{P}rojected \textbf{M}itigated \textbf{E}xtended \textbf{V}alue \textbf{I}teration (\texttt{PMEVI}), to compute bias-constrained optimal policies efficiently. 
This subroutine can be applied to various previous algorithms 
to improve regret bounds. 


    

\end{abstract}

\section{Introduction}\label{sec:intro}
Reinforcement learning (RL) \cite{burnetas_optimal_1997,sutton2018reinforcement} has become a popular approach for solving complex sequential decision-making tasks and has recently achieved notable advancements in diverse fields of application.
The RL problem is generally formulated as a Markov Decision Process (MDP) \cite{puterman_markov_1994}, where the agent interacts with an unknown environment to maximize its accumulative rewards. 

In this paper, we consider the problem of learning average-reward MDPs, where the central task is to balance between \emph{exploration} (i.e., learning the unknown environment) and \emph{exploitation} (i.e., planning optimally according to current knowledge) along the infinite-horizon learning process. 
One way to measure the performance of the learner is the regret, that compares the gathered rewards of the learner, unaware of the exact structure of its environment, to the expected performance of an omniscient agent that knows the environment in advance.
The seminal work of \cite{auer_near-optimal_2009} provides a \emph{minimax} regret lower bound $\Omega\parens{\!\!\sqrt{DSAT}}$, where $D$ is the diameter (the maximal distance between two different states), $S$ the number of states, $A$ the number of actions and $T$ the learning horizon.
They also provide an algorithm achieving regret $\widetilde{\OH}\parens{\!\!\sqrt{D^2S^2 A T}}$.
Ever since \cite{auer_near-optimal_2009}, many works have been devoted to close the gap between the regret lower and upper bounds in the average reward setting \cite{auer_near-optimal_2009,bartlett_regal_2009,filippi_optimism_2010,talebi_variance-aware_2018,fruit_efcient_2018,fruit_improved_2020,bourel_tightening_2020,zhang_regret_2019,ouyang_learning_2017,agrawal_optimistic_2023,abbasi2019exploration,wei_model-free_2020} and more.
Subsequent works \cite{fruit_efcient_2018,zhang_regret_2019} refined the minimax regret lower bound to $\Omega\parens{\!\!\sqrt{\sp{h^*}SAT}}$ where $\sp{h^*}$ is the span of the bias function, which is the maximal gap of the long-term accumulative rewards starting from two different states. 
The difference is significant, since $\sp{h^*} \le D$ and the gap between the two can be arbitrarly large.
However, no existing work achieves the following three requirements simultaneously:
\begin{itemize}[itemsep=-.15em]
    \item[(1)] The method achieves minimax optimal regret guarantees $\widetilde{\OH}\parens{\!\!\sqrt{\sp{h^*} S A T}}$;
    \item[(2)] The proposed method is tractable;
    \item[(3)] No prior knowledge on the model is required.
\end{itemize}
Most algorithms simply fail to achieve minimax optimal regret, and the only method achieving it \cite{zhang_regret_2019} is intractable because it relies an oracle to solve difficult optimization problems along the learning process.
\ifx\FALSE
Following the framework of  \emph{optimism-in-the-face-of-uncertainty} (OFU), a series of papers \cite{auer_near-optimal_2009,bartlett_regal_2009, maillard2011finite,fruit_efcient_2018,zhang_regret_2019} have established sub-linear regret bounds, which depend on either the diameter $D$, which describes the maximal distance between two different states, or the bias span $\sp{h^*}$, which is the maximal gap of the long-term accumulative reward starting from two different states. 
Given an MDP with $S$ states, $A$ actions, diameter $D$ and span $\sp{h^*}$ and total number of steps $T$, \cite{auer_near-optimal_2009} presented a minimax regret lower bound of $\Omega\left(\!\!\sqrt{DSAT}\right)$,\footnote{This also implies the lower bound $\Omega\left( \sqrt{\sp{h^*}SAT}\right)$ since $D\geq \sp{h^*}$.} alongside a regret upper bound of $\widetilde{\OH}\parens{\!\!\sqrt{D^2S^2AT}}$. Although the diameter $D$ is always greater than the bias span $\sp{h^*}$, it could potentially be much larger and there exist weakly-communicating MDPs with a fixed span but arbitrarily large diameter. As a result, regret bounds that depend on $\sp{h^*}$ are significantly superior to those that depend on $D$. Considerable effort has been devoted to discovering  span-dependent regret bounds. 
This direction was initiated by \cite{bartlett_regal_2009} with a regret bound of $\widetilde{\OH}\parens{\!\!\sqrt{(\sp{h^*})^2S^2AT}}$. Subsequently, 
\cite{zhang_regret_2019} achieved a nearly optimal regret bound of $\widetilde{\OH}(\!\!\sqrt{\sp{h^*}SAT})$. However, both methods are impractical due to exponential computational costs or mere intractability.
The work of \cite{bartlett_regal_2009} was made computationally efficient by \cite{fruit_efcient_2018}, albeit the regret bound remains sub-optimal with $\widetilde{\OH}\parens{\sp{h^*}S\!\!\sqrt{AT}}$. 
It's crucial to highlight that all three of these algorithms do not work unless $\sp{h^*}$ is known in advance, which may restrict their practicality across various applications.
\fi
Naturally, we raise the question of whether these three requirements can be met all at once:
\begin{center}
\textbf{\emph{Is there a tractable algorithm with $\widetilde{\OH}\parens{\!\!\sqrt{\sp{h^*}SAT}}$ minimax regret without prior knowledge?}}
\end{center}

\ifx\FALSE
One major technical challenge is finding the optimal policy under bias constraints efficiently. Following the framework to estimate the bias difference \cite{zhang_regret_2019}, the exponential computational cost is due to the following optimization problem: 
\begin{equation}
\label{equation:maximize-gain}
    \max_{\pi} g(\pi) 
    \quad \mathrm{s.t.} \quad
    \sp{h(\pi)}\in \mathcal{H}
    ,
\end{equation}
where $g(\pi)$ is the averaged-reward of $\pi$, $h(\pi)$ is the bias vector of $\pi$, and $\mathcal{H}$ is a set of bias constraints with form $h_{s}-h_{s'}\leq f_{s,s'}, \forall (s,s')$.  In the case $f_{s,s'}= c$ for all $(s,s')$ pairs and some fixed $c>0$,  \cite{fruit_efcient_2018}  discovered an efficient solution through truncated value iteration. Unfortunately, this approach isn't easily adaptable to general  $\{f_{s,s'}\}_{(s,s')\in \mathcal{S}\times \mathcal{S}}$ since the truncation operation might violate the bias constraints.
\fi



\paragraph{Contributions.} 
In this paper, we answer the above question affirmatively, by proposing a polynomial time algorithm with regret guarantees $\widetilde{\OH}\parens{\!\!\sqrt{\sp{h^*} S A T}}$ for average-reward MDPs.
Our method can further incorporate almost arbitrary prior bias information $\Hc_* \subseteq \R^\Sc$ to improve its regret.

\begin{theorem}[Informal]
\label{thm:main}
For any $c>0$,    provided that the confidence region used by \texttt{\upshape PMEVI-DT} satisfy mild regularity conditions (see Assumption~\ref{assumption:model-confidence-region}-\ref{assumption:evi-convergence}), if $T \ge c^5$, then for every weakly communicating model with bias span less than $c$ and with bias vector within $\Hc_*$, \texttt{\upshape PMEVI-DT}$(\Hc_*, T)$ achieves regret:
    \begin{equation*}
        \OH\parens{
            \!\!
            \sqrt{
                c S A T 
                \log\parens{\tfrac {SAT}\delta}
            }
        }
        + \OH\parens{
            c S^{\frac 52} A^{\frac 32} T^{\frac 9{20}} \log^2\parens{\tfrac{SAT}\delta}
        }
    \end{equation*}
    in expectation and with high probability.
    Moreover, if \texttt{\upshape PMEVI-DT} runs with the same confidence regions that \texttt{\upshape UCRL2} \cite{auer_near-optimal_2009}, then it enjoys a time complexity $\OH(D S^3 A T)$.
\end{theorem}

The geometry of the prior bias region $\Hc_*$ is discussed later (see \cref{assumption:bias-region}).
It can be taken trivial with $\Hc_* = \R^\Sc$ to obtain a completely prior-less algorithm.
To the best of our knowledge, this is the first tractable algorithm with minimax optimal regret bounds (up to logarithmic factors). The algorithm does not necessitate any prior knowledge of $\sp{h^*}$, thus circumventing the potentially high cost associated with learning $\sp{h^*}$.
On the technical side, a key novelty of our method is the subroutine named \texttt{PMEVI} (see \hyperlink{algorithm:PMEVI}{Algorithm~2}) that improves and can replace \texttt{EVI} \cite{auer_near-optimal_2009} in any algorithm that relies on it \cite{auer_near-optimal_2009,fruit_efcient_2018,filippi_optimism_2010,fruit_improved_2020,bourel_tightening_2020} to boost its performance and achieve minimax optimal regret.



\paragraph{Related works.}

Here is a short overview of the learning theory of average reward MDPs.
For communicating MDPs, the notable work of \cite{auer_near-optimal_2009} proposes the famous $\texttt{UCRL2}$ algorithm, a mature version of their prior \texttt{UCRL}  \cite{auer_logarithmic_2006}, achieving a regret bound of $\widetilde{\OH}(DS\!\!\sqrt{AT})$.
This paper pioneered the use \emph{optimistic} methods to learn MDPs efficiently.
A line of papers \cite{filippi_optimism_2010,fruit_improved_2020,bourel_tightening_2020} developed this direction by tightening the confidence region that \texttt{UCRL2} rely on, and sharpened its analysis through the use of local properties of MDPs, such as local diameters and local bias variances, but none of these works went beyond regret guarantees of order $S\!\!\sqrt{DAT}$ and suffer from an extra $\!\!\sqrt{S}$.
A parallel direction was initiated by \cite{bartlett_regal_2009}, that design \texttt{REGAL} to attain $\sp{h^*}$-dependent regret bounds (instead of $D$) while extending the regret bounds to weakly-communicating MDPs.
The computational intractability of \texttt{REGAL} is addressed by \cite{fruit_efcient_2018} with \texttt{SCAL}, while \cite{zhang_regret_2019} further enhance the regret analysis by evaluating the bias differences with \texttt{EBF}, eventually reaching optimal minimax regret but loosing tractability.

Another successful design approach is Bayesian-flavored sampling, derived from Thompson Sampling \cite{thompson_likelihood_1933}, that usually replaces optimism.
The regret guarantees of these algorithms usually stick to the Bayesian setting however \cite{ouyang_learning_2017,theocharous2017posterior}, although \cite{agrawal_optimistic_2023} also enjoys $\widetilde{\OH}(S\!\!\sqrt{DAT})$ high probability regret by coupling posterior sampling and optimism.
Another line of research focuses on the study of ergodic MDPs, where all policies mix uniformly according to a mixing time. 
To name a few, the model-free algorithm \texttt{Politex} \cite{abbasi2019exploration} attains a regret of $\widetilde{\OH}((t_{\mathrm{mix}})^3t_{\mathrm{hit}}\!\!\sqrt{SA}T^{\frac{3}{4}})$. By leveraging an optimistic mirror descent algorithm, \cite{wei_model-free_2020}  achieve an enhanced regret of $\widetilde{\OH}(\!\!\sqrt{(t_{\mathrm{mix}})^2t_{\mathrm{hit}}AT})$ .

We refer the readers to \cref{table:comparison} for a (non-exhaustive) list of existing algorithms.

\begin{table}[t]
    \centering
    \caption{
    \label{table:comparison}
        Comparison of related works on RL algorithms for average-reward MDP, where $S\times A$ is the size of state-action space, $T$ is the total number of steps, $D$ ($D_s$) is the (local) diameter, $\sp{h^*}\leq D$ is the span of the bias vector, $t_{\mathrm{mix}}$ is the worst-case mixing time, $t_{\mathrm{hit}}$ is the hitting time (i.e., the expected time cost to visit some certain state under any policy).
    }
    \resizebox{\linewidth}{!}{
    \def\arraystretch{1.33}
    \hspace{-.66em}
    \begin{tabular}{ cccc } 
    \hline
     \textbf{Algorithm} & \textbf{Regret in $\widetilde{\OH}(-)$} & \textbf{Tractable} &  \textbf{Comment/Requirements}\\
     \hline
     \texttt{REGAL} \cite{bartlett_regal_2009} & $\sp{h^*}S\!\!\sqrt{AT}$   &  $\times$ & knowledge of $\sp{h^*}$ \\
     \texttt{UCRL2} \cite{auer_near-optimal_2009} & $DS\!\!\sqrt{AT}$ & $\checkmark$   & -  \\
     \texttt{PSRL} \cite{agrawal_optimistic_2023} & $DS\!\!\sqrt{AT}$ & $\checkmark$ & Bayesian regret \\
     \texttt{SCAL} \cite{fruit_efcient_2018} &  $\sp{h^*}S\!\!\sqrt{AT}$&  $\checkmark$ & knowledge of $\sp{h^*}$\\
     \texttt{UCRL2B} \cite{fruit_improved_2020} & $S\!\!\sqrt{DAT}$ & $\checkmark$ & extra $\sqrt{\log(T)}$ in upper-bound \\
     \texttt{UCRL3} \cite{bourel_tightening_2020} &  $D + \!\!\sqrt{T\sum_{s,a}D_s^2 L_{s,a}} $ & $\checkmark$  & $L_{s,a}: = \sum_{s'}\!\!\sqrt{p(s'|s,a)(1-p(s'|s,a))}$\\
     \texttt{KL-UCRL} \cite{filippi_optimism_2010,talebi_variance-aware_2018} & $S\!\!\sqrt{DAT}$ & $\checkmark$ & - \\
     \texttt{EBF} \cite{zhang_regret_2019} & $\sqrt{\sp{h}^*SAT}$ & $\times$ & optimal, knowledge of $\sp{h^*}$   \\
    \texttt{Optimistic-Q} \cite{wei_model-free_2020} & $\sp{h^*}(SA)^{\frac{1}{3}}T^{\frac{2}{3}}$ & $\checkmark$ & model-free \\
    \texttt{UCB-AVG} \cite{zhang_sharper_2023} & $S^5A^2\sp{h^*}\!\!\sqrt{T}$ & $\checkmark$ & model-free, knowledge of $\sp{h^*}$ \\
    \texttt{MDP-OOMD} \cite{wei_model-free_2020} & $\sqrt{(t_{\mathrm{mix}})^2t_{\mathrm{hit}}AT}$ & $\checkmark$ & ergodic \\
    \texttt{Politex} \cite{abbasi2019exploration} & $(t_{\mathrm{mix}})^3t_{\mathrm{hit}}\!\!\sqrt{SA}T^{\frac{3}{4}}$  &$\checkmark$ & model-free, ergodic \\
    \texttt{PMEVI-DT} (\textbf{this work}) & $\sqrt{\sp{h^*}SAT}$ & $\checkmark$ &  -\\
    \hline
    \textbf{Lower bound} & $\Omega\parens{\!\!\sqrt{\sp{h^*}SAT}}$ & - &  -   
    \\ \hline
    \end{tabular}}
\end{table}

\section{Preliminaries}
We fix a finite state-action space structure $\Xc := \bigcup_{s \in \Sc} \set{s} \times \Ac(s)$, and denote $\Mc$ the collection of all MDPs with state-action space $\Xc$ and rewards supported in $[0,1]$.

\paragraph{Infinite-horizon MDP.}
An element $M \in \Mc$ is a tuple $(\Sc, \Ac, p, r)$ where $p$ is the transition kernel and $r$ the reward function.
The random state-action pair played by the agent at time $t$ is denoted $X_t \equiv (S_t, A_t)$, and the achieved reward is $R_t$. 
A policy is a \emph{deterministic} rule $\pi : \Sc \to \Ac$ and we write $\Pi$ the space of policies.
Coupled with a $M \in \Mc$, a policy properly defines the distribution of $(X_t)$ whose associated probability probability and expectation operators are denoted $\Pr^\pi_s, \E^\pi_s$, where $s \in \Sc$ is the initial state.
Under $M$, a fixed policy has a reward function $r^\pi(s) := r(s, \pi(s))$, a transition matrix $P^\pi$, a gain $g^\pi(s) := \lim \frac{1}{T} \E^\pi_s[R_0 + \ldots + R_{T-1}]$ and a bias $h^\pi := \lim \sum_{t=0}^{T-1} (R_t - g(S_t))$, that all together satisfy the Poisson equation $h^\pi + g^\pi = r^\pi + P^\pi h^\pi$, see \cite{puterman_markov_1994}.
The \emph{Bellman operator} of the MDP is:
\begin{equation}
    L u (s) 
    :=
    \max_{a \in \Ac(s)} \set{ r(s, a) + p(s, a) u }
\end{equation}

\paragraph{Weakly-communicating MDPs.} 
$M$ is  weakly-communicating \cite{puterman_markov_1994,bartlett_regal_2009} if the state space can be divided into two sets: (1) the transient set, consisting in
states that are transient under all policies; (2) the non-transient set, where every state is reachable starting from any other non-transient. 
In this case, $L$ has a span-fixpoint $h^*$ (see \cite{puterman_markov_1994}), i.e., there exists $h^* \in \R^\Sc$ such that 
$Lh^* - h^* \in \R e$
where $e$ is the vector full of ones.
We write $h^* \in \Fix(L)$. 
Then $g^* := Lh^* - h^*$ is the optimal gain function and every policy $\pi$ satisfies $r^\pi + P^\pi h^* \le g^* + h^*$.
We accordingly define the \emph{Bellman gaps}:
\begin{equation}
    \Delta^* (s, a) 
    :=
    h^*(s) + g^*(s) - r(s, a) - p(s, a) h^*
    \ge 0
    .
\end{equation}
Another important concept is the \emph{diameter}, that describes the maximal distance from one state to another state.
It is given by $D := \sup_{s \ne s'} \inf_\pi \E_s^\pi [ \inf\set{t \ge 1 : S_t = s'} ].$
An MDP is said \emph{communicating} if its diameter $D$ is finite.

\paragraph{Reinforcement learning.}
The learner is only aware that $M \in \Mc$ but doesn't have a clue about what $M$ further looks like. 
From the past observations and the current state $S_t$, the agent picks an available action $\Ac(S_t)$, receives a reward $R_t$ and observe the new state $S_{t+1}$.
The \emph{regret} of the agent is:
\begin{equation}
    \Reg(T)
    :=
    T g^* - \sum_{t=0}^{T-1} R_t
    .
\end{equation}
Its expected value satisfies $\E[\Reg(T)] = \E[\sum_{t=0}^{T-1} \Delta^*(X_t)] + \E[h^*(S_0) - h^*(S_T)]$ and the quantity $\sum_{t=0}^{T-1} \Delta^*(X_t)$ will be referred to as the \emph{pseudo-regret}.
This paper focuses on \emph{minimax regret guarantees}.
Specifically, for $c \ge 1$, denote $\Mc_c := \set{M \in \Mc : \exists h^* \in \Fix(L(M)), \sp{h^*} \le c}$ the set of weakly-communicating MDPs that admit a bias function with span at most $c$, where the \emph{span} of a vector $u \in \mathbb{R}^\Sc$ is $\sp{u} := \max(u) - \min(u)$.
Following \cite{auer_near-optimal_2009}, every algorithm $\Abf$, for all $c > 0$, we have 
\begin{equation}
    \max_{M \in \Mc_c} 
    \E^{M, \Abf} [\Reg(T)]
    = 
    \Omega\parens{\!\sqrt{c S A T}}
    .
\end{equation}
The goal of this work is to reach this lower bound with a tractable algorithm.

\section{Algorithm \ALGORITHM}
\label{section:ptevi}

The method designed in this work can be applied to \emph{any} algorithm relying on extended Bellman operators to compute the deployed policies \cite{auer_near-optimal_2009,filippi_optimism_2010,fruit_efcient_2018,bourel_tightening_2020} and beyond \cite{tewari_optimistic_2007}.
We start by reviewing the principles behind these algorithms.
These algorithms follow the \emph{optimism-in-face-of-certainty} (OFU) principle, meaning that they deploy policies achieving the highest possible gain that is plausible under their current information.
This is done by building a confidence region $\Mc_t \subseteq \Mc$ for the hidden model $M$, then searching for a policy $\pi$ solving the optimization problem:
\begin{equation}
\label{equation:optimistic-gain}
    g^*(\Mc_t)
    :=
    \sup \set{ g^\pi(\Mc_t): \pi \in \Pi, \sp{g^\pi(\Mc_t)} = 0}
    \text{~with~}
    g^\pi(\Mc_t) :=
    \sup \set{
        g\parens{\pi, \widetilde{M}}
        :
        \widetilde{M} \in \Mc_t
    }
    .
\end{equation}
The design of the confidence region $\Mc_t$ varies from a work to another. 
Provided that $\Mc_t$ has been designed, these OFU-algorithms work as follows: 
At the start of episode $k$, the optimization problem \eqref{equation:optimistic-gain} is solved, and its solution $\pi_k$ is played until the end of episode.
The duration of episodes can be managed in various ways, although the most popular is arguably the \emph{doubling trick} (DT), that essentially waits until a state-action pair is about to double the visit count it had at the beginning of the current episode (see \hyperlink{algorithm:PMEVI-dt}{Algorithm~1}). In the rest of this section, we use $\hat{p}_t(s,a)$ (and $\hat{r}_t(s,a) $) to denote the empirical transition (and reward) of the latest doubling update before the $t$-th step, and further denote $\hat{M}_t := (\hat{r}_t, \hat{p}_t)$.

\paragraph{Extended Bellman operators and \texttt{EVI}.}
To solve \eqref{equation:optimistic-gain} efficiently, the celebrated \cite{auer_near-optimal_2009} introduced the \emph{extended value iteration} algorithm (\texttt{EVI}).
Assume that $\Mc_t$ is a $(s,a)$-rectangular confidence region, meaning that $\Mc_t \equiv \prod_{s,a} (\Rc_t(s,a) \times \Pc_t(s,a))$ where $\Rc_t(s,a)$ and $\Pc_t(s,a)$ are respectively the confidence region for $r(s,a)$ and $p(s,a)$ after $t$ learning steps.
\texttt{EVI} is the algorithm computing the sequence defined by:
\begin{equation}
\label{equation:evi}
    v_{i+1}(s) \equiv 
    \Lc_t v_i (s)
    :=
    \max_{a \in \Ac(s)} \max_{\tilde{r}(s,a) \in \Rc_t(s,a)} \max_{\tilde{p}(s,a) \in \Pc_t(s,a)}
    \parens{
        \tilde{r}(s,a) + \tilde{p}(s,a) \cdot v_i
    }
\end{equation}
until $\sp{v_{i+1} - v_i} < \epsilon$ where $\epsilon > 0$ is the numerical precision.
When the process stops, it is known that any policy $\pi$ such that $\pi(s)$ achieves $\Lc_t v_i$ in \eqref{equation:evi} satisfies $g^\pi(\Mc_t) \ge g^*(\Mc) - \epsilon$, hence is nearly optimistically optimal.
This process gets its name from the observation that $\Lc_t$ is the Bellman operator of $\Mc_t$ seen as a MDP, hence \texttt{EVI} is just the Value Iteration algorithm \cite{puterman_markov_1994} ran in $\Mc_t$.
A choice of action from $s \in \Sc$ in $\Mc_t$ consists in (1) a choice of action $a \in \Ac(s)$, (2) a choice of reward $\tilde{r}(s,a) \in \Rc_t(s,a)$ and (3) a choice of transition $\tilde{p}(s,a) \in \Pc_t(s,a)$; It is an \emph{extended} version of $\Ac(s)$.

\paragraph{Towards Projected Mitigated \texttt{EVI}.}
Obviously, the regret of an OFU-algorithm is directly related to the quality of the confidence region $\Mc_t$.
That is why most previous works tried to approach the regret lower bound $\!\!\sqrt{D S A T}$ of \cite{auer_near-optimal_2009} by refining $\Mc_t$.
The older works of \cite{auer_near-optimal_2009,bartlett_regal_2009,filippi_optimism_2010} have been improved with a variance aware analysis \cite{talebi_variance-aware_2018,fruit_efcient_2018,fruit_improved_2020,bourel_tightening_2020} that essentially make use of tightened kernel confidence regions $\Pc_t$.
While all these algorithms successively reduce the gap between the regret upper and lower bounds, they fail to achieve optimal regret $\!\!\sqrt{D S A T}$.
Meanwhile, the \texttt{EVI} algorithm of \cite{zhang_regret_2019} achieves the lower bound but (1) the algorithm is intractable because it relies on an oracle to retrieve optimistically optimal policies and (2) needs prior information on the bias function.
Nonetheless, the method of \cite{zhang_regret_2019} strongly suggests that inferring bias information from the available data is key to achieve minimax optimal regret.

Rather surprisingly and in opposition to this previous line of work, our work suggests that the choice of the confidence region $\Mc_t$ has little importance.
Instead, our algorithm takes an arbitrary (well-behaved) confidence region in, infer bias information similarly to \texttt{EBF} \cite{zhang_regret_2019} and makes use of it to heavily refine the extended Bellman operator \eqref{equation:evi} associated to the input confidence region.
Our algorithm can further take arbitrary prior information (possibly none) on the bias vector to tighten its bias confidence region.
The pseudo-code given in \hyperlink{algorithm:PMEVI}{Algorithm~1} is the high level structure our algorithm \texttt{PMEVI-DT}.
In \cref{section:PMEVI}, we explain how \eqref{equation:evi} is refined using bias information and in \cref{section:bias-region}, we explain how bias information is obtained.

\begin{figure}[ht]
\begin{minipage}[t]{.49\linewidth}
    \rule{\linewidth}{2pt}
    \hypertarget{algorithm:PMEVI-dt}{\textbf{Algorithm 1:}} \texttt{PMEVI-DT}$(\Hc_*, T, t \mapsto \Mc_t)$

    \vspace{-0.5em}
    \rule{\linewidth}{1pt}
    
    \textbf{Parameters:} Bias prior $\Hc_*$, horizon $T$, a system of confidence region $t \mapsto \Mc_t$
    
    \begin{algorithmic}[1]\label{alg:PMEVI}
        \FOR{$k=1, 2, \ldots$}
            \STATE Set $t_k \gets t$, update confidence region $\Mc_t$;
            \STATE $\Hc'_t \gets \texttt{BiasEstimation}(\Fc_t, \Mc_t, \delta)$:
            \STATE $\Hc_t \gets \Hc_* \cap \{u : \sp{u} \le T^{1/5}\} \cap \Hc_t'$;
            \STATE $\Gamma_t \gets \texttt{BiasProjection}(\Hc_t, -)$; 
            \STATE $\beta_t \gets \texttt{VarianceApprox}(\Hc'_t, \Fc_t)$; 
            \STATE $\mathfrak{h}_k \gets \texttt{PMEVI}(\Mc_t, \beta_t, \Gamma_t,\!\!\sqrt{ \log(t)/t})$ ;
            \STATE $\mathfrak{g}_k \gets \mathfrak{L}_t \mathfrak{h}_k - \mathfrak{h}_k$ ;
            \STATE Update policy $\pi_k \gets \texttt{Greedy}(\Mc_t, \mathfrak{h}_k, \beta_t)$;
            \REPEAT
                \STATE{Play $A_t \gets \pi_k(S_t)$, observe $R_t, S_{t+1}$;}
                \STATE{Increment $t \gets t+1$};
            \UNTIL (DT) $N_t(S_t, \pi_k(S_t)) \ge 1 \vee 2 N_{t_k}(X_t)$.
        \ENDFOR
    \end{algorithmic}
    \vspace{-.5em}
    \rule{\linewidth}{1pt}
\end{minipage}\hfill
\begin{minipage}[t]{.49\linewidth}
    \rule{\linewidth}{2pt}
    \hypertarget{algorithm:PMEVI}{\textbf{Algorithm 2:}} \texttt{PMEVI}$(\Mc, \beta, \Gamma, \epsilon)$

    \vspace{-0.5em}
    \rule{\linewidth}{1pt}
    
    \textbf{Parameters:} region $\Mc$, mitigation $\beta$, projection $\Gamma$, precision $\epsilon$, initial vector $v_0$ (optional)
    
    \begin{algorithmic}[1]
        \STATE{\textbf{if} $v_0$ not initialized \textbf{then} set $v_0 \gets 0$};
        \STATE $n \gets 0$
        \STATE{$\Lc \gets $ extended operator associated to $\Mc$};
        \REPEAT
            \STATE{$v_{n+\frac 12} \gets \Lc^\beta v_n$};
            \STATE{$v_{n+1} \gets \Gamma v_{n + \frac 12}$};
            \STATE{$n \gets n+1$};
        \UNTIL{$\sp{v_n - v_{n-1}} < \epsilon$}
        \RETURN $v_n$.
    \end{algorithmic}
    \vspace{-.5em}
    \rule{\linewidth}{1pt}
\end{minipage}
\end{figure}

\subsection{Projected mitigated extended value iteration (\texttt{PMEVI})}
\label{section:PMEVI}

Assume that an external mechanism provides a confidence region $\Hc_t$ for the bias function $h^*$.
Provided that $\Mc_t$ is correct ($M \in \Mc_t$) and that $\Hc_t$ is correct ($h^* \in \Hc_t$), we want to find a pair of policy-model $(\pi, \tilde{M})$ that maximize the gain and such that $h(\pi, \tilde{M}) \in \Hc_t$.
This is done with an improved version of \eqref{equation:evi}  combining two ideas.

\begin{enumerate}
    \item \textbf{Projection ({\normalfont\cref{section:bias-region}}).}
    Whenever it is correct, the bias confidence region $\Hc_t$ informs the learner that the search of an optimistic model can be constrained to those with bias within $\Hc_t$.
    This is done by projecting $\Lc_t^\beta$ (see \emph{mitigation}) using an operator $\Gamma_t : \R^\Sc \to \Hc_t$, that
    has to satisfy a few non-trivial regularity conditions that are specified in \cref{proposition:PMEVI}.
    
    \item \textbf{Mitigation ({\normalfont\cref{section:beta}}).}
    When one is aware that $h^* \in \Hc_t$, the \emph{dynamical bias update} $\tilde{p} (s,a) u_i$ in \eqref{equation:evi} can be controlled better, by trying to restrict \eqref{equation:evi} to some $\tilde{p}(s,a)$ such that $\tilde{p}(s,a) u_i \le \hat{p}_t(s,a) u_i + (p(s,a) - \hat{p}_t (s,a)) u_i$ with the knowledge that $u_i \in \Hc_t$.
    
    For a fixed $u \in \R^\Sc$, the empirical Bernstein inequality (\cref{lemma:empirical-bernstein}) provides a variance bound of the form $(\hat{p}_t(s,a) - p(s,a))u \le \beta_t (s, a, u)$. 
    By computing $\beta_t (s,a) := \max_{u \in \Hc_t} \beta_t (s,a, u)$, the search makes sure that $(\hat{p}_t(s,a) - p(s,a)) h^* \le \beta_t (s,a)$ even though $h^*$ is unknown.
    For $\beta \in \R_+^\Xc$, we introduce the \emph{$\beta$-mitigated} extended Bellman operator:
    \begin{equation}
        \Lc_t^\beta u (s)
        :=
        \max_{a \in \Ac(s)} \sup_{\tilde{r}(s,a) \in \Rc_t(s,a)} \sup_{\tilde{p}(s,a) \in \Pc_t(s,a)} \Big\{
            \tilde{r}(s,a) + \min \set{
                \tilde{p}(s,a) u_i,
                \hat{p}_t(s,a) u_i + \beta_t(s,a)
            }
        \Big\}
    \end{equation}
\end{enumerate}

The proposition below shows how well-behaved the composition $\Lfk_t := \Gamma_t \circ \Lc_t^\beta$ is.
Its proof requires to build a complete analysis of projected mitigated Bellman operators.
This is deferred to the appendix.

\begin{proposition}
\label{proposition:PMEVI}
    Fix $\beta \in \R_+^\Xc$ and assume that there exists a projection operator $\Gamma_t : \R^\Xc \to \Hc_t$ which is 
    (\textbf{O1}) monotone: $u \le v \Rightarrow \Gamma u \le \Gamma v$; 
    (\textbf{O2}) non span-expansive: $\sp{\Gamma u - \Gamma v} \le \sp{u - v}$; 
    (\textbf{O3}) linear: $\Gamma(u + \lambda e) = \Gamma u + \lambda e$ and 
    (\textbf{O4}) $\Gamma u \le u$.
    Then, the \emph{projected mitigated extended Bellman operator} $\Lfk_t := \Gamma_t \circ \Lc_t^\beta$ has the following properties:
    \begin{itemize}[itemsep=-.25em]
        \item[(1)] 
            There exists a unique $\mathfrak{g}_t \in \R e$ such that $\exists \mathfrak{h}_t \in \Hc_t, \Lfk_t \mathfrak{h}_t = \mathfrak{h}_t + \mathfrak{g}_t$;
        \item[(2)] 
            If $M \in \Mc_t$, $h^* \in \Hc_t$ and $(\hat{p}_t(s,a) - p(s,a))h^* \le \beta_t(s,a)$, then $\mathfrak{g}_t \ge g^*(M)$;
        \item[(3)]
            If $\Mc_t$ is convex, then for all $u \in \R^\Sc$, the policy $\pi =: \texttt{\upshape Greedy}(\Mc_t, u, \beta_t)$ picking the actions achieving $\Lc_t^\beta u$ satisfies $\Lfk_t u = \tilde{r}^\pi + \tilde{P}^\pi u$ for $\tilde{r}^\pi (s) \le \sup \Rc_t (s, \pi(s))$ and $\tilde{P}^\pi(s) \in \Pc_t (s, \pi(s))$;
        \item[(4)]
            For all $u \in \R^\Sc$ and $n \ge 0$, $\sp{\Lfk_t^{n+1} u - \Lfk_t^n u} \le \sp{\Lc_t^{n+1} u - \Lc_t^n u}$.
    \end{itemize}
\end{proposition}

The property (1) guarantees that $\Lfk_t$ has a fix-point while (2) states that this fix-point corresponds to an optimistic gain $\mathfrak{g}_t$ if the model and the bias confidence region are correct and the mitigation isn't too aggressive. 
Combined with (3), the Poisson equation of a policy corresponds to this fix-point, i.e., $\tilde{r}^\pi + \tilde{P}^\pi \mathfrak{h}_t = \mathfrak{h}_t + \mathfrak{g}_t$, so that $\mathfrak{g}_t$ is the gain and $\mathfrak{h}_t \in \Hc_t$ is a legal bias for $\pi$ under the model $(\tilde{r}^\pi, \tilde{P}^\pi)$.
Lastly, the property (4) guarantees that the iterates $\Lfk_t^n u$ converge to a fix-point of $\Lfk$ at least as quickly as $\Lc_t^n u$ goes to a fix-point of $\Lc_t$; the convergence of $\Lc_t^n u$ is already guaranteed by existing studies and is discussed in the appendix.

Provided that the bias confidence region is constructed, \cref{proposition:PMEVI} foreshadows how powerful is the construction: The algorithm \texttt{PMEVI}, obtained by iterating $\mathfrak{L}_t$ instead of $\Lc_t$ in \texttt{EVI}, can replace the well-known \texttt{EVI} within any algorithm of the literature that relies on it (\texttt{UCRL2} \cite{auer_near-optimal_2009}, \texttt{UCRL2B} \cite{fruit_improved_2020} or \texttt{KL-UCRL} \cite{filippi_optimism_2010}) for an immediate improvement of its theoretical guarantees.

\subsection{Building the bias confidence region and its projection operator}
\label{section:bias-region}

The bias confidence region used by \texttt{PMEVI-DT} is obtained as a collection of constraints of the form:
\begin{equation}
\label{equation:bias-region-constraints}
    \forall s \ne s', \quad
    \mathfrak{h}(s) - \mathfrak{h}(s') - c(s, s')
    \le d(s, s')
    .
\end{equation}
Such constraints include 
(1)~prior bias constraints (if any) of the form of $\mathfrak{h}(s) - \mathfrak{h}(s') \le c_*(s,s')$;
(2)~span constraints of the form $\mathfrak{h}(s) - \mathfrak{h}(s') \le c_0 := T^{1/5}$ spawning the span semi-ball $\{u: \sp{u} \le T^{1/5}\}$; and 
(3)~pair-wise constraints obtained by estimating bias differences in the style of \cite{zhang_regret_2019, zhang_sharper_2023} that we further improve.
We start by defining a bias difference estimator.

\begin{definition}[Bias difference estimator]
    Given a pair of states $s \ne s'$, their sequence of \emph{commute times} $(\tau_i^{s \toto s'})_{i \ge 0}$ is defined by
    $
        \tau_{2i}^{s \toto s'}
        :=
        \inf \{t > \tau_{2i-1}^{s \toto s'} : S_t = s\}
        \text{~and~}
        \tau_{2i+1}^{s \toto s'}
        :=
        \inf \{t > \tau_{2i}^{s \toto s'} : S_t = s'\}
    $
    with the convention that $\tau_{-1}^{s \toto s'} = - \infty$.
    The number of commutations up to time $t$ is $N_t (s \toto s') := \inf \{i : \tau_i^{s \toto s'} \le t\}$, and $\hat{g}(t) := \frac 1t \sum_{i=0}^{t-1} R_i$ is the empirical gain.
    The \emph{bias difference estimator} at time $T$ is any quantity $c_T (s, s') \in \R$ such that:
    \begin{equation}
    \label{equation:bias-difference-estimator}
        N_t (s \toto s') c_T (s, s')
        =
        \sumnl_{t=0}^{N_T(s \toto s') -1}
        (-1)^i \sumnl_{t = \tau_i^{s \toto s'}}^{\tau_{i+1}^{s \toto s'}-1}
        (\hat{g}(T) - R_t)
        .
    \end{equation}
\end{definition}

\begin{lemma}
\label{lemma:bias-differences-error}
    With probability $1 - 2\delta$, for all $T' \le T$ and all $\tilde{g} \ge g^*$, we have:
    \begin{equation}
        N_{T'} (s \toto s') \abs{
            h^*(s) - h^*(s') - c_{T'} (s, s')
        }
        \le
        3 \sp{h^*} +
        (1 + \sp{h^*}) \sqrt{8 T \log(\tfrac 2\delta)}
        + 2 \sumnl_{t=0}^{T'-1} (\tilde{g} - R_t)
        .
    \end{equation}      
\end{lemma}

\cref{lemma:bias-differences-error} says that the quality of the estimator $c_T (s, s')$ is directly linked to the number of observed commutes between $s$ and $s'$ as well as the regret.
The idea is that if the algorithm makes many commutes between $s$ and $s'$ and if its regret is small, then the algorithm mostly takes optimal paths from $s$ to $s'$.
The bound provided by \cref{lemma:bias-differences-error} is not accessible to the learner however, because $\sp{h^*}$ is unknown in general.
To overcome this issue, $\sp{h^*}$ is upper-bounded by $c_0 := T^{1/5}$.
Overall, this leads to the design of the algorithm estimating the bias confidence region as specified in \hyperlink{algorithm:bias-estimation}{Algorithm 3}.

\begin{figure}[ht]
\begin{minipage}[t]{.49\linewidth}
    \rule{\linewidth}{2pt}
    \hypertarget{algorithm:bias-estimation}{\textbf{Algorithm 3:}} \texttt{BiasEstimation}$(\Fc_t, \Mc_t, \delta)$

    \vspace{-0.5em}
    \rule{\linewidth}{1pt}
    
    \textbf{Parameters:} History $\Fc_t$, model region $\Mc_t$, confidence $\delta > 0$
    \begin{algorithmic}[1]
        \STATE Estimate bias differences $c_t$ via \eqref{equation:bias-difference-estimator};
        \STATE Estimate optimistic gain $\tilde{g} \gets \min_{k < K(t)} \mathfrak{g}_k$;
        \STATE Inner regret estimation $B_0 \gets t \tilde{g} - \sum_{i=0}^{t-1} R_i$;
        \STATE $\ell \gets \sqrt{8 T \log\parens{\tfrac 2\delta}}$, $c_0 \gets T^{\frac 15}$;
        \STATE Estimate the bias difference errors as:
        
        \medskip
        \resizebox{\linewidth}{!}{
        $
            \displaystyle
            d_t (s,s') \equiv \text{error}(c_t, s, s')
            :=
            \frac{3 c_0 + (1 + c_0) (1 + \ell) + 2 B_0}{N_t (s\toto s')}
        $}
        \RETURN $(c_t, \text{error}(c_t, -, -))$, \eqref{equation:bias-region-constraints} defines $\Hc'_t$.
    \end{algorithmic}

    \vspace{-0.5em}
    \rule{\linewidth}{1pt}
\end{minipage} \hfill
\begin{minipage}[t]{.49\linewidth}
    \rule{\linewidth}{2pt}
    \hypertarget{algorithm:bias-projection}{\textbf{Algorithm 4:}} \texttt{BiasProjection}$(\Hc_t, u)$

    \vspace{-0.5em}
    \rule{\linewidth}{1pt}
    
    \textbf{Parameters:} $\Hc_t$ a collection of linear constraints \eqref{equation:bias-region-constraints}, $u \in \R^\Sc$ to project
    \begin{algorithmic}[1]
        \STATE $v \gets 0^\Sc$;
        \FOR{$s \in \Sc$}
            \STATE Using linear programming, compute:
            \STATE $v(s) \gets \sup \set{w(s) : w \le u \text{~and~} w \in \Hc_t}$;
        \ENDFOR
        \RETURN $v$.
    \end{algorithmic}

    \vspace{-0.5em}
    \rule{\linewidth}{1pt}
\end{minipage}
\end{figure}


Coupled with prior information and span constraints, the obtained bias confidence region $\Hc_t$ is a polyhedron of the same kind as the one encountered in \cite{zhang_sharper_2023} generated by constraints of the form \eqref{equation:bias-region-constraints}, and similarly to their Proposition 3, one can project onto $\Hc_t$ in polynomial time with \hyperlink{algorithm:bias-projection}{Algorithm 4}.
Moreover, the resulting projection operator satisfies the prerequisites (\textbf{O1-4}) of \cref{proposition:PMEVI}, making sure that \texttt{PMEVI} (\hyperlink{algorithm:PMEVI}{Algorithm~2}) is well-behaved.
This is proved in the appendix \cref{section:projection-operation}.

\begin{lemma}
\label{lemma:projection-algorithm}
    Assume that $\Hc$ is a set of $\mathfrak{h} \in \R^\Sc$ satisfying a system of equations of the form of \eqref{equation:bias-region-constraints}.
    If $\Hc$ is non empty, then the operator $\Gamma u := \text{\texttt{\upshape BiasProjection}}(\Hc, u)$ (see \hyperlink{algorithm:bias-projection}{Algorithm 4}) is a projection on $\Hc$ and satisfies the properties (\textbf{O1-4}) defined in \cref{proposition:PMEVI}.
\end{lemma}

\subsection{Mitigation using finer bias dynamical error}
\label{section:beta}

The fact that $h^* \in \Hc_t$ with high probability is used in \texttt{PMEVI-DT} to restrict the search of \texttt{EVI} by reducing the dynamical bias error.
This reduction is based on a empirical Bernstein inequality (see \cref{lemma:empirical-bernstein}) applied to $(\hat{p}(s, a) - p(s,a)) u$.
Here, it gives that with probability $1 - \delta$, we have:
\begin{equation}
\label{equation:empirical-bernstein}
    \parens{\hat{p}_t (s,a) - p(s,a)} u
    \le 
    \sqrt{
        \frac{2 \Var(\hat{p}_t(s,a), u) \log\parens{\tfrac{3T}\delta}}{\max\set{1, N_t (s,a)}}
    }
    +
    \frac{3 \sp{u} \log\parens{\tfrac{3T}\delta}}{\max\set{1, N_t(s,a)}}
    =:
    \beta_t (s,a, u)
\end{equation}
where $\Var(\hat{p}_t(s,a), u)$ is the variance of $u$ under the probability vector $\hat{p}_t(s,a)$.
More specifically, if $q$ is a probability on $\mathcal{S}$ and $q \in \R^\Sc$, we set $\Var(q, u) := \sum_s q(s) (u(s) - q \cdot u)^2$.
In \eqref{equation:empirical-bernstein}, $u \in \R^\Sc$, $(s,a) \in \Xc$ and $T \ge 1$ are fixed.
Once is tempted to use \eqref{equation:empirical-bernstein} directly to mitigate the extended Bellman operator, but the resulting operator is ill-behaved because it loses monotony.
This issue is avoided by changing $\beta_t (s,a, u)$ to $\max_{u \in \Hc_t} \beta_t (s,a, u)$ in \eqref{equation:bias-difference-estimator}.
We obtain a variance maximization problem, which is a \emph{convex maximization problem} with linear constraints.
Even in very simple settings, such optimization problems are NP-hard \cite{pardalos_checking_1988} hence computing $\max_{u \in \Hc_t} \beta_t (s,a,u)$ is not reasonable in general.
Thankfully, this value can be upper-bounded by a tractable quantity that is enough to guarantee regret efficiency.
The mitigation $\beta_t$ used by \texttt{PMEVI-DT} is provided with \hyperlink{algorithm:variance-approximation}{Algorithm 5}.

\begin{figure}[ht]
    \centering
    \begin{minipage}[t]{.66\linewidth}
        \rule{\linewidth}{2pt}
        \hypertarget{algorithm:variance-approximation}{\textbf{Algorithm 5:}} \texttt{VarianceApproximation}$(\Hc'_t, \Fc_t)$
    
        \vspace{-0.5em}
        \rule{\linewidth}{1pt}
        
        \textbf{Parameters:} Bias region $\Hc'_t$, history $\Fc_t$
        \begin{algorithmic}[1]
            \STATE Extract constraints $(c, \text{error}(c, -, -)) \gets \Hc'_t$;
            \STATE Set $c_0 \gets T^{\frac 15}$;
            \STATE Pick a reference point $h_0 \gets \text{\texttt{BiasProjection}}(\Hc_t, c(-, s_0))$;
            \FOR{$(s, a) \in \Xc$}
                \STATE $\rho \gets \log\parens{\tfrac{S A T}\delta}/\max\set{1, N_t(s,a)}$;
                \STATE $\text{var}(s,a) \gets \Var(\hat{p}_t(s,a), h_0) + 8 c_0 \sum_{s' \in \Sc} \hat{p}_t(s'|s,a) c(s', s)$;
                \STATE $\beta_t(s, a) \gets \sqrt{2 \text{var}(s,a) \rho} + 3 c_0 \rho$ or $+\infty$ if $N_t(s,a) = 0$;
            \ENDFOR
            \RETURN $\beta_t$.
        \end{algorithmic}
    
        \vspace{-0.5em}
        \rule{\linewidth}{1pt}
    \end{minipage}
\end{figure}

\section{Regret guarantees}

\cref{theorem:main} below shows that \ALGORITHM{} has minimax optimal regret under regularity assumptions on the used confidence region $\Mc_t$.
\cref{assumption:model-confidence-region} asserts that the confidence region holds uniformly with high probability.
\cref{assumption:sub-weissman} asserts that the reward confidence region is sub-Weissman (see \cref{lemma:uniform-weissman}) and \cref{assumption:evi-convergence} assumes that the model confidence region makes sure that \texttt{EVI} \eqref{equation:evi} converges in the first place.
\cref{assumption:bias-region} asserts that the prior bias region is correct.

\begin{assumption}
\label{assumption:model-confidence-region}
    With probability $1 - \delta$, we have $M \in \bigcap_{k=1}^{K(T)} \Mc_{t_k}$.
\end{assumption}

\begin{assumption}
\label{assumption:sub-weissman}
    There exists a constant $C > 0$ such that for all $(s, a) \in \Sc$, for all $t \le T$, we have:
    \begin{equation*}
        \Rc_t (s,a) 
        \subseteq 
        \set{
            \tilde{r}(s,a) \in \Rc(s,a):
            N_t (s,a) \norm{\hat{r}_t(s,a) - \tilde{r}(s,a)}_1^2
            \le 
            C \log\parens{\tfrac{2 S A (1+N_t(s,a))}\delta}
        }
        .
    \end{equation*}
\end{assumption}

\begin{assumption}
\label{assumption:evi-convergence}
    For $t \ge 0$, $\Mc_t$ is a $(s,a)$-rectangular convex region and $\Lc_t^n u$ converges a fix-point.
\end{assumption}

\begin{assumption}
\label{assumption:bias-region}
    The prior bias region $\Hc_*$ contains $h^*(M)$ and is generated by constraints of the form:
    \begin{equation*}
        \forall s \ne s', 
        \quad 
        \mathfrak{h}(s) - \mathfrak{h}(s') \le c_*(s,s')
    \end{equation*}
    with $c_*(s,s') \in [-\infty, \infty]$ (possibly infinite).
\end{assumption}

Refer to \cref{section:model-confidence-region} for the feasibility of \cref{assumption:model-confidence-region}, \cref{section:sub-weissman} for \cref{assumption:sub-weissman}, and \cref{section:evi-convergence} for \cref{assumption:evi-convergence}.

\begin{theorem}[Main result]
    \label{theorem:main}
    Let $c > 0$.
    Assume that \texttt{\upshape PMEVI-DT} runs with a confidence region system $t \mapsto \Mc_t$ that guarantees Assumptions~\ref{assumption:model-confidence-region}-\ref{assumption:evi-convergence}.
    If $T \ge c^5$, then for every weakly communicating model with $\sp{h^*} \le c$ and such that \cref{assumption:bias-region} is satisfied ($h^* \in \Hc_*$), \texttt{\upshape PMEVI-DT} achieves regret:
    \begin{equation*}
        \OH\parens{
            \!\!
            \sqrt{
                c S A T 
                \log\parens{\tfrac {SAT}\delta}
            }
        }
        + \OH\parens{
            c S^{\frac 52} A^{\frac 32} T^{\frac 9{20}} \log^2\parens{\tfrac{SAT}\delta}
        }
    \end{equation*}
    with probability $1 - 26\delta$, and in expectation if $\delta < \!\!\sqrt{1/T}$.
    Moreover, if \texttt{\upshape PMEVI-DT} runs with the same confidence regions that \texttt{\upshape UCRL2} \cite{auer_near-optimal_2009}, then it enjoys a time complexity $\OH(D S^3 A T)$.
\end{theorem}

To have a completely prior-less algorithm, pick $\Hc_* = \R^\Sc$.
The proof of \cref{theorem:main} is too long to fit within these pages, so the complete proof is deferred to appendix.
We will focus here on the main ideas.

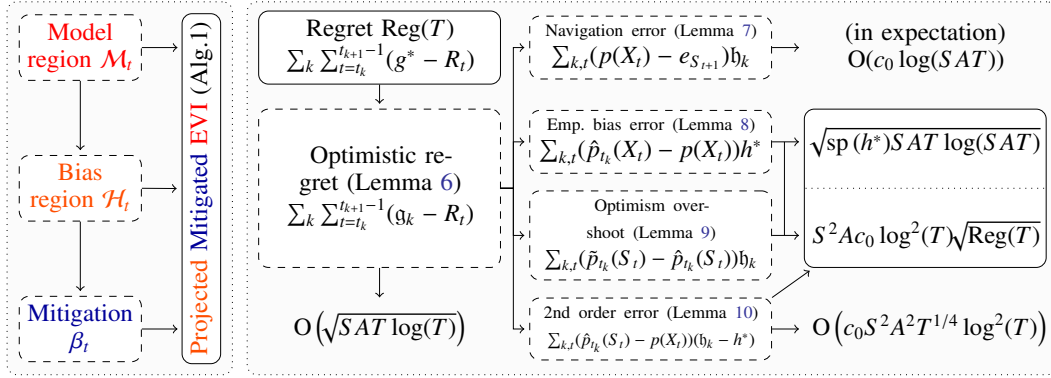
\begin{figure}[ht]
    \resizebox{\linewidth}{!}{\begin{tikzpicture}[shorten >=2pt]
        \tikzstyle{lemma}=[text width=3.2cm, align=center, draw, rounded corners, fill=white]
        \tikzstyle{error}=[text width=3.2cm, align=center]

        \node[draw, rounded corners, fill=gray!3, minimum width=3.2cm, minimum height=5.3cm, dotted] at (0.56, 2) {};
        \node[draw, rounded corners, fill=gray!3, minimum width=11.5cm, minimum height=5.3cm, dotted] at (8.12, 2) {};

        \node[lemma, text width=1.5cm, dashed] at (0, 4) (Mt) {\color{red}Model region $\Mc_t$};
        \node[lemma, text width=1.5cm, dashed] at (0, 2) (Ht) {\color{red!70!yellow}Bias region $\Hc_t$};
        \node[lemma, text width=1.5cm, dashed] at (0, 0) (Bt) {\color{blue!60!black}Mitigation $\beta_t$};
        \node[lemma, rotate=90, text width=4.74cm, anchor=south west] (PMEVI) at (2, -.5) {\textcolor{red!70!yellow}{Projected} \textcolor{blue!60!black}{Mitigated} \textcolor{red}{EVI} ({Alg.1})};

        \draw[->] (Mt) to (Ht);
        \draw[->] (Ht) to (Bt);
        \draw[->] (Mt.east) to (1.45, 4);
        \draw[->] (Ht.east) to (PMEVI);
        \draw[->] (Bt.east) to (1.45, 0);

        \node[lemma, minimum height=.9cm] (Reg) at (4.25, 4) {Regret $\Reg(T)$ \\ $\sum_{k} \sum_{t=t_k}^{t_{k+1}-1} (g^* - R_t)$};
        \node[lemma, dashed, minimum height=2.25cm] (OptReg) at (4.25, 2) {Optimistic regret (\cref{lemma:reward-optimism})\\ $\sum_{k} \sum_{t=t_k}^{t_{k+1}-1} (\mathfrak{g}_k - R_t)$};
        \node[error] (Rew) at (4.25, 0) {$\OH\parens{\!\!\sqrt{SAT\log(T)}}$};
        
        \node[lemma, dashed] (Nav) at (8.1, 4) {{\scriptsize Navigation error (\cref{lemma:navigation-error})} $\sum_{k,t} (p(X_t) - e_{S_{t+1}}) \mathfrak{h}_k$};
        \node[lemma, dashed] (Emp) at (8.1, 2.66) {{\scriptsize Emp.~bias error (\cref{lemma:empirical-bias-error})} $\sum_{k,t} (\hat{p}_{t_k}(X_t) - p(X_t)) h^*$};
        \node[lemma, dashed] (Opt) at (8.1, 1.33) {{\scriptsize Optimism overshoot (\cref{lemma:optimism-overshoot})} \small $\sum_{k,t} (\tilde{p}_{t_k} (S_t) - \hat{p}_{t_k}(S_t))\mathfrak{h}_k$};
        \node[lemma, dashed] (2nd) at (8.1, 0) {{\scriptsize 2nd order error (\cref{lemma:second-order-error})} \scriptsize $\sum_{k,t} (\hat{p}_{t_k}(S_t) - p(X_t)) (\mathfrak{h}_k - h^*)$};
        
        \node[error] (EqEp) at (12, 4) {(in expectation) $\OH(c_0 \log(SAT))$};
        \node[lemma, minimum height=2.25cm] (EqBx) at (12, 2) {};
        \node[error] (EqSp) at (12, 2.66) {$\!\!\sqrt{\sp{h^*} S A T \log(SAT)}$};
        \node[error] (EqRg) at (12, 1.33) {$S^2 A c_0 \log^2(T) \!\!\sqrt{\Reg(T)}$};
        \draw[dotted] (10.3, 2) to (13.8, 2);
        \node[error] (Eq2d) at (12, 0) {$\OH\parens{c_0 S^2 A^2 T^{1/4} \log^2(T)}$};

        \draw[->] (Reg) to (OptReg);
        \draw[->] (OptReg) to (Rew);

        \draw[->] (OptReg.east) to (6.15, 2) to (6.15, 4) to (Nav.west);
        \draw[->] (OptReg.east) to (6.15, 2) to (6.15, 2.66) to (Emp.west);
        \draw[->] (OptReg.east) to (6.15, 2) to (6.15, 1.33) to (Opt.west);
        \draw[->] (OptReg.east) to (6.15, 2) to (6.15, 0) to (2nd.west);
        
        \draw (Nav) edge[->] (EqEp);
        \draw (Emp) edge[->] (EqSp);
        \draw (Opt) edge[->] (EqRg);
        \draw (10, 1.33) to (10, 2.73);
        \draw (2nd) edge[->] (Eq2d);
        \draw (2nd.north east) edge[->] (EqBx);
    \end{tikzpicture}}
    \caption{
        \label{figure:overview}
        An overview of \texttt{PMEVI-DT} and its regret analysis.
        In the above, $\mathfrak{g}_k$ and $\mathfrak{h}_k$ are the optimistic gain and bias functions produced by \texttt{PMEVI} (see Algorithm~2) at episode $k$, and $\hat{p}_{t_k}$ and $\tilde{p}_{t_k}$ are respectively the empirical and optimistic kernel models at episode $k$.
    }
\end{figure}

We start by introducing notations.
At episode $k$, the played policy is denoted $\pi_k$.
As a greedy response to $\mathfrak{h}_k$, by \cref{proposition:PMEVI} (3), there exists $\tilde{r}_k(s) \le \sup \Rc_{t_k}(s, \pi_k(s))$ and $\tilde{P}_k(s) \in \Pc_{t_k}(s, \pi(x))$ such that $\mathfrak{h}_k + \mathfrak{g}_k = \tilde{r}_k + \tilde{P}_k\mathfrak{h}_k$.
The reward-kernel pair $\tilde{M}_k = (\tilde{r}_k, \tilde{P}_k)$ is referred to as the \emph{optimistic model} of $\pi_k$.
We write $P_k := P_{\pi_k}(M)$ the true kernel and $\hat{P}_k := P_{\pi_k} (\hat{M}_{t_k})$ the empirical kernel.
Likewise, we define the reward functions $r_k$ and $\hat{r}_k$. 
The optimistic gain and bias satisfy $\mathfrak{g}_k = g(\pi_k, \widetilde{M}_k)$ and $\mathfrak{h}_k = h(\pi_k, \widetilde{M}_k)$. 
We further denote $c_0 = T^{\frac 15}$.

The regret is first decomposed episodically with $\Reg(T) = \sum_k \sum_{t=t_k}^{t_{k+1}-1} (g^* - R_t)$.
The first step goes back to the analysis of \texttt{UCRL2} \cite{auer_near-optimal_2009}, and consists in upper-bounding the regret over episode $k$ with optimistic quantities that are exclusive to that episode.

\begin{lemma}[Reward optimism]
\label{lemma:reward-optimism}
    With probabililty $1 - 6\delta$, we have:
    \begin{equation}
        \Reg(T)
        \le
        \sumnl_k \sumnl_{t=t_k}^{t_{k+1}-1} (\mathfrak{g}_k - R_t)
        \le
        \sumnl_k \sumnl_{t=t_k}^{t_{k+1}-1} (\mathfrak{g}_k - \tilde{r}_k(X_t))
        + \OH\parens{\sqrt{S A T \log\parens{\tfrac T\delta}}}
        .
    \end{equation}
\end{lemma}

We introduce the two optimistic regrets $B(T) := \sum_k \sum_{t=t_k}^{t_{k+1}-1} (\mathfrak{g}_k - R_t)$ and $\tilde{B}(T) := \sum_k \sum_{t=t_k}^{t_{k+1}-1} (\mathfrak{g}_k - \tilde{r}_k(X_t))$.
Rewriting the summand $\mathfrak{g}_k - \tilde{r}_k(X_t)$ using the Poisson equation $\mathfrak{h}_k + \mathfrak{g}_k = \tilde{r}_k + \tilde{P}_k \mathfrak{h}_k$, we get:
\begin{equation*}
    \tilde{B}(T) 
    =
    \sumnl_k \sumnl_{t=t_k}^{t_{k+1}-1} \parens{\tilde{p}_k(S_t) - e_{S_t}} \mathfrak{h}_k
    .
\end{equation*}
The analysis proceed by decomposing the above expression of $\tilde{B}(T)$ in the style of \cite{zhang_regret_2019}.
We write $\sum_{t=t_k}^{t_{k+1}-1} (\tilde{p}_k(S_t) - e_{S_t}) \mathfrak{h}_k$ as:
\begin{equation*}
    \sumnl_{t=t_k}^{t_{k+1}-1} \parens{
        \underbrace{\parens{p_k(S_t) - e_{S_t}} \mathfrak{h}_k}_{\text{navigation error } (1k)}
        +
        \underbrace{\parens{\hat{p}_k(S_t) - p_k(S_t)} h^*}_{\text{empirical bias error } (2k)}
        +
        \underbrace{\parens{\tilde{p}_k(S_t) - \hat{p}_k(S_t)} \mathfrak{h}_k}_{\text{optimistic overshoot } (3k)}
        +
        \underbrace{\parens{\hat{p}_k(S_t) - p_k(S_t)} (\mathfrak{h}_k - h^*)}_{\text{second order error } (4k)}
    }
\end{equation*}
Each error term is bounded separately.
Below, we denote $\Var(q, u) := \sum_s q(s)(u(s) - q \cdot u)^2$.

\begin{lemma}[Navigation error]
\label{lemma:navigation-error}
    With probability $1 - 7 \delta$, the navigation error is bounded by:
    \begin{equation*}
        \displaystyle
        \sumnl_{k} \sumnl_{t=t_k}^{t_{k+1}-1} 
        (p_k(S_t)-e_{S_t})\mathfrak{h}_k  
        \le 
        \!\!\sqrt{
            2 \sumnl_{t=0}^{T-1} \Var(p(X_t), h^*) \log\parens{\tfrac{T}\delta}
        }
        + 2 S A^{\frac 12} \!\!\sqrt{3B(T)} \log\parens{\tfrac T\delta} 
        + \widetilde{\OH}\parens{
            T^{\frac 7{20}} 
        }
        .
    \end{equation*}
\end{lemma}

\begin{lemma}[Empirical bias error]
\label{lemma:empirical-bias-error}
    With probability $1 - \delta$, the empirical bias error is bounded by:
    \begin{equation*}
        \sumnl_k \sumnl_{t=t_k}^{t_{k+1}-1} 
        \parens{\hat{p}_k(S_t) - p_k(S_t)} h^*
        \le 
        4 \sqrt{
            S A \sumnl_{t=0}^{T-1} \Var(p(X_t), h^*) \log\parens{\tfrac{S A T}\delta}
        }
        + 
        \OH\parens{\log^{2}(T)}
        .
    \end{equation*}
\end{lemma} 

\begin{lemma}[Optimism overshoot]
\label{lemma:optimism-overshoot}
    With probability $1 - 6\delta$, the optimism overshoot is bounded by:
    \begin{equation*}
        \sumnl_k \sumnl_{t=t_k}^{t_{k+1}-1} 
        \parens{\tilde{p}_k(S_t) - \hat{p}_k(S_t)} \mathfrak{h}_k
        \le 
        \begin{Bmatrix}
            4 \sqrt{
                2 S A \sumnl_{t=0}^{T-1} \Var(p(X_t), h^*) \log\parens{\tfrac{S A T}\delta}
            }
            \\
            + 8 (1 + c_0) S^{\frac 32} A \log^{\frac 32}\parens{\tfrac{S A T}\delta} \sqrt{B(T)}
            + \widetilde{\OH}\parens{T^{\frac 14}}
        \end{Bmatrix}
        .
    \end{equation*}
\end{lemma}

\begin{lemma}[Second order error]
\label{lemma:second-order-error}
    With probability $1 - 6\delta$, the second order error is bounded by:
    \begin{equation*}
        \sumnl_k \sumnl_{t=t_k}^{t_{k+1}-1} 
        \parens{\hat{p}_k(S_t) - p_k(S_t)} (\mathfrak{h}_k - h^*)
        \le 
        16 S^2 A (1 + c_0) \log^{\frac 12} \parens{\tfrac{S^2 A T}{\delta}} \sqrt{2 B(T)}
        + \widetilde{\OH}\parens{T^{\frac 14}}
        .
    \end{equation*}
\end{lemma}

We see that the empirical bias error (\cref{lemma:empirical-bias-error}) and the optimism overshoot (\cref{lemma:optimism-overshoot}) both involve the sum of variances $\sum_{t=0}^{T-1} \Var(p(X_t), h^*)$, which is shown in \cref{lemma:variance-sum} to be of order $\sp{h^*} \sp{r} T + \sum_{t=0}^{T-1} \Delta^*(X_t)$.
The pseudo-regret term $\sum_{t=0}^{T-1} \Delta^*(X_t)$ is bounded with the regret using \cref{corollary:regret-pseudo-regret}, then by $B(T)$.
With high probability, we obtain an equation of the form:
\begin{equation*}
    B(T)
    \le
    C \sqrt{(1+\sp{h^*}) SA T \log\parens{\tfrac T\delta}}
    + C S^2 A (1 + c_0) \log^2(T) \sqrt{B(T)}
    + \tilde{\OH}\parens{T^{\frac 14}}
\end{equation*}
where $C$ is a constant.
Setting $\alpha := C S^2 A (1 + c_0) \log^2(T)$ and $\beta := C \sqrt{(1+\sp{h^*}) S A T \log(T/\delta)} + \tilde\OH (T^{1/4})$, the above equation is of the form $B(T) \le \beta + \alpha \sqrt{B(T)}$.
Solving in $B(T)$, we find $B(T) \le \beta + 2 \sqrt{\alpha \beta} + \alpha^2$.
The dominant term is $\beta$, hence we readily obtain:
\begin{equation}
    B(T)
    \le
    C\sqrt{(1+\sp{h^*}) \sp{r} S A T \log\parens{\tfrac T\delta}}
    + \widetilde{\OH}\parens{
        \sp{h^*}\sp{r} S^{\frac 52} A^{\frac 32} (1 + c_0) T^{\frac 14}
    }
    .
\end{equation}  
Since $c_0 = \oh(T^{\frac 14})$, we conclude that $B(T) = \OH\parens{\!\sqrt{\sp{h^*} S A T \log(T/\delta)}}$, ending the proof.

\section{Experimental illustrations}

To get a grasp of how \texttt{PMEVI-DT} behaves in practice, we provide in \cref{figure:experiments} of few illustrative experiments.
In both experiments, the environment is a river-swim which is a model known to be hard to learn despite its size, with high diameter and bias span.
Its description is found in \cite{bourel_tightening_2020} and is reported in the appendix for self-containedness.

\begin{figure}[ht]
    \begin{tikzpicture}
        \node at (0,0) {\includegraphics[width=.49\linewidth]{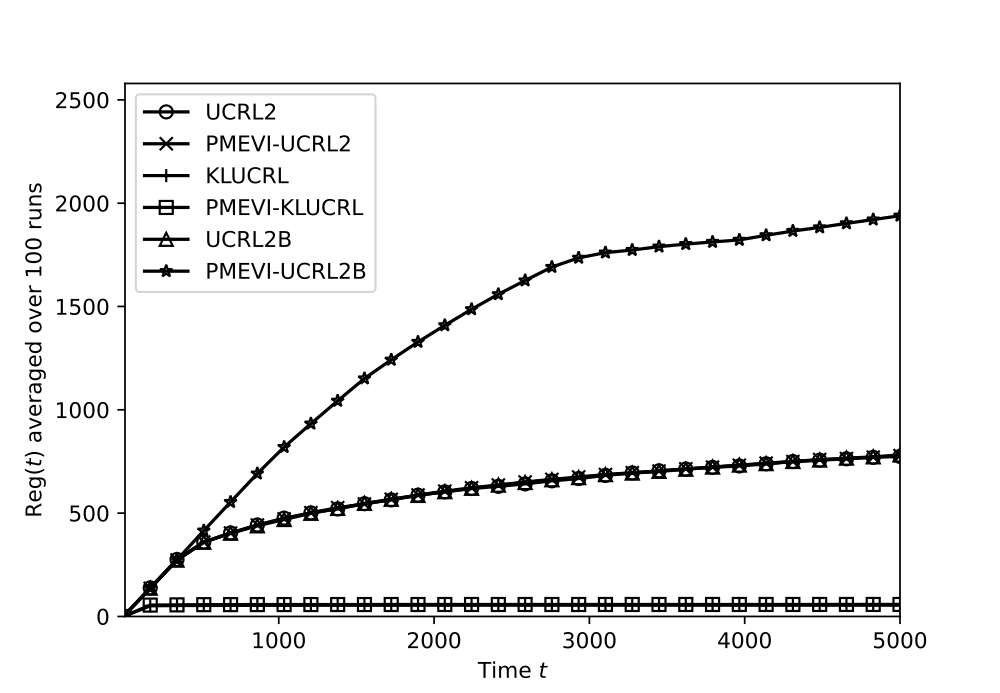}};
        \node[rotate=8] at (1.3, 0.9) {\scriptsize\texttt{UCRL2B \& PMEVI-UCRL2B}};
        \node[rotate=5] at (1, -.6) {\scriptsize\texttt{UCRL2 \& PMEVI-UCRL2}};
        \node at (1, -1.5) {\scriptsize\texttt{KLUCRL \& PMEVI-KLUCRL}};
    \end{tikzpicture}
    \hfill
    \begin{tikzpicture}
        \node at (0,0) {\includegraphics[width=.49\linewidth]{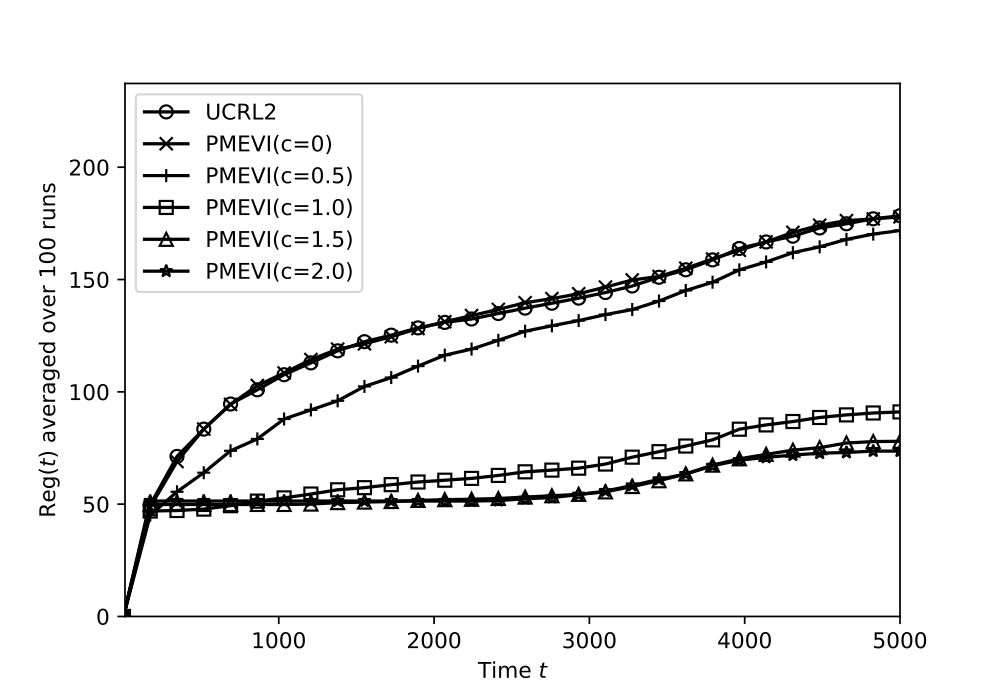}};
        \node[rotate=13] at (1.9, 1.0) {\scriptsize\texttt{UCRL2}};
        \node[rotate=3] at (2.0, -1.1) {\scriptsize\texttt{PMEVI(c=2)}};
        \node[rotate=13] at (1.9, 0.3) {\scriptsize\texttt{PMEVI(c=0.5)}};
        \node[rotate=9] at (1.9, -.3) {\scriptsize\texttt{PMEVI(c=1)}};
    \end{tikzpicture}
    \caption{
    \label{figure:experiments}
        (\textbf{To the left})
        Running a few algorithms of the literature on $5$-state river-swim and comparing their average regret against their \texttt{PMEVI} variants, obtained by changing calls to the \texttt{EVI} sub-routine to calls to \texttt{PMEVI}.
        (\textbf{To the right}) 
        Running \texttt{UCRL2} and \texttt{PMEVI-DT} with the same confidence region that \texttt{UCRL2} on a $3$-state river-swim. 
        \texttt{PMEVI-DT} is run with prior knowledge $h^*(s_1) \le h^*(s_2) - c \le h^*(s_3) - 2c$ for $c \in \set{0, 0.5, 1, 1.5, 2}$.
    }
\end{figure}
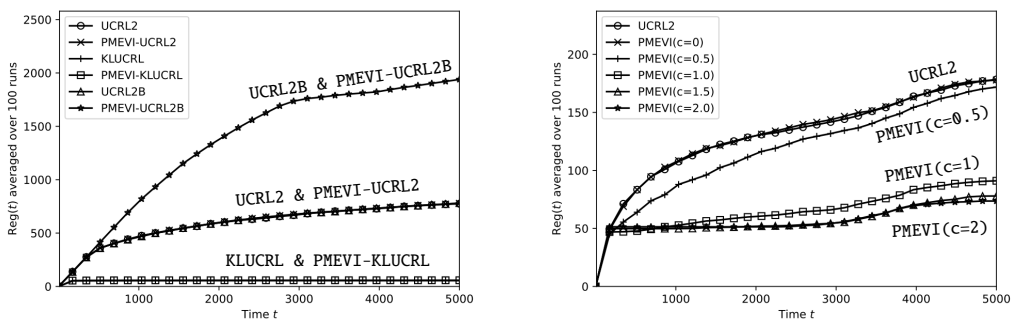

We observe on the first experiment that \texttt{PMEVI} behaves almost identically to its \texttt{EVI} counterparts when no prior on the bias region is given.
This is because most of the regret is due to the earlier learning phase, when bias information is impossible to get (the regret is still growing linearly and the bias estimator is off).
Accordingly, the bias confidence region is too large and all projections onto it are trivial during the iterations of \texttt{PMEVI}. 
Thankfully, this also makes the calls to \texttt{PMEVI} not substantially heavier than calls to \texttt{EVI} from a computational perspective.
On the second experiment, we measure the influence of prior bias information on the behavior of \texttt{PMEVI-DT}.
We observe that \texttt{PMEVI-DT} is very efficient at using adequate bias prior information to strikingly reduce the burn-in cost of the learning process on this $3$-state riverswim.



\bibliographystyle{plainnat}
\bibliography{bibliography.bib}

\newpage
\appendix
\allowdisplaybreaks


\addcontentsline{toc}{section}{Appendix} 
\part{Appendix} 
\parttoc 

\clearpage
\section{Construction of \texttt{PMEVI-DT}}

This section provides the technical details required to understand the design of \texttt{PMEVI-DT} in \cref{section:ptevi}.
We further discuss the assumptions \ref{assumption:model-confidence-region}-\ref{assumption:bias-region} appearing in \cref{theorem:main} and provide sufficient conditions so that they are met.

\subsection{Proof of \cref{lemma:bias-differences-error}, estimation of the bias error}

    Fix $s, s' \in \Sc$.
    We denote $\alpha_T := N_T (s \toto s') (h^*(s) - h^*(s') - c_T(s, s'))$.
    We will start by considering the better estimator $c'_T(s, s')$ that satisfies the same equation \eqref{equation:bias-difference-estimator} than $c_T(s,s')$ but with $\hat{g}(T)$ changed to $h^*$, readily:
    \begin{equation*}
        N_t (s \toto s') c'_T (s, s')
        =
        \sumnl_{t=0}^{N_T(s \toto s') -1}
        (-1)^i \sumnl_{t = \tau_i^{s \toto s'}}^{\tau_{i+1}^{s \toto s'}-1}
        (g^* - R_t)
        .
    \end{equation*}
    To avoid a typographical clutter, we write $\tau_i$ instead of $\tau_i^{s \toto s'}$ in the remaining of the proof and we write $\alpha_T' := N_T (s \toto s') (h^*(s) - h^*(s') - c'_T(s, s')$.
    
    \STEP{1}
    We start by relating the two estimators.
    Intuitively, $\hat{g}(T)$ is a good estimator for $g^*$ when the regret is small.
    Recall that $\hat{g}(T) := \frac 1T \sum_{t=0}^{T-1} R_t$, hence:
    \begin{equation*}
        \sumnl_{t=0}^{T-1} \abs{\hat{g}(T) - g^*}
        = \abs{\sumnl_{t=0}^{T-1} (R_t - g^*)}
        = \abs{\Reg(T)}
        .
    \end{equation*}
    Therefore,
    \begin{equation*}
        \abs{\alpha_T}
        \le \abs{\alpha_T'} + \abs{\alpha_T - \alpha_T'} 
        \le \abs{\alpha_T'} + \sumnl_{t=0}^{T-1} \abs{\hat{g}(T) - g^*}
        \le \abs{\alpha_T'} + \abs{\Reg(T)}
        .
    \end{equation*}
    We are left with upper-bounding $\abs{\alpha_T'}$.
    
    \STEP{2}
    If $i$ is even, then $S_{\tau_i}$ and $\S_{\tau_{i+1}} = s'$; otherwise $S_{\tau_i} = s'$ and $S_{\tau_{i+1}} = s$.
    In both cases, we have $h^*(S_{\tau_{i+1}}) - h^*(S_{\tau_i}) = (-1)^i (h^*(s') - h^*(s))$.
    Therefore, using Bellman's equation, the quantity $\text{A} := \sum_{t=\tau_i}^{\tau_{i+1}-1} (g^* - R_t)$ satisfies:
    \begin{align*}
        \text{A}
        & = 
        \sumnl_{t=\tau_i}^{\tau_{i+1}-1} \parens{p(X_t) - e_{S_t}} h^* 
        + \sumnl_{t=\tau_i}^{\tau_{i+1}-1} (r(X_t) - R_t)
        + \sumnl_{t=\tau_i}^{\tau_{i+1}-1} \Delta^*(X_t)
        \\
        & =
        \sumnl_{t=\tau_i}^{\tau_{i+1}-1} \parens{e_{S_{t+1}} - e_{S_t}} h^*
        + \sumnl_{t=\tau_i}^{\tau_{i+1}-1} \parens{p(X_t) - e_{S_{t+1}}} h^*
        + \sumnl_{t=\tau_i}^{\tau_{i+1}-1} (r(X_t) - R_t)
        + \sumnl_{t=\tau_i}^{\tau_{i+1}-1} \Delta^*(X_t)
        \\
        & = 
        (-1)^i (h^*(s') - h^*(s)) 
        + \sumnl_{t=\tau_i}^{\tau_{i+1}-1} \parens{p(X_t) - e_{S_{t+1}}} h^* 
        + \sumnl_{t=\tau_i}^{\tau_{i+1}-1} (r(X_t) - R_t)
        + \sumnl_{t=\tau_i}^{\tau_{i+1}-1} \Delta^*(X_t)
        .
    \end{align*}
    Multiplying by $(-1)^i$ and rearranging, $h^*(s') - h^*(s) + (-1)^{i+1} \sum_{t=\tau_i}^{\tau_{i+1}-1} (g^* - R_t)$ appears to be equal to:
    \begin{equation*}
        (-1)^{i+1}
        \parens{
            \sumnl_{t=\tau_i}^{\tau_{i+1}-1} \parens{
                \parens{p(X_t) - e_{S_{t+1}}} h^* + r(X_t) - R_t
            }
            + 
            \sumnl_{t=\tau_i}^{\tau_{i+1}-1} \Delta^*(X_t)
        }
        .
    \end{equation*}
    Proceed by summing over $i$.
    By triangular inequality, we obtain:
    \begin{equation*}
        \abs{\alpha_T'}
        \le
        \abs{
            \sumnl_{i=0}^{N_T(s\toto s')-1}
            \sumnl_{t=\tau_i}^{\tau_{i+1}-1}
            (-1)^{i+1} \parens{
                \parens{p(X_t) - e_{S_{t+1}}} h^*
                + r(X_t) - R_t
            }
        }
        + 
        \sumnl_{i=0}^{N_T(s\toto s')-1}
        \sumnl_{t=\tau_i}^{\tau_{i+1}-1}
        \Delta^*(X_t)
        .
    \end{equation*}
    Because all Bellman gaps $\Delta^*$ are non-negative, the second term is upper-bounded by the pseudo-regret $\sum_{t=0}^{T-1} \Delta^*(X_t)$.
    The first term is a martingale, and the martingale difference sequence $(-1)^{i+1} ((p(X_t) - e_{S_{t+1}})h^* + r(X_t) - R_t$ has span at most $\sp{h^*} + 1$ since rewards are supported in $[0, 1]$. 
    Although the number of involved is random, it is upper-bounded by $T$, hence by the maximal version of Azuma-Hoeffding's inequality (\cref{lemma:azuma}), we have that with probability at least $1 - \delta$, uniformly for $T' \le T$, 
    \begin{equation*}
        \abs{
            \sumnl_{i=0}^{N_{T'}(s\toto s')-1}
            \sumnl_{t=\tau_i}^{\tau_{i+1}-1}
            (-1)^{i+1} \parens{
                \parens{p(X_t) - e_{S_{t+1}}} h^*
                + r(X_t) - R_t
            }
        }
        \le
        (1 + \sp{h^*}) \sqrt{\tfrac 12 T \log\parens{\tfrac 2\delta}}.
    \end{equation*}
    
    \STEP{3}
    We conclude that with probability $1 - \delta$, for all $T' \le T$,
    \begin{equation*}
        \alpha_{T'} \le (1 + \sp{h^*}) \sqrt{\tfrac 12 T \log\parens{\tfrac 2\delta}} + \sumnl_{t=0}^{T'-1} \Delta^* (X_t) + \abs{\Reg(T')}.
    \end{equation*}
    We are left with relating both $\sum_{t=0}^{T'-1} \Delta^*(X_t)$ and $\abs{\Reg(T')}$ to $\sum_{t=0}^{T'-1} (\tilde{g} - R_t)$.
    Using the Bellman equation again, we find that:
    \begin{align*}
        \abs{
            \sumnl_{t=0}^{T'-1} (g^* - R_t - \Delta^*(X_t))
        }
        & \le
        \abs{h^*(S_0) - h^*(S_{T'})}
        + \abs{
            \sumnl_{t=0}^{T'-1} \parens{
                \parens{p(X_t) - e_{S_{t+1}}} h^*
                + \parens{r(X_t) - R_t}
            }
        }
        \\
        & \le 
        \sp{h^*} + 
        (1 + \sp{h^*}) \sqrt{\tfrac 12 T \log\parens{\tfrac 2\delta}}
    \end{align*}
    where the last inequality holds with probability $1 - \delta$ uniformly over $T' \le T$ by Azuma-Hoeffding's inequality again (\cref{lemma:azuma}).
    Remark that if $y - z \le x \le y + z$, then $\abs{x} \le \abs{y} + \abs{z}$, hence we conclude that with probability $1 - \delta$, for all $T' \le T$:
    \begin{align*}
        \sumnl_{t=0}^{T'-1} \Delta^*(X_t) + \abs{\Reg(T')}
        & \le
        2 \sumnl_{t=0}^{T'-1} \Delta^*(X_t) 
        + (1 + \sp{h^*}) \sqrt{\tfrac 12 T \log\parens{\tfrac 2\delta}} + \sp{h^*}
        \\
        & \le 
        2 \sumnl_{t=0}^{T'-1} \parens{g^* - R_t}
        + 3 (1 + \sp{h^*}) \sqrt{\tfrac 12 T \log\parens{\tfrac 2\delta}} + 3 \sp{h^*}
        \\
        & \le 
        2 \sumnl_{t=0}^{T'-1} \parens{\tilde{g} - R_t}
        + 3 (1 + \sp{h^*}) \sqrt{\tfrac 12 T \log\parens{\tfrac 2\delta}} + 3 \sp{h^*}
    \end{align*}
    where the last inequality invokes $\tilde{g} \ge g^*$.
    We conclude that, with probability $1 - 2\delta$, for all $T' \le T$, we have:
    \begin{equation*}
        N_{T'} (s \toto s') (h^*(s) - h^*(s') - c_{T'}(s, s'))
        \le
        3 \sp{h^*}
        + \parens{1 + \sp{h^*}} \sqrt{8 T \log\parens{\tfrac 2\delta}} + \sumnl_{t=0}^{T'-1} (\tilde{g} - R_t)
        .
    \end{equation*}
    This concludes the proof.
    \qed

\subsection{The confidence region of \texttt{PMEVI-DT}}
\label{section:model-confidence-region}

The algorithm \texttt{PMEVI-DT} can be instantiated with a large panel of possibilities, depending on the type of confidence region one is willing to use for rewards and kernels.
In this work, we allow for four types of confidence regions, described below.
For conciseness, $q \in \set{r, p}$ is a symbolic letter that can be a reward or a kernel and denote $\Qc_t (s,a)$ the confidence region for $q(s,a)$ at time $t$.
If $q = r$, then $\dim(q) = 2$ (Bernoulli rewards) with $\Qc(s,a) = [0,1]$; and if $q = p$, then $\dim(q) = S$ with $\Qc(s,a) = \Pc(\Sc)$.

\begin{itemize}
    \item[(\textbf{C1})] 
        \emph{Azuma-Hoeffding} or \emph{Weissman} type confidence regions, with $\Qc_t (s,a)$ taken as:
        \begin{equation*}
            \set{
                \tilde{q}(s,a) \in \Qc(s,a): 
                N_t (s,a) \norm{\hat{q}_t(s,a) - \tilde{q}(s,a)}_1^2
                \le 
                \dim(q)\log\parens{\tfrac{2 S A (1+N_t(s,a))}\delta}
            }
            .
        \end{equation*}

    \item[(\textbf{C2})]
        \emph{Empirical Bernstein} type confidence regions, with $\Qc_t (s,a)$ taken as:
\end{itemize} 
        \begin{equation*}
            \set{
            \tilde{q}(s,a) \in \Qc(s,a):
            \forall i, 
            \abs{\hat{q}_t(i|s,a) - \tilde{q}(i|s,a)}
            \le
            \sqrt{\tfrac{2 \Var(\hat{q}_t(i|s,a)) \log\parens{\tfrac {2 \dim(q) S A T}\delta}}{N_t(s,a)}}
            + 
            \tfrac{3 \log\parens{\tfrac {2 \dim(q) S A T}\delta}}{N_t(s,a)}
        }
        .
        \end{equation*}
\begin{itemize}
    \item[]
        with the convention that $x/0 = +\infty$ for $x > 0$.

    \item[(\textbf{C3})]
        \emph{Empirical likelihood} type confidence regions, with $\Qc_t (s,a)$ taken as:
\end{itemize}
        \begin{equation*}
            \set{
                \tilde{q}(s,a) \in \Qc(s,a) :
                N_t(s,a) \KL(\hat{q}_t(s,a)\|\tilde{q}(s,a))
                \le
                \log\parens{\tfrac {2 S A}\delta}
                + (\dim(q)-1) \log\parens{e\parens{1 + \tfrac{N_t(s,a)}{\dim{q}-1}}}
            }
            .
        \end{equation*}
\begin{itemize}
    \item[(\textbf{C4})]
        \emph{Trivial} confidence region with $\Qc_t(s,a) = \Qc(s,a)$.
\end{itemize}

A few remarks are in order.
When rewards are not Bernoulli, only the confidence regions (\textbf{C1}) and (\textbf{C4}) are elligible among the above. 
Then, Weissman's inequality must be changed to Azuma's inequality for $\sigma$-sub-Gaussian random variables, see \cref{lemma:time-uniform-azuma}.
Since rewards are supported in $[0, 1]$, Hoeffding's Lemma guarantees that reward distributions are $\sigma$-sub-Gaussian with $\sigma = \tfrac 12$.

\subsubsection{Correctness of the model confidence region $\Mc_t$ and \cref{assumption:model-confidence-region}}

The confidence regions $\Qc_t (s,a)$ described with (\textbf{C1-4}) are tuned so that the following result holds:

\begin{lemma}
\label{lemma:confidence-region}
    Assume that, for all $q \in \set{r, p}$ and $(s, a) \in \Xc$, we choose $\Qc_t(s,a)$ among (\textbf{C1-4}).
    Then \cref{assumption:model-confidence-region} holds.
    More specifically, the region of models $\Mc_t := \prod_{s,a} (\Rc_t(s,a) \times \Pc_t(s,a))$ satisfies $\Pr(\exists t \le T: M \notin \Mc_t) \le \delta$.
\end{lemma}
\begin{proof}
    We show that, for all $q \in \set{r, q}$ and $(s,a) \in \Xc$, if $\Qc_t(s,a)$ is chosen amoung (\textbf{C1-4}), then
    \begin{equation*}
        \Pr \parens{
            \exists t \le T:
            q(s,a) \notin \Qc_t(s,a)
        }
        \le \delta
        .
    \end{equation*}
    If $\Qc_t(s,a)$ is chosen with (\textbf{C1}), this is a direct application of \cref{lemma:uniform-weissman}; with (\textbf{C2}), this is \cref{lemma:uniform-bernstein}; with (\textbf{C3}), this is \cref{lemma:uniform-kl}; and with (\textbf{C4}) this is by definition.
\end{proof}

\subsubsection{Simultaneous correctness of bias confidence region $\Hc_t$, mitigation $\beta_t$ and optimism}

In this section, we show that if \cref{assumption:model-confidence-region} holds, then the bias confidence region constructed by \texttt{PMEVI-DT} is correct with high probability, and that the mitigation is not too strong. 
Recall that $(\mathfrak{g}_k, \mathfrak{h}_k)$ are the optimistic gain and bias of the policy deployed in episode $k$ (see \hyperlink{algorithm:PMEVI-dt}{Algorithm~1}).
In particular, we have $\mathfrak{g}_k = \Lfk_{t_k} \mathfrak{h}_k - \mathfrak{h}_k$ with $\mathfrak{h}_k \in \Hc_{t_k}$.
We start by a result on the deviation of the variance, which is what the variance approximation \hyperlink{algorithm:variance-approximation}{Algorithm~5} is based on.
Recall that the bias confidence region $\Hc_t$ is obtained as the collection of constraints:
\vspace{-.5em}
\begin{itemize}
    \item[(1)] prior constraints (if any) $\mathfrak{h}(s) - \mathfrak{h}(s') \le c_*(s,s')$;
    \item[(2)] span constraints $\mathfrak{h}(s) - \mathfrak{h}(s') \le c_0 := T^{1/5}$;
    \item[(3)] dynamically infered constraints $\abs{\mathfrak{h}(s) - \mathfrak{h}(s') - c_t(s',s)} \le \text{error}(c_t, s', s)$ (see \hyperlink{algorithm:bias-constraints}{Algorithm~3}).
\end{itemize}
We have the following result.

\begin{lemma}
\label{lemma:variance-approximation}
    Let $u, v \in \Hc_t$ and fix $p$ a probability distribution on $\Sc$.
    Then for all $s \in \Sc$,
    \begin{equation*}
        \Var(p, u) \le \Var(p, v) + 8 c_0 \sumnl_{s' \in \Sc} p(s') ~\text{\upshape error}(c_t, s', s)
        .
    \end{equation*}
\end{lemma}
\begin{proof}
    We start by establishing the following result: If $p$ is a probability distribution on $\Sc$ and $u, v \in \R^\Sc$, we have:
    \begin{equation}
    \label{equation:variance-deviations}
        \Var(p, u) \le \Var(p, v) + 2 \parens{p \cdot \abs{u - v}} \max(u + v)
    \end{equation}
    where $\cdot$ is the dot product, $u^2$ the Hadamard product $uu$ and $\abs{u}$ the vector whose entry $s$ is $\abs{u(s)}$.
    \eqref{equation:variance-deviations} is obtained with a straight forward computation:
    \begin{align*}
        \Var(p, u) - \Var(p, v)
        & = p \cdot(u^2 - v^2) + (p \cdot v)^2 - (p \cdot u)^2
        \\
        & = p \cdot ((u-v)(u+v)) + (p \cdot (u-v)) (p \cdot (u+v)) 
        \\
        & \le p \cdot (\abs{u - v}(u + v)) + (p \cdot \abs{u-v}) (p \cdot \abs{u+v})
        \\
        & \le 2 (p \cdot \abs{u - v}) \max(u + v)
        .
    \end{align*}
    Observe that $v$ can be changed to $v + \lambda e$, where $e$ is the vector full of ones, without changing the result. 
    The same goes for $u$.
    We now move to the proof of the main statement.
    First, translate $u$ and $v$ such that $u(s) = v(s) = 0$.
    Then, we have:
    \begin{align*}
        p \cdot (u - v)
        & = \sumnl_{s' \in \Sc} p(s') \abs{u(s') - u(s) - c_t(s', s) + v(s) - v(s') + c_t(s', s)}
        \\
        & \le \sumnl_{s' \in \Sc} p(s') \parens{
            \abs{u(s') - u(s) - c_t(s',s)}
            +
            \abs{v(s') - v(s) - c_t(s',s)}
        }
        \\
        & \le 2 \sumnl_{s' \in \Sc} p(s') ~ \text{error}(c_t, s', s)
        .
    \end{align*} 
    Conclude using that $\max(u + v) \le \max(u) + \max(v) + 2 c_0$ for $u, v \in \Hc$ such that $u(s) = v(s) = 0$.
\end{proof}

\begin{lemma}
\label{lemma:bias-confidence-region}
    Assume that \cref{assumption:model-confidence-region} holds and that $c_0 \ge \sp{h^*}$.
    Then, with probability $1 - 4\delta$, for all $k \le K(T)$, (1) $\mathfrak{g}_k \ge g^*$ and (2) $h^* \in \Hc_{t_k}$ and (3) for all $(s, a)$, $(\hat{p}_{t_k}(s,a) - p(s,a)) h^* \le \beta_{t_k}(s,a)$.
\end{lemma}
\begin{proof}
    Let $E_1$ the event $(\forall k \le K(T), M \in \Mc_{t_k})$.
    Let $E_2$ the event stating that, for all $T' \le T$,
    \begin{equation*}
        N_{T'} (s \toto s') \abs{
            h^*(s) - h^*(s') - c_{T'} (s, s')
        }
        \le
        3 \sp{h^*} +
        (1 + \sp{h^*}) \sqrt{8 T \log(\tfrac 2\delta)}
        + 2 \sumnl_{t=0}^{T'-1} (\tilde{g} - R_t)
        ,
    \end{equation*}
    and let $E_3$ the event stating that, for all $T' \le T$ and for all $(s,a) \in \Xc$, we have:
    \begin{equation*}
        \parens{\hat{p}_{T'}(s,a) - p(s,a)}h^*
        \le
        \sqrt{
            \tfrac{2 \Var(\hat{p}_{T'}(s,a), h^*) \log\parens{\frac{S A T}\delta}}{N_{T'}(s,a)}
        }
        + 
        \tfrac{3 \sp{h^*} \log\parens{\frac{S A T}\delta}}{N_{T'}(s,a)}
        .
    \end{equation*}
    
    By \cref{lemma:bias-differences-error}, we have $\Pr(E_2) \ge 1 - 2 \delta$ and by \cref{lemma:uniform-bernstein}, we have $\Pr(E_3) \ge 1 - \delta$, so $\Pr(E_1 \cap E_2 \cap E_3) \ge 1 - 4\delta$.
    We prove by induction on $k \le K(T)$ that, on $E_1 \cap E_2$, (1) $\mathfrak{g}_k \ge g^*$, (2) $h^* \in \Hc_{t_k}$ (3) and for all $(s, a)$, $(\hat{p}_{t_k}(s,a) - p(s,a)) h^* \le \beta_{t_k}(s,a)$, where $\mathfrak{g}_k$ is the optimistic gain of the policy deployed at episode $k$. 
    For $k = 0$, this is obvious.
    Indeed, $N_0 (s \toto s') = 0$ for all $s, s'$ hence $c_0(s,s') = c_0 \ge \sp{h^*}$.
    Therefore,
    \begin{equation*}
        \Hc_0 
        \supseteq \set{\mathfrak{h} \in \R^\Sc: \sp{\mathfrak{h}} \le c_0}
        \supseteq \set{\mathfrak{h} \in \R^\Sc: \sp{\mathfrak{h}} \le \sp{h^*}} 
    \end{equation*}
    so contains $h^*$, proving (2).
    Moreover, since $N_0 (s, a) = 0$, we have $\beta_0 (s, a) = +\infty$, proving (3).
    Finally, since $M \in \Mc_0$ on $E_1$, by the statement (2) of \cref{proposition:PMEVI}, we have $\mathfrak{g}_k \ge g^*$, hence proving (1).

    Now assume that $k \ge 1$.
    By induction $\mathfrak{g}_\ell \ge g^*$ for all $\ell < k$, so on $E_2$ we have:
    \begin{equation*}
         N_{t_k} (s \toto s') \abs{
            h^*(s) - h^*(s') - c_{t_k} (s, s')
        }
        \le
        3 \sp{h^*} +
        (1 + \sp{h^*}) \sqrt{8 T \log(\tfrac 2\delta)}
        + 2 \sumnl_{\ell=1}^{k-1} \sumnl_{t=t_\ell}^{t_{\ell+1}-1} (\mathfrak{g}_\ell - R_t)
        .
    \end{equation*}
    By design of $\Hc_{t_k}$ (see \hyperlink{algorithm:bias-estimation}{Algorithm~3}), we deduce that (2) $h^* \in \Hc_{t_k}$. 
    Denote $h_0 \in \Hc_{t_k}$ the reference point used by \hyperlink{algorithm:variance-approximation}{Algorithm~5}.
    We have, for all $(s,a) \in \Xc$, on $E_1 \cap E_2 \cap E_3$, we have:
    \begin{align*}
        \parens{\hat{p}_{t_k}(s,a) - p(s,a)}h^*
        & \le 
        \sqrt{
            \tfrac{2 \Var(\hat{p}_{t_k}(s,a), h^*) \log\parens{\frac{S A T}\delta}}{N_{t_k}(s,a)}
        }
        + 
        \tfrac{3 \sp{h^*} \log\parens{\frac{S A T}\delta}}{N_{t_k}(s,a)}
        \\
        \text{($h^* \in \Hc_{t_k}$ + \cref{lemma:variance-approximation})} & \le 
        \sqrt{
            \tfrac{2 \parens{
                \Var(\hat{p}_{t_k}(s,a), h_0) \log\parens{\frac{S A T}\delta}
                + 
                8 c_0 \sum_{s' \in \Sc} \hat{p}_{t_k}(s'|s,a) ~\text{error}(c_{t_k}, s', s)
            } \log\parens{\frac{S A T}\delta} }{N_{t_k}(s,a)}
        }
        +
        \tfrac{3 c_0 \log\parens{\frac{S A T}\delta}}{N_{t_k}(s,a)}
        \\
        & =:
        \beta_{t_k} (s, a)
    \end{align*}
    by construction of \hyperlink{algorithm:variance-approximation}{Algorithm~5}.
    Accordingly, (3) is satisfied.
    Finally, $M \in \Mc_{t_k}$ on $E_1$ so by \cref{proposition:PMEVI}, we have (1) $\mathfrak{g}_k \ge g^*$. 
\end{proof}

\begin{corollary}
    Assume that, for all $q \in \set{r, p}$ and $(s, a) \in \Xc$, we choose $\Qc_t(s,a)$ among (\textbf{C1-4}).
    Then, with probability $1 - 3\delta$, for all $k \in K(T)$, we have $\mathfrak{g}_k \ge g^*$ and (2) $h^* \in \Hc_{t_k}$ and (3) for all $(s, a)$, $(\hat{p}_{t_k}(s,a) - p(s,a)) h^* \le \beta_{t_k}(s,a)$.
\end{corollary}

\begin{proof}
    By \cref{lemma:confidence-region}, \cref{assumption:model-confidence-region} is satisfied.
    Apply \cref{lemma:bias-confidence-region}.
\end{proof}

\subsubsection{Sub-Weissman reward confidence region and \cref{assumption:sub-weissman}}
\label{section:sub-weissman}

Although the kernel confidence region can even chosen to be trivial with (\textbf{C4}), in order to work, \texttt{PMEVI-DT} needs the reward confidence region to be sub-Weissman in the following sense:

\paragraph{Assumption~2.}
    There exists a constant $C > 0$ such that for all $(s, a) \in \Sc$, for all $t \le T$, we have:
    \begin{equation*}
        \Rc_t (s,a) 
        \subseteq 
        \set{
            \tilde{r}(s,a) \in \Rc(s,a):
            N_t (s,a) \norm{\hat{r}_t(s,a) - \tilde{r}(s,a)}_1^2
            \le 
            C \log\parens{\tfrac{2 S A (1+N_t(s,a))}\delta}
        }
        .
    \end{equation*}

This is indeed the case if $\Rc_t (s,a)$ is chosen among (\textbf{C1-3}).

\subsection{Convergence of \texttt{EVI} and \cref{assumption:evi-convergence}}
\label{section:evi-convergence}

We start with a preliminary lemma on the speed of convergence of \texttt{EVI}.
The \cref{lemma:evi-iterations-bound} is thought to be applied to extended MDPs.
Below, when we claim that the action space is compact, we further claim that $a \in \Ac(s) \mapsto p(s, a)$ is a continuous map, so that the Bellman operator is continuous and that $g^*$ and $h^*$ are well-defined, see \cite{puterman_markov_1994}.

\begin{lemma}
\label{lemma:evi-iterations-bound}
    Let $M$ a weakly-communicating MDP with finite state space $\R^\Sc$ and compact action space, and let $L$ its Bellman operator.
    Assume that there exists $\gamma > 0$ such that, $\forall u \in \R^\Sc$,
    \begin{equation}
    \tag{$*$}
        \forall s \in \Sc, \exists a \in \Ac(s),
        \quad
        Lu(s) = r(s,a) + p(s,a)u = r(s,a) + \gamma \max(u) + (1 - \gamma) q_s^u u
    \end{equation}
    with $q_s^u \in \Pc(\Sc)$.
    Then, for all $u \in \R^\Sc$ and all $\epsilon > 0$, if $\sp{L^{n+1} u - L^n u} \ge \epsilon$, then:
    \begin{equation*}
        n \le 2 + \frac{4\sp{w_0}}{\gamma \epsilon} + \frac{2}{\gamma} \log\parens{\frac{2\sp{w_0}}{\epsilon}}
        .
    \end{equation*}
\end{lemma}

\begin{proof}
    Since $M$ is weakly communicating, has finitely many states and compact action space, it has well-defined gain $g^*$ and bias $h^*$ functions.
    Denote $u_{n+1} := L^n u$.
    \begin{align*}
        w_n 
        & := \max_{\pi \in \Pi} \set{r_\pi + P_\pi u_{n-1}} - ng^* - h^*
        \\
        & = \max_{\pi \in \Pi} \set{
            r_\pi - g^* + (P_\pi - I) h^* 
            + 
            P_\pi \parens{u_{n-1} - h^* - (n-1) g^*}
        }
        =: \max_{\pi \in \Pi} \set{r'_\pi +  P_\pi w_{n-1}}
        .
    \end{align*}
    Observe that the policy achieving the maximum is the one achieving $u_n = r_\pi + P_\pi u_{n-1}$.
    Remark that $r'_\pi(s) = - \Delta^*(s, \pi(s)) \le 0$ is the Bellman gap of the pair $(s, \pi(s))$, that we more simply write $\Delta_{\pi}$.
    For all $n$, there exists $\pi_n \in \Pi$ such that $w_{n+1} = - \Delta_{\pi_n} + P_{\pi_n} w_n$.
    Moreover, by assumption, we have $P_{\pi_n} = \gamma \cdot e_{s_n}^\top e + (1 - \gamma) Q_n$ where $Q_n$ is a stochastic matrix.
    Moreover,
    \begin{equation*}
        \parens{\min(-\Delta_{\pi_n}) + \gamma w_{n}(s_n)} e + (1 - \gamma) Q_n w_n
        \le
        w_{n+1}
        \le 
        \parens{\max(-\Delta_{\pi_n}) + \gamma w_{n}(s_n)} e + (1 - \gamma) Q_n w_n
        .
    \end{equation*}
    Hence, $\sp{w_{n+1}} \le (1 - \gamma) \sp{w_n} + \sp{\Delta_{\pi_n}}$.
    In addition, $w_n = L^n u - L^n h^*$, so by non-expansiveness of $L$ in span semi-norm, $\sp{w_{n+1}} \le \sp{w_n}$. 
    Overall,
    \begin{equation}
    \label{equation:span-w}
        \sp{w_{n+1}} \le \min \parens{
            (1 - \gamma) \sp{w_n} + \sp{\Delta_{\pi_n}},
            \sp{w_n}
        }.
    \end{equation}
    Fix $\epsilon > 0$, and let $n_\epsilon := \inf \set{n:\sp{w_n} < \epsilon}$.

    Let $\pi^*$ an optimal policy.
    We have $w_{n+1} \ge P_{\pi^*} w_n$ so by induction, $w_{n+1} \ge P_{\pi^*}^{n+1} w_0 \ge \min(w_0) e$.
    Meanwhile, we see that $\norm{w_n}_1 \ge \sum_{k=0}^{n-1} \norm{\Delta_{\pi_k}}_1 + S \min(w_0)$, so $\sum_{k=0}^{n-1} \norm{\Delta_{\pi_k}}_1 \le \sp{w_0}$.
    Since $\Delta_{\pi_k} \le 0$ for all $k$, we have $\sp{\Delta_{\pi_k}} \le \norm{\Delta_{\pi_k}}_1$ so $\sum_{k=0}^{n-1} \sp{\Delta_{\pi_k}} \le \sp{w_0}$.

    By \eqref{equation:span-w}, either $\sp{w_{n+1}} \le (1 - \tfrac 12 \gamma) \max(\epsilon, \sp{w_n})$ or $\sp{\Delta_{\pi_n}} \ge \tfrac 12 \gamma \epsilon$, but because $\sum_{k=0}^{+\infty}\sp{\Delta_{\pi_k}} \le \sp{w_0}$, the second case can happen at most $\tfrac {2\sp{w_0}}{\gamma \epsilon}$ times.
    We deduce that, for all $n \le n_\epsilon$,
    \begin{equation*}
        \sp{w_{n+1}} \le \parens{1 - \tfrac 12\gamma}^{n - \frac{2\sp{w_0}}{\gamma \epsilon}} \sp{w_0}.
    \end{equation*}
    In particular, for $n = n_\epsilon - 1$, we get:
    \begin{equation*}
        \epsilon \le 
        \parens{1 - \tfrac 12\gamma}^{n_\epsilon - 2 - \frac{2\sp{w_0}}{\gamma \epsilon}} \sp{w_0}
        .
    \end{equation*}
    We obtain:
    \begin{equation*}
        n_\epsilon \le
        2 
        + \frac{2 \sp{w_0}}{\gamma \epsilon}
        + \frac 2\gamma \log\parens{\frac {\sp{w_0}}\epsilon}
        .
    \end{equation*}
    To conclude, check that $\sp{L^{n+1} u - L^n u} = \sp{w_{n+1} - w_n} \le 2 \sp{w_n}$.
\end{proof}

Before moving to the application of interest, remark that this result can be greatly improved if the supremum $\sup \set{\Delta^*(s,a) : \Delta^*(s,a) < 0}$ is not zero, to change the dominant term $\frac{4 \sp{w_0}}{\gamma \epsilon}$ for a constant independent of $\epsilon$.

\begin{corollary}
\label{corollary:evi-convergence}
    Assume that the $\Mc_t$ has non-empty interior, and that its Bellman operator satisfies the requirement of \cref{lemma:evi-iterations-bound}, i.e., there exists $\gamma > 0$ such that, $\forall u \in \R^\Sc, \forall s \in \Sc, \exists a \in \Ac(s), \exists \tilde{r}_t(s,a) \in \Rc_t(s,a), \exists \tilde{p}_t(s,a) \in \Pc_t(s,a)$:
    \begin{equation*}
        \Lc_t u(s) = \tilde{r}_t(s,a) + \tilde{p}_t(s,a)u = \tilde{r}_t(s,a) + \gamma \max(u) + (1 - \gamma) q_s^u u
    \end{equation*}
    for some $q_s^u \in \Pc(\Sc)$.
    Then \cref{assumption:evi-convergence} is satisfied, and span fix-points $\tilde{h}_t$ of $\Lc_t$ are such that $g^*(\Mc_t) = \Lc_t \tilde{h}_t - \tilde{h}_t$.
\end{corollary}
\begin{proof}
    If $\Mc_t$ is has non-empty interior, it means that for all $(s, a)$, $\Pc_t(s,a)$ has non-empty interior. 
    Therefore, for all state-action pair, there exists $\tilde{p}_t(s,a) \in \Pc_t(s,a)$ that is fully supported.
    It follows that $\Mc_t$ is communicating, and it follows from standard results \cite{puterman_markov_1994} that its span fix-points $\tilde{h}$ do exist and that $\tilde{g}_t := \Lc\tilde{h}_t - \tilde{h}_t \in \R e$ does not depend on the initial state.
    
    Moreover, if $\widetilde{M} \in \Mc_t$ and $\pi \in \Pi$ with $\tilde{g}_\pi \equiv g(\pi, \Mc_t) \in \R e$, letting $\tilde{r}_\pi := r_\pi(\tilde{M})$ and $\tilde{P}_\pi := P_\pi(\tilde{M})$, we have:
    \begin{equation*}
        \tilde{r}_\pi + \tilde{p}_\pi \tilde{h}_t \le \Lc_t \tilde{h}_t \le \tilde{g}_t e + \tilde{h}_t
        .
    \end{equation*}
    So by induction and since $\Lc_t$ is obviously monotone and linear, we show that:
    \begin{equation*}
        \sum_{k=0}^n \tilde{P}_\pi^k \tilde{r}_\pi
        \le
        n \tilde{g}_t e + (I - \tilde{P}_\pi^n) \tilde{h}_\pi
        .
    \end{equation*}
    Dividing by $n$ and letting it go to infinity, we obtain $g(\pi, \Mc_t) \le \tilde{g}_t$.
    Observe that we have equility by taking the policy achieving $(\tilde{g}_t, \tilde{h}_t)$.

    To see that \texttt{EVI} converges indeed, simply observe that \cref{lemma:evi-iterations-bound} provides a finite bound on how much time is required until the $\sp{\Lc_t^{n+1} u - \Lc_t^n u} \le \epsilon$.
    Hence $\sp{\Lc_t^{n+1} u - \Lc_t^n u}$ vanishes to $0$.
\end{proof}

\paragraph{About \cref{assumption:evi-convergence}.}
The assumptions made by \cref{corollary:evi-convergence} are met if the kernel confidence regions are:
\begin{itemize}
    \item Built out of Weissman's inequality (\textbf{C1}) (see the next section, also \cite{auer_near-optimal_2009});
    \item Built out of Bernstein's inequality (\textbf{C2}) (because the maximization algorithm to compute $\tilde{p}_t(s,a) u_i$ in \texttt{EVI} has the same greedy properties than with Weissman's inequality);
    \item Trivial (\textbf{C4}) obviously.
\end{itemize}
For confidence regions build with empirical likelihood estimates (\textbf{C3}), there is no guarantee of convergence (although we conjecture that one could be established), although the gain is still well-defined because $\Mc_t$ remains communicating.
However, just like the original work of \cite{filippi_optimism_2010}, the convergence is always met numerically.

\subsection{Proof of \cref{theorem:main}: Complexity of \texttt{PMEVI} with Weissman confidence regions}

In this section, we show that when one is using Weissman confidence regions for kernels (\textbf{C1}), then the iterates of $\Lc_t$ converge to an $\epsilon$ span-fix-point quickly.

\begin{proposition}
\label{proposition:PMEVI-complexity}
    Assume that \texttt{\upshape PMEVI-DT} uses kernel confidence regions of Weissman type (\textbf{C1}) satisfying \cref{assumption:model-confidence-region}.
    Then with probability $1 - \delta$, the number of iterations of \texttt{\upshape PMEVI} (see \hyperlink{algorithm:PMEVI}{Algorithm~2}) is $\OH\parens{D \!\!\sqrt{S} A T}$, hence the algorithm has polynomial per-step amortized complexity.
\end{proposition}
\begin{proof}
    With Weissman type confidence regions for kernels, for all $t \le T$ and $(s,a) \in \Xc$, we have
    \begin{equation*}
        \Pc_t(s,a) \supseteq \set{
            \tilde{p}(s,a) \in \Pc(s,a):
            \norm{\tilde{p}(s,a) - \hat{p}_t(s,a)}_1 
            \le \sqrt{
                \frac{S \log(2SAT)}T
            }
        }
    \end{equation*}
    It follows that, for all $t \le T$, the extended Bellman operator $\Lc_t$ satisfies the prerequisite $(*)$ of \cref{lemma:evi-iterations-bound} with
    \begin{equation*}
        \gamma = \frac 12 \sqrt{
            \frac{S \log(2SAT/\delta)}T
        }
        = \Omega\parens{\!\!\sqrt{\frac{S\log(T/\delta)}T}}
        .
    \end{equation*}
    Under \cref{assumption:model-confidence-region}, we have $M \in \Mc_t$ with probability $1 - \delta$.
    Under this event, $\Mc_t$ is weakly communicating and $\sp{h^*(\Mc_t)} \le D(M)$, we can apply \cref{lemma:evi-iterations-bound} and conclude that every calls to \texttt{PMEVI} (\hyperlink{algorithm:PMEVI}{Algorithm~2}) takes
    \begin{equation*}
        \OH\parens{\frac{\sp{w_0} \sqrt{T}}{\epsilon \sqrt{\frac{S\log(T/\delta)}T}}}
        =
        \OH\parens{
            \frac{
                D T
            }{
                \sqrt{S} \log(T)
            }
        }
    \end{equation*}
    where we use that $\epsilon = \sqrt{\tfrac{\log(SAT/\delta)}{T}}$, that $\sp{w_0} = \OH\parens{\sp{h^*(\Mc_t)}} = \OH(D(M))$ and that $\delta \ge \frac 1T$.
    Since the number of episodes under the doubling trick (DT) is $\OH(SA\log(T))$, we conclude accordingly.
\end{proof}

Every call to the projection operator solves a linear program. 
Although in theory, this time is polynomial (relying on recent work on the complexity of LP such as \cite{cohen2020solving}, it is the current matrix multiplication time $\OH(S^{2.38})$), in practice, reducing the number of calls to the projection operator is key to run \texttt{PMEVI-DT} in reasonable time.

\newpage
\section{Analysis of the projected mitigated Bellman operator}

In this section, we fix the model region $\Mc$, the bias region $\Hc$ and the mitigation vector $\beta$, dropping the sub-script $t$ for conciseness.
We denote $\hat{r}, \hat{p}$ the respective empirical reward and kernel.
Further assume that $\Hc = \Hc_0 + \R e$ with $\Hc_0$ a compact convex set.
The associated projection operation (see \cref{section:projection-operation}) is denoted $\Gamma$.
The (vanilla) extended Bellman operator $\Lc$ associated to $\Mc$ is given by $\Lc u(s) := \max_{a \in \Ac(s)} \set{\sup \Rc(s,a) + \sup \Pc(s,a) u}$.
The \emph{$\beta$-mitigated extended Bellman operator} associated to $\Mc$ is:
\begin{equation}
    \Lc^\beta u(s)
    :=
    \max_{a \in \Ac(s)} \sup_{\tilde{r}(s,a) \in \Rc(s,a)} \sup_{\tilde{p}(s,a) \in \Pc(s,a)} \Big\{
        \tilde{r}(s,a) + \min \set{
            \tilde{p}(s,a) u_i,
            \hat{p}(s,a) u_i + \beta(s,a)
        }
    \Big\}
    .
\end{equation}
The function Greedy$(\Mc, u, \beta)$ returns a stationary deterministic policy that picks its actions among the one reaching the maximum above.
The projection of $\Lc^\beta$ to $\Hc$ is 
\begin{equation}
    \mathfrak{L} \equiv \mathfrak{L}^{\beta, \Hc} := \Gamma \circ \Lc^\beta.
\end{equation}
The goal of this section is to establish \cref{proposition:PMEVI} and 
\begin{itemize}
    \item \cref{proposition:PMEVI} statement (1) is a consequence of \cref{lemma:fixpoint};
    \item \cref{proposition:PMEVI} statement (2) follows from \cref{theorem:optimism};
    \item \cref{proposition:PMEVI} statement (3) follows from \cref{corollary:modelization};
    \item \cref{proposition:PMEVI} statement (4) follows from \cref{corollary:PMEVI-operator-regularity};
    \item \cref{proposition:PMEVI} prerequisites on the projection operator and \cref{lemma:projection-algorithm} follows from \cref{lemma:projection-properties}
\end{itemize}

\subsection{Finding an optimistic policy under bias constraints}
\label{section:optimistic-policy-under-bias-constraints}

The main goal is to find and optimistic policy under \emph{bias constraints} (projection) and \emph{bias error constraints} (mitigation).
The bias constraints imply that we search for a policy $\pi$ together with a model $\widetilde{M}$ such that $h(\pi, \widetilde{M}) \in \Hc$.
The bias error means that, for $\tilde{h} \equiv h(\pi, \widetilde{M})$, we want in addition $\tilde{p}(s, \pi(s)) \tilde{h} \le \hat{p}(s, \pi(s)) \tilde{h} + \beta(s, \pi(s))$ where $\tilde{p}$ is the transition kernel of $\widetilde{M}$.
In the end, our goal is to track the solution of the following optimization problem:
\begin{equation}
\label{equation:constrained-gain-optimization}
    g^*(\Hc, \beta, \Mc)
    :=
    \sup \set{
        g\parens{\pi, \widetilde{M}}
        :
        \begin{array}{c}
            \pi \in \Pi, \widetilde{M} \in \Mc, 
            \\
            \forall s \in \Sc,~ \tilde{p}(s, \pi(s)) \tilde{h} \le \hat{p}(s, \pi(s)) \tilde{h} + \beta(s, \pi(s)), 
            \\
            \tilde{h} \equiv h(\pi, \widetilde{M}) \in \Hc, ~
            \sp{g\parens{\pi, \widetilde{M}}} = 0
        \end{array}
    }
\end{equation}
where the supremum is taken with respect to the product order $\R^\Sc$.
In particular, if $\Uc \subseteq \Rc^\Sc$, check that $u^* = \sup \Uc$ is obtained as $u^*(s) := \sup\set{v(s): v \in \Uc}$.
The constraint $\sp{g\parens{\pi, \widetilde{M}}} = 0$ is suggested by the work of \cite{fruit_efcient_2018,fruit_exploration-exploitation_2019} and is key for the problem to be solvable.

The bias and the $\beta$-constraints make the problem to handle with a ``pure'' extended MDP solution, which is why the extended Bellman operators are mitigated (with $\beta$) then projected (with $\Gamma$).
The mitigation operation guarantees that the $\beta$-constraint is satisfied, while the projection on $\Hc$ makes sure that the bias constraint is satisfied.
It is important for both operations to be compatible, i.e., that the $\beta$-constraint that $\Lc^\beta$ forces is not lost when applying $\Gamma$.
As a matter of fact, projecting then mitigating would not work.

We now explain why $\mathfrak{L}$ can be used to solve \eqref{equation:constrained-gain-optimization}.

\subsection{Projection operation and definition of $\mathfrak{L}$}
\label{section:projection-operation}

We start by discussing why $\mathfrak{L}$ is well-defined at all.
The well-definition of $\Lc^\beta$ is obvious.
The point is to explain why the projection onto $\Hc$ is possible while preserving mandatory structural properties such as monotony, non-expansivity, linearity and more.
For general $\Hc$, such properties are impossible to meet.
But the bias confidence region constructed with \hyperlink{algorithm:bias-estimation}{Algorithm~3} has a specific shape that makes the projection possible.
The central property is the one below:

(\textbf{A1}) \textit{
    The downward closure $\set{v \le u: v \in \Hc}$ of every $u \in \R^\Sc$ has a maximum in $\Hc$.
}

The only order that we will be considering is the product order on $\R^\Sc$.
Recall that a set $\Uc \subseteq \R^\Sc$ has a \emph{maximum} if there exists $u \in \Uc$ such that $v \le u$ for all $u \in \Uc$.
A \emph{supremum} of $\Uc$ is a minimal upper-bound of $\Uc$, i.e., $u$ such that (1) $v \le u$ for all $v \in \Uc$ and (2) no $w$ satisfying (1) can be smaller than $u$.
For the product order, the supremum of a subset $\Uc$ is unique and of the form $u(s) = \sup\set{v(s): v \in \Uc}$.

Define the projection $\Gamma : \R^\Sc \to \Hc$ as such:
\begin{equation}
\label{equation:projection-max}
    \Gamma u := \max\set{v \le u : v \in \Hc}.
\end{equation}
In general, Assumption (\textbf{A1}) is satisfied when $\Hc$ admits a join, i.e., is stable by finite supremum: $u, v \in \Hc \Rightarrow \sup(u, v) \in \Hc$.

\begin{lemma}
    If $\Hc$ is generated by constraints of the form $\mathfrak{h}(s) - \mathfrak{h}(s') - c(s, s') \le d(s,s')$, then it has a join and (\textbf{A1}) is satisfied.
    Moreover, $\Gamma$ is then correctly computed with \hyperlink{algorithm:bias-projection}{Algorithm~4}.
\end{lemma}

\begin{proof}
    The first half of the result is well-known, see \cite{zhang_sharper_2023}, but we recall a proof for self-containedness.
    Let $v_1, v_2 \in \Hc$ and define $v_3 := \sup(v_1, v_2)$.
    Observe that $v_3(s) - v_3(s') \le \max(v_1(s) - v_1(s'), v_2(s) - v_2(s')) \le c(s,s') + d(s,s')$.
    So $v_3 \in \Hc$.

    We continue by showing that if $\Hc$ has a join, then \eqref{equation:projection-max} is well-defined.
    For $s \in \Sc$, take a sequence $v_n^s$ such that $v_n^s(s) \to \alpha(s) := \sup \set{v(s): v\le u, v \in \Hc}$.
    Because the span of every element of $\Hc$ is upper-bounded by $c := \sup\set{\sp{v} : v\in \Hc}$, it follows that $v_n^s$ evolves in the compact region $\set{v \le u: v\in\Hc} \cap \set{v: \norm{v - \alpha{s} e}_\infty = 1 + c}$.
    We can therefore extract a convergent sequence of $v_n^s$, converging $v_*^s$ that belongs to $\Hc$ since the latter is closed.
    By construction, $v_*^s(s) = \alpha(s)$.
    Because $\Hc$ has a join, $v_* := \sup\set{v_*^s: s \in \Sc} \in \Hc$.
\end{proof}

\begin{lemma}
\label{lemma:projection-properties}
    Under assumption (\textbf{A1}), the operator $\Gamma u := \max \set{v \le u: v\in \Hc}$ is well-defined, and is:
    \begin{enumerate}[itemsep=-.25em]
        \item[(1)] monotone: $u \le v \Rightarrow \Gamma u \le \Gamma v$;
        \item[(2)] non span-expansive: $\sp{\Gamma u - \Gamma v} \le \sp{u - v}$;
        \item[(3)] linear: $\Gamma(u + \lambda e) = \Gamma u + \lambda e$;
        \item[(4)] $\Gamma u \le u$.
    \end{enumerate}
\end{lemma}
\begin{proof}
    The well-definition of $\Gamma$ is obvious from (\textbf{A1}).
    For (2), if $u \le v$ then $w \le u \Rightarrow w \le v$.
    Hence $\Gamma u := \max \set{w \le u: w \in \Hc} \le \max\set{w \le v: w \in \Hc} =: \Gamma v$.
    For (3), check that it follows from $\Hc = \Hc + \R e$.
    For (4), we obviously have $\Gamma u := \max \set{v \le u: v \in \Hc} \le u$.

    The more difficult point is (2) span non-expansivity.
    Pick $u, v \in \R^\Sc$.
    By linearity, it suffices to show the result for $\sum_s u(s) = \sum_s v(s)$.
    In that case, we have $\sp{v - u} = \max(v - u) + \max(u - v)$.
    Observe that for all $w \le u$, we have $w + \min(v - u) e \le v$.
    Since $\Hc = \Hc + \R e$, it follows that:
    \begin{equation*}
        \max \set{ w \le u: u \in \Hc}
        \le
        \max \set{w \le v : w \in \Hc}
        + \max(u-v) e.
    \end{equation*}
    Similarly, we have $\max\set{w \le u: w \in \Hc} \ge \max \set{w \le v: w \in \Hc} + \min (v - u) e$.
    Using them both at once, we find $\sp{\Gamma u - \Gamma v} \le \sp{v - u}$.
\end{proof}

The properties (1), (3) and (4) are essential for $\mathfrak{L}$ to properly address the optimization problem \eqref{equation:constrained-gain-optimization}.
The property (2) is just as important, because it plays a central part in the convergence of value iteration.
The next result shows similar properties for the $\beta$-mitigated extended Bellman operator $\Lc^\beta$.
From now on, we will assume (\text{A1}), because it is almost-surely satisfied by the bias confidence region generated by \hyperlink{algorithm:bias-estimation}{Algorithm~3}.

\begin{lemma}
    The $\beta$-mitigated extended Bellman operator $\Lc^\beta$ is (1) monotone, (2) non-span-expansive and (3) linear.
\end{lemma}
\begin{proof}
    The properties (1) and (3) directly follow from the definition.
    We focus on (2).
    Fix $u, u' \in \R^\Sc$.
    By \cref{lemma:modelization}, we can write $\Lc^\beta u = \tilde{r}_\pi + \tilde{P}_\pi u$ and $\Lc^\beta u' = \tilde{r}_{\pi'} + \tilde{P}_{\pi'} u'$. 
    In the following, we write $\beta_\pi(s) := \beta(s, \pi(s))$.
    Check that:
    \begin{equation*}
        \Lc^\beta u - \Lc^\beta u' 
        = \tilde{r}_\pi + \tilde{P}_\pi u - \parens{\tilde{r}_{\pi'} + \tilde{P}_{\pi'} u'}
        \le
        \tilde{r}_\pi + \tilde{P}_\pi u - \parens{
            \tilde{r}_\pi + \min \set{
                \tilde{P}_\pi u',
                \hat{P}_\pi u' + \beta_\pi
            }
        }
        .
    \end{equation*}
    If the minimum is reached with $\tilde{P}_\pi u'$, then:
    \begin{equation*}
        \Lc^\beta u - \Lc^\beta u' \le \tilde{P}_\pi (u - u').
    \end{equation*}
    If the minimum is reached with $\hat{P}_\pi u' + \beta_\pi$, then upper-bound $\tilde{P}_\pi u$ by $\hat{P}_\pi u + \beta_\pi$ to obtain:
    \begin{equation*}
        \Lc^\beta u - \Lc^\beta u' \le \hat{P}_\pi (u - u').
    \end{equation*}
    Overall, we find that there exists $Q_\pi \in \Pc_\pi$ such that $\Lc^\beta u - \Lc^\beta u' \le Q_\pi (u - u')$.
    Similarly, we find $Q_{\pi'} \in \Pc_{\pi'}$ such that $\Lc^\beta u - \Lc^\beta u' \ge Q_{\pi'} (u - u')$.
    We conclude that:
    \begin{equation*}
        \sp{\Lc^\beta u - \Lc^\beta u'} 
        \le
        \sp{(Q _\pi - Q_{\pi'}) (u - u')}
        \le
        \sp{u - u'}.
    \end{equation*}
    This concludes the proof.
\end{proof}

By composition, we obtain the following result.

\begin{corollary}
\label{corollary:PMEVI-operator-regularity}
    $\mathfrak{L}$ is (1) monotone, (2) non-span-expansive and (3) linear.
    Moreover, $\sp{\mathfrak{L}u - \mathfrak{L}v} \le \sp{\Lc u - \Lc v}$ for all $u, v \in \R^\Sc$.
\end{corollary}

\subsection{Fix-points of $\mathfrak{L}$ and (weak) optimism}
\label{section:fixpoints}

\begin{lemma}
\label{lemma:fixpoint}
    $\mathfrak{L}$ has a fix-point in span semi-norm, i.e., $\exists u \in \Hc, \sp{\mathfrak{L}u - u} = 0$.
\end{lemma}
\begin{proof}
    The idea is to apply Brouwer's fix-point theorem in $\R^\Sc$ quotiented by the equivalence relation $u \sim v \Leftrightarrow \sp{u - v} = 0$, where $\sp{-}$ becomes a norm.
    By linearity (\cref{corollary:PMEVI-operator-regularity}), $\mathfrak{L}$ is well-defined in this quotient space, and if $\mathfrak{L}$ is shown continuous on $\R^\Sc$, so will it be on the quotient.

    We show that $\mathfrak{L}$ is sequentially continuous on $\Hc$.
    Consider a sequence $u_n \in \Hc^\N$ converging to $u \in \Hc$ and fix $\epsilon > 0$.
    Provided that $n > N_\epsilon$ for $N_\epsilon$ large enough, we have $\norm{u_n - u}_\infty < \epsilon$, i.e., $u_n - \epsilon e \le u_n \le u + \epsilon e$.
    Therefore, in the one hand, for all $v \le u_n$, we have $v - \epsilon e \le u$ so $\max\set{v \le u_n: v \in \Hc} \le \max\set{v \le u: v \in \Hc} + \epsilon e$;
    And on the other hand, for all $v \le u$, $v + \epsilon e \le u_n$ so $\max\set{v \le u: v \in \Hc} \le \max\set{v \le u_n: v \in \Hc} + \epsilon e$.
    Hence:
    \begin{equation*}
        \norm{\max\set{v \le u: v \in \Hc} - \max \set{v \le u_n: v \in \Hc}} \le \epsilon.
    \end{equation*}
    It shows that $\Gamma$ is continuous.
    The operator $\Lc^\beta$ is obviously continuous as well, so $\mathfrak{L} = \Gamma \circ \Lc^\beta$ is continuous by composition.
    Since $\Hc = \Hc_0 + \R e$ with $\Hc_0$ compact and ocnvex, the quotient $\Hc /{\sim}$ is compact and convex, and is preserved by $\mathfrak{L} / {\sim}$.
    By Brouwer's fix-point theorem, $\mathfrak{L}/{\sim}$ has a fix-point in $\Hc / {\sim}$.
    So $\mathfrak{L}$ has a span fix-point in $\Hc$.
\end{proof}

We write $\Fix(\mathfrak{L})$ the span fix-points of $\mathfrak{L}$.

\begin{lemma}
\label{lemma:growth}
    $\mathfrak{L}$ has well-defined growth.
    Specifically, if $\mathfrak{L} u = u + \mathfrak{g} e$, then:
    \begin{enumerate}[itemsep=-.25em]
        \item[(1)] There exists $c > 0$, s.t., for all $v \in \Hc_0$, $(n \mathfrak{g} - c) e + u \le \mathfrak{L}^n v \le (n \mathfrak{g} + c) e + u$;
        \item[(2)] If $u' \in \Fix(\mathfrak{L})$, then $\mathfrak{L} u' - u' = \mathfrak{g} e$.
    \end{enumerate}
\end{lemma}
\begin{proof}
    Setting $c := \max_{v \in \Hc_0} \norm{v - u}_\infty < \infty$, one can check that $u - ce \le v \le u + ce$ for all $v \in \Hc_0$.
    this proves (1) for $n = 0$ and we then proceed by induction on $n \ge 0$.
    By induction, $\mathfrak{L}^n v \le u + (n \mathfrak{g} + c) e$ and by \cref{corollary:PMEVI-operator-regularity}, $\mathfrak{L}$ is monotone, so we have:
    \begin{equation*}
        \mathfrak{L}^{n+1} v 
        \le \mathfrak{L} \mathfrak{L}^n v 
        \le \mathfrak{L} ( u + (n \mathfrak{g} + c) e)
        = u + ((n+1) \mathfrak{g} + c)e
    \end{equation*}
    where the last inequality use the linearity of $\mathfrak{L}$ together with $\mathfrak{L} u = u + \mathfrak{g} e$.
    The lower bound of $\mathfrak{L}^n v$ is shown similarly, establishing (1).

    For (2), pick $u' \in \Fix(\mathfrak{L})$ with $\mathfrak{L} u' = u' + \mathfrak{g}' e$.
    Up to translating $u'$, we can assume that $u' \in \Hc_0$ and apply (1).
    We get:
    \begin{equation*}
        (n \mathfrak{g} - c)e + u \le n \mathfrak{g}' e + u' \le (n \mathfrak{g} + c) e + u.
    \end{equation*}
    Divided by $n$ and let it go to infinity.
    We conclude that $\mathfrak{g} = \mathfrak{g}'$.
\end{proof} 

We finally have everything in hand to claim that $\mathfrak{L}$ solves \eqref{equation:constrained-gain-optimization}.

\begin{corollary}
\label{corollary:growth}
    The \emph{growth} of $\mathfrak{L}$ given by $\mathfrak{g} = \mathfrak{L} u - u$ for $u \in \Fix(\mathfrak{L})$ is well-defined, and:
    \begin{equation*}
        \forall u \in \Hc, \quad
        \mathfrak{g} e 
        = \liminf_{n \to \infty} \frac{\mathfrak{L}^n u} n 
        = \limsup_{n \to \infty} \frac{\mathfrak{L}^n u} n
        .
    \end{equation*}
    Moreover, $\mathfrak{g} \ge g^*(\Hc, \beta, \Mc)$.
\end{corollary}
\begin{proof}
    The growth property is a direct consequence of \cref{lemma:growth}.
    We show $\mathfrak{g} \ge g^*(\Hc, \beta, \Mc)$ which is defined in \eqref{equation:constrained-gain-optimization}.
    Pick $\pi \in \Pi, \widetilde{M} \in \Mc$ its model with $\tilde{h} \equiv h(\pi, \widetilde{M})$ and $\tilde{P}_\pi \tilde{h} \le \hat{P}_\pi \tilde{h} + \beta_\pi$ where $\beta_\pi(s) := \beta(s, \pi(s))$.
    Up to translation, we can assume that $\tilde{h} \in \Hc_0$.

    We have $g(\pi, \widetilde{M}) = \tilde{g} e$ for $\tilde{g} \in \R$, so
    \begin{equation*}
        \tilde{h} + \tilde{g} e = \tilde{r}_\pi + \tilde{P}_\pi \tilde{h} \le \mathfrak{L} \tilde{h}
    \end{equation*}
    by definition.
    By monotony of $\mathfrak{L}$, see \cref{corollary:PMEVI-operator-regularity}, $n \tilde{g} e + \tilde{h} \le \mathfrak{L}^n \tilde{h}$ follows by induction on $n \ge 0$.
    By \cref{lemma:growth}, we further have $\mathfrak{L}^n \tilde{h} \le n (\mathfrak{g} + c) e + u$ where $u \in \Fix(\mathfrak{L})$.
    In tandem,
    \begin{equation*}
        \tilde{g} e \le \mathfrak{g} e + \frac{c e + u - \tilde{h}} n.
    \end{equation*}
    Letting $n \to \infty$, we deduce that $\tilde{g} \le \mathfrak{g}$.
    Conclude by taking the best $\pi$ and $\widetilde{M}$.
\end{proof}

The next theorem follows directly with the same proof technique, and guarantees optimism.

\begin{theorem}
\label{theorem:optimism}
    Assume that $g^* + h^* \le \mathfrak{L} h^*$.
    Then $\mathfrak{g} \ge g^*$.
\end{theorem}

The condition ``$g^* + h ^* \le \mathfrak{L}h^*$'' can be referred to as a \emph{weak} form of optimism.
We qualify this version of optimism as \emph{weak} because it is much weaker than optimism property suggested by \cite{fruit_exploration-exploitation_2019} $\Lc \ge L$ where $L$ is the Bellman operator of the true MDP.
Here, we only ask for $\mathfrak{L} h^* \ge Lh^*$, i.e., optimism at the fix-point of $L$.
This condition is met as soon as $M \in \Mc$, $h^* \in \Hc$ and $\beta$ is large enough, but is in fact much more general.

\subsection{Modelization of the projected mitigated Bellman operator $\mathfrak{L}$}
\label{section:modelization}

The aim of this paragraph is to establish \cref{corollary:modelization}, stating that $\mathfrak{L} u$ can be viewed as a policy produced by \texttt{Greedy}$(\Mc, u, \beta)$.

\begin{lemma}[Modelization]
\label{lemma:modelization}
    For $\pi \in \Pi$, denote $\beta_\pi(s) := \beta(s, \pi(s))$, $\Rc_\pi := \prod_s \Rc(s, \pi(s))$ and $\Pc_\pi := \prod_s \Pc(s, \pi(s))$.
    Fix $u \in \R^\Sc$ and let $\pi := \texttt{\upshape Greedy}(\Mc, u, \beta)$.

    \begin{enumerate}[itemsep=-.25em]
        \item[(1)] 
            If $\Pc$ is convex, then there exists $(\tilde{r}_\pi, \tilde{P}_\pi) \in \Rc_\pi \times \Pc_\pi$ such that $\Lc_\beta u = \tilde{r}_\pi + \tilde{P}_\pi u$.
        \item[(2)]
            Assume that $\Lc_\beta u = \tilde{r}_\pi + \tilde{P}_\pi u$.
            There exists $r'_\pi \le \tilde{r}_\pi$ such that $\mathfrak{L} u = r'_\pi + \tilde{P}_\pi u$.
    \end{enumerate}
\end{lemma}

The convexity requirement of (1) is always true if the kernel confidence region is chosen via (\textbf{C1-4}).

\begin{proof}
    For (1), fix a state $s \in \Sc$, let $a := \pi(s)$ and $\rho := \min(\sup \Pc(s,a) u, \hat{p}(s,a) u + \beta(s,a))$.
    If $\rho = \sup \Pc(s,a) u$, then there is nothing to say because $\Pc$ is compact, hence the sup is a max and $\rho$ is of the form $\tilde{p}(s,a) u$.
    Otherwise, let $\tilde{p}(s,a) u > \hat{p}(s,a) u + \beta(s,a)$ with $\tilde{p}(s,a) \in \Pc(s,a)$.
    Introduce, for $\lambda \in [0, 1]$,
    \begin{equation*}
        \tilde{p}_\lambda (s,a) := \lambda \tilde{p}(s,a) + (1 - \lambda) \hat{p}(s,a).
    \end{equation*}
    By continuity, there exists $\lambda \in (0, 1)$ such that $\tilde{p}_\lambda (s,a) u = \hat{p}(s,a) u + \beta(s,a)$ and by convexity of $\Pc(s,a)$, $\tilde{p}_\lambda(s,a) \in \Pc(s,a)$.
    This proves (1).

    For (2), recall that $\mathfrak{L} u = \Gamma \Lc^\beta u = \Gamma (\tilde{r}_\pi + \tilde{P}_\pi u)$.
    Since $\Gamma v \le v$, for $v \in \R^\Sc$, we have:
    \begin{equation*}
        \Gamma (\tilde{r}_\pi + \tilde{P}_\pi u) \le \tilde{r}_\pi + \tilde{P}_\pi u.
    \end{equation*}
    Set $r'_\pi := \Gamma(\tilde{r}_\pi + \tilde{P}_\pi u) - \tilde{P}_\pi u$.
    Check that $r'_\pi$ satisfies $r'_\pi \le \tilde{r}_\pi$ and $\mathfrak{L} u = r'_\pi + \tilde{P}_\pi u$.
\end{proof} 

The last corollary bellow is crucial to claim that greedy policies are good choices in \texttt{PMEVI-DT}.

\begin{corollary}[Greedy modelization]
\label{corollary:modelization}
    Let $u \in \R^\Sc$ and fix $\pi := \texttt{\upshape Greedy}(\Mc, u, \beta)$.
    If $\Pc$ is convex, then with the notations of \cref{lemma:modelization}, there exists $\tilde{r}_\pi \le \sup \Rc_\pi$ and $\tilde{P}_\pi \in \Pc_\pi$ such that $\mathfrak{L}u = \tilde{r}_\pi + \tilde{P}_\pi u$.
\end{corollary}

\newpage
\section{Proof of \cref{theorem:main}: Regret analysis of \texttt{PMEVI-DT}}

We recall a few notations.
At episode $k$, the played policy is denoted $\pi_k$.
As a greedy response to $\mathfrak{h}_k$, by \cref{proposition:PMEVI} (3), there exists $\tilde{r}_k(s) \le \sup \Rc_{t_k}(s, \pi_k(s))$ and $\tilde{P}_k(s) \in \Pc_{t_k}(s, \pi(x))$ such that $\mathfrak{h}_k + \mathfrak{g}_k = \tilde{r}_k + \tilde{P}_k\mathfrak{h}_k$.
The reward-kernel pair $\tilde{M}_k = (\tilde{r}_k, \tilde{P}_k)$ is referred to as the \emph{optimistic model} of $\pi_k$.
We write $P_k := P_{\pi_k}(M)$ the true kernel and $\hat{P}_k := P_{\pi_k} (\hat{M}_{t_k})$ the empirical kernel.
Likewise, we define the reward functions $r_k$ and $\hat{r}_k$. 
The optimistic gain and bias satisfy $\mathfrak{g}_k = g(\pi_k, \widetilde{M}_k)$ and $\mathfrak{h}_k = h(\pi_k, \widetilde{M}_k)$. 
We further denote $c_0 = T^{\frac 15}$.

\paragraph{Important remark.}
To slightely simplify the analysis, we assume that \texttt{PMEVI} is run with perfect precision $\epsilon = 0$, i.e., that $\mathfrak{h}_k = \texttt{PMEVI}(\Mc_{t_k}, \beta_{t_k}, \Gamma_{t_k}, 0)$ hence is a span fix-point of $\mathfrak{L}_{t_k}$.
This assumption is mild and can be dropped by adding an extra error term that has to be carried out in the calculations.

\subsection{Number of episodes under doubling trick (DT)}

\begin{lemma}[Number of episodes, \cite{auer_near-optimal_2009}]
\label{lemma:episode-number}
    The number of episodes up to time $T \ge SA$ is upper-bounded by:
    \begin{equation*}
        K(T) \le SA \log_2\parens{\tfrac{8T}{SA}}
        .
    \end{equation*}
\end{lemma}

\subsection{Sum of bias variances}

The \cref{lemma:variance-sum} below shows that $\sum_{t=0}^{T-1} \Var(p(X_t), h^*)$ scales as $T\sp{h^*} \sp{r} + \sp{h^*} \Reg(T)$ in probability.

\begin{lemma}
\label{lemma:variance-sum}
    With probability at least $1 - \delta$, we have:
    \begin{equation*}
        \sumnl_{t=0}^{T-1} \Var(p(X_t), h^*)
        \le
        2 \sp{h^*} \sp{r} T 
        + \sp{h^*}^2 \sqrt{\tfrac 12 T \log\parens{\tfrac 1\delta}}
        + 2 \sp{h^*} \sumnl_{t=0}^{T-1} \Delta^*(X_t) 
        + \sp{h^*}^2.
    \end{equation*}
\end{lemma}
\begin{proof}
    Using the Bellman equation $h^*(s) + g^*(s) = r(s,a) + p(s, a)h^* + \Delta^*(s,a)$, we have:
    \begin{equation*}
        \Var(p(X_t), h^*) 
        = 
        \parens{p(X_t) - e_{S_t}} h^{*2} 
        + 2 h^*(S_t) (\Delta^*(X_t) + r(X_t) - g^*(S_t))
        .
    \end{equation*}
    Since $\sp{h^{*2}} \le \sp{h^*}^2$, we get:
    \begin{align*}
        \sumnl_{t=0}^{T-1} \Var(p(X_t), h^*) 
        & \le
        \sumnl_{t=0}^{T-1} \parens{p(X_t) - e_{S_t}} h^{*2} 
        + 2 \sp{h^*} \parens{
            \sp{r} T + \sumnl_{t=0}^{T-1} \Delta^*(X_t)
        }
        \\
        & = 
        \sumnl_{t=0}^{T-1} \parens{p(X_t) - e_{S_{t+1}}} h^{*2}
        + 2 \sp{h^*} \parens{
            \tfrac 12 \sp{h^*}
            \sp{r} T + \sumnl_{t=0}^{T-1} \Delta^*(X_t)
        }
        \\
        \text{(\cref{lemma:azuma})} & \le
        2 \sp{h^*} \sp{r} T + \sp{h^*}^2 \sqrt{\tfrac 12 T \log\parens{\tfrac 1\delta}} + 2 \sp{h^*} \sumnl_{t=0}^{T-1} \Delta^*(X_t) + \sp{h^*}^2
    \end{align*}
    where the last inequality holds with probability $1 - \delta$.
    This concludes the proof.
\end{proof}

\subsection{Regret and pseudo-regret: A tight relation}

In this paragraph, we bound the regret with respect to the pseudo-regret (and conversely) up to a factor of order $(\sp{h^*} \sp{r} \log(\tfrac T\delta))^{1/2}$. Hence, in proofs, the pseudo-regret can be changed to the regret with ease.

\begin{lemma}
\label{lemma:regret-pseudo-regret}
    With probability $1 - 4 \delta$, the regret and the pseudo-regret and linked as follows:
    \begin{equation*}
        \abs{
            \sum_{t=0}^{T-1} (g^* - R_t)
            - 
            \sum_{t=0}^{T-1} \Delta^*(X_t)
        }
        \le 
        \begin{Bmatrix}
           2\!\!\sqrt{\parens{2 \sp{h^*} \sp{r} + \tfrac 18} T \log\parens{\tfrac T\delta}} 
           + \!\!\sqrt{2 \sp{h^*} \log\parens{\tfrac T\delta} \sumnl_{t=0}^{T-1} \Delta^*(X_t)} \\
            + \sp{h^*} \parens{\tfrac 12 T}^{\frac 14} \log^{\frac 34} \parens{\frac T\delta} + 4 \sp{h^*} \log\parens{\tfrac T\delta} + 2 \sp{h^*}
        \end{Bmatrix}
        .
    \end{equation*}
\end{lemma}
\begin{proof}
    We rely again on the Poisson equation $g^*(S_t) - r(X_t) - \Delta^*(X_t) = (p(X_t) - e_{S_t}) h^*$, so: 
    \begin{align*}
        \textrm{A} 
        := \abs{
            \sumnl_{t=0}^{T-1} (g^* - R_t - \Delta^*(X_t))
        }
        & \le 
        \abs{
            \sumnl_{t=0}^{T-1}
            \parens{p(X_t) - e_{S_t}} h^*
        }
        + 
        \abs{
            \sumnl_{t=0}^{T-1}
            \parens{R_t - r(X_t)}
        }
        \\
        & \le 
        \sp{h^*}
        + \abs{
            \sumnl_{t=0}^{T-1}
            \parens{p(X_t) - e_{S_{t+1}}} h^*
        } +        
        \abs{
            \sumnl_{t=0}^{T-1}
            \parens{R_t - r(X_t)}
        }
        .
    \end{align*}
    Up to the constant $\sp{h^*}$, the two error terms are respectively a navigation and a reward error.
    The second is bounded using Azuma's inequality (\cref{lemma:azuma}), showing that with probability $1 - 2\delta$, we have: 
    \begin{equation*}
        \abs{\sumnl_{t=0}^{T-1} (R_t - r(X_t))} 
        \le 
        \sqrt{\tfrac 12 T \log\parens{\tfrac 1\delta}}
        .
    \end{equation*}
    We continue by using Freedman's inequality, instantiated in the form of \cref{lemma:additive-freedman}.
    With probability $1 - \delta$, we have:
    \begin{equation*}
        \abs{
            \sumnl_{t=0}^{T-1} \parens{p(X_t) - e_{S_{t+1}}} h^*
        }
        \le 
        \sqrt{2 \sumnl_{t=0}^{T-1} \Var(p(X_t), h^*) \log \parens{\tfrac T\delta}} + 4 \sp{h^*} \log \parens{\tfrac T\delta}
        .
    \end{equation*}
    The quantity $\sum_{t=0}^{T-1} \Var(p(X_t), h^*)$ is a classical one that appears at several places throughout the analysis.
    Using \cref{lemma:variance-sum}, we bount it explicitely.
    Further simplifying the bound with $\sqrt{a + b} \le \sqrt{a} + \sqrt{b}$, we get that with probability $1 - 4\delta$, we have:
    \begin{equation*}
        \textrm{A} 
        \le
        \begin{Bmatrix}
            \sqrt{2 \sp{h^*} \sp{r} T \log\parens{\tfrac T\delta}} 
            + \sqrt{\tfrac 12 T \log\parens{\tfrac 1\delta}}
            + \sqrt{2 \sp{h^*} \log\parens{\tfrac T\delta} \sumnl_{t=0}^{T-1} \Delta^*(X_t)} \\
            + \sp{h^*} \parens{\tfrac 12 T}^{\frac 14} \log^{\frac 34} \parens{\frac T\delta} + 4 \sp{h^*} \log\parens{\tfrac T\delta} + 2 \sp{h^*}
        \end{Bmatrix}
        .
    \end{equation*}
    Bound $\log(\tfrac 1\delta)$ by $\log(\tfrac T\delta)$ and use $\sqrt{a} + \sqrt{b} \le 2 \sqrt{a + b}$ to merge the terms in $\sqrt{T \log(\tfrac T\delta)}$ under a single square-root.
\end{proof}

Overall, \cref{lemma:regret-pseudo-regret} states that the regret $\sum_{t=0}^{T-1} (g^* - R_t)$ and the pseudo-regret $\sum_{t=0}^{T-1}  \Delta^*(X_t)$ differ by about $(\sp{h^*} T \log(\tfrac T\delta))^{1/2}$ in probability (up to asymptotically negligible additional terms).
In general, the precise form of \cref{lemma:regret-pseudo-regret} is not convenient to use because it is of form form $x \le y + \alpha \sqrt{y} + \beta$ that is not linear in $y$. 
\cref{corollary:regret-pseudo-regret} factorizes the result into one which will be more convenient in proofs.

\begin{corollary}
\label{corollary:regret-pseudo-regret}
    Denote $x := \sum_{t=0}^{T-1} (g^* - R_t)$ and $y := \sum_{t=0}^{T-1} \Delta^*(X_t)$.
    Further introduce:
    \begin{align*}
        \alpha & := 
        \sqrt{2 \sp{h^*} \log\parens{\tfrac T\delta}}
        \\
        \beta & := 
        2 \sqrt{\parens{2 \sp{h^*} \sp{r} 
        + \tfrac 12} T \log\parens{\tfrac T\delta}} 
        + \sp{h^*} \parens{\tfrac 12 T}^{\frac 14} \log^{\frac 34}\parens{\tfrac T\delta} 
        + 2 \sp{h^*}\parens{2 \log\parens{\tfrac T\delta} + 1}
        .
    \end{align*}
    Then, with probability $1 - 4\delta$, we have $\sqrt{x} \le \sqrt{y} + \tfrac 12 \alpha + \sqrt{\beta}$ and $\sqrt{y} \le \sqrt{x} + \alpha + \sqrt{\beta}$.
\end{corollary}
\begin{proof}
    This is straight forward algebra from the result of \cref{lemma:regret-pseudo-regret}.
\end{proof}

\subsection{Proof of \cref{lemma:reward-optimism}, reward optimism}

We start by getting rid of the reward noise.
We have:
\begin{align*}
    \Reg(T)
    & := \sumnl_{t=0}^{T-1} (g^* - R_t) 
    = \sumnl_{t=0}^{T-1} (g^* - r(X_t)) + \sumnl_{t=0}^{T-1} (r(X_t) - R_t)
    \\ 
    & \le \sumnl_{t=0}^{T-1} (g^* - r(X_t)) + \sqrt{\tfrac 12 T \log\parens{\tfrac 1\delta}}
\end{align*}
with probability $1 - \delta$ by Azuma's inequality (\cref{lemma:azuma}).
We are left with $\sum_{t=0}^{T-1} (g^* - r(X_t))$.
We continue by splitting the regret episodically and invoking optimism. 
By \cref{lemma:bias-confidence-region}, with probability $1 - 4\delta$, we have $\sumnl_{t=0}^{T-1} (g^* - r(X_t)) \le \sum_k \sum_{t=t_k}^{t_{k+1}-1} (\mathfrak{g}_k - r(X_t))$. 
Introduce
\begin{equation}
    B_0 (T) := \sum_k \sum_{t=t_k}^{t_{k+1}-1} (\mathfrak{g}_k - r(X_t)).
\end{equation}
We focus on bounding $B_0(T)$.
By \cref{assumption:sub-weissman}, $\tilde{r}_k(s,a)$ is of the form $\hat{r}_k(s,a) + \sqrt{C \log(2SAT/\delta)/N_{t_k}(s,a)} - \eta_k(s,a)$ with $\eta_k(s,a) \in \R$. 
By the statement (3) of \cref{proposition:PMEVI}, $\eta_k (s,a) \ge 0$.
Therefore,
\begin{align*}
    B_0 (T)
    & = 
    \sum_k \sum_{t=t_k}^{t_{k+1}-1} 
    \parens{\mathfrak{g}_k - \tilde{r}_k(X_t)}
    + 
    \sum_k \sum_{t=t_k}^{t_{k+1}-1}
    \parens{\tilde{r_k}(X_t) - r(X_t)}
    \\
    & \le 
    \sum_k \sum_{t=t_k}^{t_{k+1}-1}
    \parens{\mathfrak{g}_k - \tilde{r}_k(X_t)}
    + S A +
    \sum_k \sum_{t=t_k}^{t_{k+1}-1}
    \indicator{N_{t_k}(X_t) \ge 1} 
    \parens{
        \hat{r_k}(X_t) - r(X_t)
        +
        \sqrt{\frac{C \log\parens{\tfrac{2 S A T}\delta}}{N_{t_k}(X_t)}}
    }
    \\
    & \overset{(*)}{\le}
    \sum_k \sum_{t=t_k}^{t_{k+1}-1}
    \parens{\mathfrak{g}_k - \tilde{r}_k(X_t)}
    + S A +
    \sum_k \sum_{t=t_k}^{t_{k+1}-1}
    \indicator{N_{t_k}(X_t) \ge 1} 
    \parens{
        \sqrt{\frac{2 \log\parens{\tfrac{2 S A T}\delta}}{N_{t_k}(s,a)}}
        +
        \sqrt{\frac{C \log\parens{\tfrac{2 S A T}\delta}}{N_{t_k}(s,a)}}
    }
\end{align*}
where $(*)$ holds with probability $1 - \delta$ following \cref{lemma:uniform-weissman}.
By the doubling trick rule (DT), we have $N_t(X_t) \le 2 N_{t_k}(X_t)$ for $t < t_{k+1}$, so, with probability $1 - \delta$,
\begin{align*}
    B_0 (T)
    & \le 
    \sum_k \sum_{t=t_k}^{t_{k+1}-1}
    \parens{\mathfrak{g}_k - \tilde{r}_k(X_t)}
    + S A +
    2 \sum_k \sum_{t=t_k}^{t_{k+1}-1}
    \indicator{N_{t_k}(X_t) \ge 1} 
    \sqrt{\frac{(2+C) \log\parens{\tfrac{2 S A T}\delta}}{N_{t_k}(s,a)}}
    \\
    & \le 
    \sum_k \sum_{t=t_k}^{t_{k+1}-1}
    \parens{\mathfrak{g}_k - \tilde{r}_k(X_t)}
    + S A +
    2 \sqrt{(2+C) \log\parens{\tfrac{2SAT}\delta}} 
    \sum_{s,a} \sum_{n=1}^{N_T(s,a)-1} \sqrt{\tfrac 1n}
    \\
    & \le
    \sum_k \sum_{t=t_k}^{t_{k+1}-1}
    \parens{\mathfrak{g}_k - \tilde{r}_k(X_t)}
    + S A 
    + 4 \sqrt{(2+C) \log\parens{\tfrac{2SAT}\delta}} 
    \sum_{s,a} \sqrt{N_T(s,a)}
    \\
    \text{(Jensen)} & \le 
    \sum_k \sum_{t=t_k}^{t_{k+1}-1}
    \parens{\mathfrak{g}_k - \tilde{r}_k(X_t)}
    + S A 
    + 4 \sqrt{(2+C) SA T \log\parens{\tfrac{2SAT}\delta}}
    .
\end{align*}
We conclude that with probability $1 - 6\delta$, we have:
\begin{equation}
\label{equation:reward-optimism}
    \Reg(T)
    \le
    \sum_k \sum_{t=t_k}^{t_{k+1}-1}
    \parens{\mathfrak{g}_k - \tilde{r}_k(X_t)}
    + 4 \sqrt{(2+C) SA T \log\parens{\tfrac{2SAT}\delta}}
    + \sqrt{\tfrac 12 T \log\parens{\tfrac{2SAT}\delta}}
    + S A 
    .
\end{equation}
This concludes the proof.
\qed

\subsection{Proof of \cref{lemma:navigation-error}, navigation error}

We have:


\begin{align*}
    \sum_{k} \sum_{t=t_k}^{t_{k+1}-1} 
    (p_k(S_t)-e_{S_t})\mathfrak{h}_k  
    & \le
    \sum_k \sum_{t=t_k}^{t_{k+1}-1}
    (p_k(S_t)-e_{S_{t+1}})\mathfrak{h}_k 
    + \sum_k \sp{\mathfrak{h}_k}
    \\ 
    & \le
    \underbrace{
        \sum_k \sum_{t=t_k}^{t_{k+1}-1}
    (p_k(S_t)-e_{S_{t+1}}) (\mathfrak{h}_k - h^*)
    }_{\mathrm{A}_1}
    + \underbrace{
        \sum_k \sum_{t=t_k}^{t_{k+1}-1}
    (p_k(S_t)-e_{S_{t+1}}) h^*
    }_{\mathrm{A}_2}
    + \sum_k \sp{\mathfrak{h}_k}
    .
\end{align*}
The last term is $\OH(c_0 S A \log(T))$ by \cref{lemma:episode-number}, hence is $\OH(T^{1/5} \log(T))$.

\STEP{1}
We start by bounding $\mathrm{A}_1$.
By \cref{lemma:bias-confidence-region}, with probability $1 - 4 \delta$, we have $h^* \in \Hc_{t_k}$ for all $k \le K(T)$. 
So $\sp{\mathfrak{h}_k - h^*} \le \sp{\mathfrak{h}_k} + \sp{h^*} \le 2 c_0$.
By Freedman's inequality invoked in the form of \cref{lemma:additive-freedman}, we have with probability $1 - 5\delta$,
\begin{equation*}
    \mathrm{A}_1 
    \le 
    \sqrt{
        2 \sum_k \sum_{t=t_k}^{t_{k+1}-1}
        \Var\parens{p(X_t), \mathfrak{h}_k - h^*} 
        \log\parens{\tfrac T\delta}
    }
    + 8 c_0 \log\parens{\tfrac T\delta}
\end{equation*}
It suffices to bound the first term.
Recall that $e$ is the vector full of ones.
We have:
\begin{align*}
    & \sum_k \sum_{t=t_k}^{t_{k+1}-1}
    \mathbf{V}(p(X_t),\mathfrak{h}_k - h^*) 
    =
    \sum_k \sum_{t=t_k}^{t_{k+1}-1}
    \mathbf{V}\parens{p(X_t),\mathfrak{h}_k - h^* - (\mathfrak{h}_k(S_t)-h^*(S_t))\cdot e}
    \\
    & \le
    \sum_k \sum_{t=t_k}^{t_{k+1}-1}
    \sum_{s' \in \Sc} p(s'|X_t) \parens{
        \mathfrak{h}_k(s') - h^*(s') 
        - (\mathfrak{h}_k(S_t)-h^*(S_t))
    }^2
    \\
    & \overset{(*)}\le
    3 \sum_k \sum_{t=t_k}^{t_{k+1}-1}
    \E\brackets{
        \sum_{s' \in \Sc} p(s'|X_t) \parens{
            \mathfrak{h}_k(s') - h^*(s') 
            - (\mathfrak{h}_k(S_t)-h^*(S_t))
        }^2
        \Bigg| \Fc_t
    } + 16 c_0^2 \log\parens{\tfrac 1\delta}
    \\
    & = 
    3 \sum_k \sum_{t=t_k}^{t_{k+1}-1}
    \parens{\mathfrak{h}_k(S_{t+1})-h^*(S_{t+1}) - (\mathfrak{h}_k(S_t)- h^*(S_t))}^2 
    + 16 c_0^2 \log\parens{\tfrac 1\delta}
    .
\end{align*}
Here the inequality $(*)$ holds with probability $1 - \delta$ following \cref{lemma:upper-conditional-expectation}.
We will bound the summand with the bias estimation error $\text{error}(c_k, s, s')$ that spawns the inner regret estimation $B_0(t_k) = \sum_{\ell=1}^{k-1} \sum_{t=t_\ell}^{t_{\ell+1}-1} (\mathfrak{g}_\ell - R_t)$.
This inner estimation is linked to $B(T) := \sum_{k, t} (\mathfrak{g}_k - R_t)$ the overall optimistic regret by:
\begin{align*}
    B_0(t_k) 
    & \le 
    \sumnl_{\ell=1}^{K(T)} \sumnl_{t=t_\ell}^{t_{\ell+1}-1} 
    \parens{\mathfrak{g}_k - R_t} 
    - 
    \sumnl_{\ell=k}^{K(T)} \sumnl_{t=t_\ell}^{t_{\ell+1}-1} 
    \parens{\mathfrak{g}_k - R_t}
    \\
    & \overset{(*)}\le 
    \sumnl_{\ell=1}^{K(T)} \sumnl_{t=t_\ell}^{t_{\ell+1}-1} 
    \parens{\mathfrak{g}_k - R_t} 
    - 
    \sumnl_{\ell=k}^{K(T)} \sumnl_{t=t_\ell}^{t_{\ell+1}-1} 
    \parens{g^* - R_t}
    \\
    & \le
    \sumnl_{\ell=1}^{K(T)} \sumnl_{t=t_\ell}^{t_{\ell+1}-1} 
    \parens{\mathfrak{g}_k - R_t} 
    - 
    \sumnl_{\ell=k}^{K(T)} \sumnl_{t=t_k}^{T-1} \parens{
        \Delta^*(X_t) + \parens{p(X_t) - e_{S_t}} h^* + r(X_t) - R_t
    }
    \\
    & \le 
    \sumnl_{\ell=1}^{K(T)} \sumnl_{t=t_\ell}^{t_{\ell+1}-1} 
    \parens{\mathfrak{g}_k - R_t} 
    + \sp{h^*} - 
    \sumnl_{\ell=k}^{K(T)} \sumnl_{t=t_k}^{T-1} \parens{
        \parens{p(X_t) - e_{S_{t+1}}} h^* + r(X_t) - R_t
    }
    \\
    & \overset{(\dagger)}\le 
    \sumnl_{\ell=1}^{K(T)} \sumnl_{t=t_\ell}^{t_{\ell+1}-1} 
    \parens{\mathfrak{g}_k - R_t} 
    + \sp{h^*} 
    + (1 + \sp{h^*}) \sqrt{\tfrac 12 T \log\parens{\tfrac 1\delta}}
    \\
    & =: 
    B(T) 
    + \sp{h^*}
    + (1 + \sp{h^*}) \sqrt{\tfrac 12 T \log\parens{\tfrac 1\delta}}
    .
\end{align*}
In the above, $(*)$ holds with probability $1 - 4\delta$ uniformly on $k$ following \cref{lemma:bias-confidence-region} and $(\dagger)$ holds, also uniformly on $k$, with probability $1 - \delta$ by applying Azuma-Hoeffding's inequality (\cref{lemma:azuma}).
Continuing, still on the event specified by \cref{lemma:bias-confidence-region}, we have with probability $1 - 6\delta$:
\begin{align*}
    \sum_k \sum_{t=t_k}^{t_{k+1}-1}
    \mathbf{V}(p(X_t),\mathfrak{h}_k - h^*) 
    & \le 
    3 \sum_k \sum_{t=t_k}^{t_{k+1}-1}
    \frac{3 c_0 + (1+c_0) \sqrt{8 t_k \log\parens{\tfrac 2\delta}} + 2 B_0(t_k)}{N_{t_k}(S_{t+1} \toto S_t)}
    + 16 c_0^2 \log\parens{\tfrac 1\delta}
    \\
    & \le 
    3 \sum_k \sum_{t=t_k}^{t_{k+1}-1}
    \frac{4 c_0 + (1+c_0) \sqrt{32 T \log\parens{\tfrac 2\delta}} + 2 B(T)}{N_{t_k}(S_t, A_t, S_{t+1})}
    + 16 c_0^2 \log\parens{\tfrac 1\delta}
    \\
    \text{(DT)} & \le 
    12 c_0^2 S^2 A + 3 \parens{4 c_0 + (1+c_0) \sqrt{32 T \log\parens{\tfrac 2\delta}} + 2 B(T)} S^2 A \log(T)
    \\ & \phantom{\le{}} + 16 c_0^2 \log\parens{\tfrac 1\delta}.
\end{align*}

\STEP{2}
For $\mathrm{A}_2$, by Freedman's inequality invoked in the form of \cref{lemma:additive-freedman} again, we have with probability $1 - \delta$,
\begin{align*}
    \mathrm{A}_2
    & \le 
    \sqrt{
        2 \sum_k \sum_{t=t_k}^{t_{k+1}-1} \Var(p_k(S_t), h^*) \log\parens{\tfrac{T}\delta}
    }
    + 8 c_0 \log\parens{\tfrac T\delta}
    \\
    & \le 
    \sqrt{
        2 \sum_{t=0}^{T-1} \Var(p(X_t), h^*) \log\parens{\tfrac{T}\delta}
    }
    + 8 c_0 \log\parens{\tfrac T\delta}
    .
\end{align*}
We recognize the sum of variance $\sum_{t=0}^{T-1} \Var(p(X_t), h^*)$ that we leave as is.

\STEP{3}
As a result, with probability $1-7\delta$, we have:

\resizebox{\linewidth}{!}{
$
    \displaystyle
    \sum_{k} \sum_{t=t_k}^{t_{k+1}-1} 
    (p_k(S_t)-e_{S_t})\mathfrak{h}_k  
    \le 
    \sqrt{
        2 \sum_{t=0}^{T-1} \Var(p(X_t), h^*) \log\parens{\tfrac{T}\delta}
    }
    + 2 S A^{\frac 12} \sqrt{3B(T)} \log\parens{\tfrac T\delta} 
    + \OH\parens{
        S A^{\frac 12} T^{\frac 7{20}} \log^{\frac 34}\parens{\tfrac T\delta}
    }
$
}

when $c_0 = T^{\frac 15}$.
\qed

\subsection{Proof of \cref{lemma:empirical-bias-error}, empirical bias error}

Because $h^*$ is a fixed vector, Bennett's inequality (see \cref{lemma:bennett}) guarantees that $(\hat{p}_k(S_t) - p_k(S_t) h^*$ is small as follows.
By doing a union bound over \cref{lemma:bennett} with confidence $\frac{\delta}{SAT}$ over all pairs $(s,a)$ and visits counts $N(s,a) \le T$, we see that with probability $1 - \delta$, for all $k$, we have:
\begin{align*}
    \sum_{t=t_k}^{t_{k+1}-1} 
    \parens{\hat{p}_k (S_t) - p_k(S_t)} h^*
    & \le 
    \sp{h^*} S A + 
    \sum_{t=t_k}^{t_{k+1}-1} 
    \indicator{N_{t_k}(X_t) \ge 1} \parens{
        \sqrt{\tfrac{2 \Var(p(X_t), h^*) \log \parens{\frac{SAT}\delta}}{N_{t_k}(X_t)}} 
        +
        \tfrac{\sp{h^*} \log\parens{\tfrac{SAT}\delta}}{3 N_{t_k}(X_t)}
    }
    \\
    \text{(by doubling trick)} & \le 
    \sp{h^*} S A + 
    2 \sum_{t=t_k}^{t_{k+1}-1}
    \indicator{N_t(X_t) \ge 1} \parens{
        \sqrt{\tfrac{2 \Var(p(X_t), h^*) \log \parens{\frac{SAT}\delta}}{N_t(X_t)}} 
        +
        \tfrac{\sp{h^*} \log\parens{\tfrac{SAT}\delta}}{3 N_t(X_t)}
    }
    .
\end{align*}
Summing this over $k$ and factorizing over state-action pairs, we get that with probability $1 - \delta$, 
\begin{align*}
    \sum_k (2k)
    & \le 
    \sp{h^*} S A +
    2 \sum_{s,a} \parens{
        \sum_{n=1}^{N_T(s,a)} \sqrt{
            \tfrac{2 \Var(p(s,a), h^*) \log\parens{\frac{SAT}\delta}}{n}
        }
        + 
        \sum_{n=1}^{N_T(s,a)} \tfrac{\sp{h^*} \log\parens{\tfrac{S A T}\delta}}{n}
    }
    \\
    & \le
    \sp{h^*} S A +
    4 \sum_{s,a} \sqrt{
        N_T(s,a)\Var(p(s,a), h^*) \log\parens{\tfrac{SAT}\delta}
    }
    + 
    2 \sp{h^*} S A \log\parens{\tfrac{SAT}\delta} \log(T)
    \\
    \text{(Jensen)} & \le
    \sp{h^*} S A +
    4 \sqrt{
        SA 
        \sumnl_{s,a} 
        \Var(p(s,a), h^*) \log\parens{\tfrac{SAT}\delta}
    }
    + 
    2 \sp{h^*} S A \log\parens{\tfrac{SAT}\delta} \log(T)
    \\
    & =
    \sp{h^*} S A +
    4 \sqrt{
        \sumnl_{t=0}^{T-1} 
        \Var(p(X_t), h^*) \log\parens{\tfrac{SAT}\delta}
    }
    + 
    2 \sp{h^*} S A \log\parens{\tfrac{SAT}\delta} \log(T)
\end{align*}
We recognize the sum of variances $\sum_{t=0}^{T-1} \Var(p(X_t), h^*)$, that is left to be upper-bounded later on.
\qed

\subsection{Proof of \cref{lemma:optimism-overshoot}, optimism overshoot}

Because of the $\beta$-mitigation generated by \hyperlink{algorithm:variance-approximation}{Algorithm~5}, the quantity $(\tilde{p}_k (S_t) - \hat{p}_k (S_t)) \mathfrak{h}_k$ is shown to be directly related to $\Var(p(X_t), h^*)$ up to a provably negligible error.
Denote $h'_k$ the reference point \texttt{BiasProjection}$(\Hc_{t_k}, c_{t_k}(-, s_0))$ used in \hyperlink{algorithm:variance-approximation}{Algorithm~5} (denoted $h_0$ in the algorithm). 
By \cref{lemma:bias-confidence-region}, with probability $1 - 4\delta$, we have $h^* \in \Hc_{t_k}$ for all $k$. 
To lighten up notations, we write $d_{t_k}(s', s)$ instead of $\text{error}(c_{t_k}, s', s)$.

\STEP{1}
Denote $\text{A} := (\tilde{p}_k(S_t) - \hat{p}_k(S_t)) \mathfrak{h}_k$.
By construction of $\tilde{p}_k$, we have $\text{A} \le \beta_{t_k}(X_t)$, so:
\begin{align*}
    \text{A}
    & \le \beta_{t_k}(X_t)
    \\
    & =: 
    \sqrt{
        \frac{2\parens{
            \Var(\hat{p}_k(S_t), h'_k) + 8c_0 \sum_{s' \in \Sc} \hat{p}_k(s'|S_t) d_{t_k}(s', S_t) \log \parens{\tfrac{SA T}\delta}
        }}{N_{t_k}(X_t)}
    }
    + \frac{3 c_0 \log\parens{\tfrac{SAT}\delta}}{N_{t_k}(X_t)}
    \\
    & \le 
    \underbrace{\sqrt{
        \frac{2\Var(\hat{p}_k(S_t), h'_k)}{N_{t_k}(X_t)}
    }}_{\text{A}_1}
    + 
    \underbrace{\sqrt{
        \frac{16 c_0 \sum_{s' \in \Sc} \hat{p}_k(s'|S_t) d_{t_k}(s', S_t) \log \parens{\tfrac{SA T}\delta}}{N_{t_k}(X_t)}
    }}_{\text{A}_2}
    + \frac{3 c_0 \log\parens{\tfrac{SAT}\delta}}{N_{t_k}(X_t)}
    .
\end{align*}
The rightmost term of A is of order $\OH(\log^2(T))$ hence is negligible.
We focus on the other two.
The analysis of $\text{A}_1$ will spawn a term similar to $\text{A}_2$, hence we start by the second.
Recall that $d_{t_k}$ is the bias error provided by \hyperlink{algorithm:bias-estimation}{Algorithm~3} and that the inner regret estimation is $B_0(t_k) = \sum_{\ell=1}^{k-1} \sum_{t=t_\ell}^{t_{\ell+1}-1} (\mathfrak{g}_\ell - R_t)$.
Now, remark that:
\begin{align*}
    B_0(t_k) 
    & \le 
    \sumnl_{\ell=1}^{K(T)} \sumnl_{t=t_\ell}^{t_{\ell+1}-1} 
    \parens{\mathfrak{g}_k - R_t} 
    - 
    \sumnl_{\ell=k}^{K(T)} \sumnl_{t=t_\ell}^{t_{\ell+1}-1} 
    \parens{\mathfrak{g}_k - R_t}
    \\
    & \overset{(*)}\le 
    \sumnl_{\ell=1}^{K(T)} \sumnl_{t=t_\ell}^{t_{\ell+1}-1} 
    \parens{\mathfrak{g}_k - R_t} 
    - 
    \sumnl_{\ell=k}^{K(T)} \sumnl_{t=t_\ell}^{t_{\ell+1}-1} 
    \parens{g^* - R_t}
    \\
    & \le
    \sumnl_{\ell=1}^{K(T)} \sumnl_{t=t_\ell}^{t_{\ell+1}-1} 
    \parens{\mathfrak{g}_k - R_t} 
    - 
    \sumnl_{\ell=k}^{K(T)} \sumnl_{t=t_k}^{T-1} \parens{
        \Delta^*(X_t) + \parens{p(X_t) - e_{S_t}} h^* + r(X_t) - R_t
    }
    \\
    & \le 
    \sumnl_{\ell=1}^{K(T)} \sumnl_{t=t_\ell}^{t_{\ell+1}-1} 
    \parens{\mathfrak{g}_k - R_t} 
    + \sp{h^*} - 
    \sumnl_{\ell=k}^{K(T)} \sumnl_{t=t_k}^{T-1} \parens{
        \parens{p(X_t) - e_{S_{t+1}}} h^* + r(X_t) - R_t
    }
    \\
    & \overset{(\dagger)}\le 
    \sumnl_{\ell=1}^{K(T)} \sumnl_{t=t_\ell}^{t_{\ell+1}-1} 
    \parens{\mathfrak{g}_k - R_t} 
    + \sp{h^*} 
    + (1 + \sp{h^*}) \sqrt{\tfrac 12 T \log\parens{\tfrac 1\delta}}
    \\
    & =: 
    B(T) 
    + \sp{h^*}
    + (1 + \sp{h^*}) \sqrt{\tfrac 12 T \log\parens{\tfrac 1\delta}}
    .
\end{align*}
In the above, $(*)$ holds with probability $1 - 4\delta$ uniformly on $k$ following \cref{lemma:bias-confidence-region} and $(\dagger)$ holds, also uniformly on $k$, with probability $1 - \delta$ by applying Azuma-Hoeffding's inequality (\cref{lemma:azuma}).
Therefore, with probability $1 - 5\delta$, for all $k$ and $t \in \set{t_k, \ldots, t_{k+1}-1}$, we have:
\begin{align*}
    \sqrt{
        \frac{16 c_0 \sum_{s' \in \Sc} \hat{p}_k(s'|S_t) d_{t_k}(s', S_t) \log \parens{\tfrac{SA T}\delta}}{N_{t_k}(X_t)}
    }
    & \le
    \frac{
        \sqrt{
            16 c_0 \log\parens{\tfrac{S A T}\delta} \sum_{s' \in \Sc} N_{t_k}(S_t, A_t, s') d_{t_k}(s', S_t)
        }
    }{
        N_{t_k}(X_t)
    }
    \\
    & \le 
    \frac{
        \sqrt{
            16 c_0 \log\parens{\tfrac{S A T}\delta} \sum_{s' \in \Sc} N_{t_k}(S_t \toto s') d_{t_k}(s', S_t)
        }
    }{
        N_{t_k}(X_t)
    }
    \\
    & \le 
    \tfrac{
        \sqrt{
            16 c_0 \log\parens{\tfrac{S A T}\delta} S \parens{
                3 c_0 + (1+c_0) \parens{1 + \sqrt{8T\log\parens{\tfrac 2\delta}}} + 2 B_0(t_k)
            }
        }
    }{
        N_{t_k}(X_t)
    }
    \\
    & \le 
    \tfrac{
        \sqrt{
            16 c_0 \log\parens{\tfrac{S A T}\delta} S \parens{
                (1+c_0) \parens{3 + 2 \sqrt{8T\log\parens{\tfrac 2\delta}}} + 2 B(T)
            }
        }
    }{
        N_{t_k}(X_t)
    }
    \\
    & \le 
    \tfrac{
        \sqrt{
            16 c_0 \log\parens{\tfrac{S A T}\delta} S \parens{
                (1+c_0) \parens{3 + 2 \sqrt{8T\log\parens{\tfrac 2\delta}} + 2 B(T)} 
            }
        }
    }{
        N_{t_k}(X_t)
    }
    .
\end{align*}
This bound will be enough.
We move on to $\text{A}_1$.
We have:
\begin{align*}
    \sqrt{\Var(\hat{p}_k(S_t), h'_k)}
    & \le 
    \sqrt{\abs{\Var(\hat{p}_k(S_t), h'_k) - \Var(p(X_t), h^*)}} + \sqrt{\Var(p(X_t), h^*)} 
    \\
    & \le 
    \sqrt{\abs{\Var(\hat{p}_k(S_t), h'_k) - \Var(\hat{p}_k(X_t), h^*)}} 
    \sqrt{\abs{\Var(\hat{p}_k(S_t), h^*) - \Var(p(X_t), h^*)}} 
    + \sqrt{\Var(p(X_t), h^*)} 
    \\
    & \overset{(*)}\le 
    \sqrt{
        8 c_0 \sumnl_{s'\in\Sc} \hat{p}_k (s'|S_t) d_k(s', S_t)
    }
    + 
    \sp{h^*} \sqrt{\norm{\hat{p_k}(S_t) - p_k(S_t)}_1} 
    + 
    \sqrt{\Var(p(X_t), h^*)}
    \\
    & \overset{(\dagger)}\le 
    \sqrt{
        8 c_0 \sumnl_{s'\in\Sc} \hat{p}_k (s'|S_t) d_k(s', S_t)
    } 
    +
    \sp{h^*} \parens{
        \frac{S \log\parens{\tfrac{SAT}\delta}}{N_{t_k}(X_t)}
    }^{\frac 14}
    +
    \sqrt{\Var(p(X_t), h^*)}
    \\
    & \le 
    \frac{\mathrm{A}_2}{\sqrt{2 N_{t_k}(X_t)}}
    +
    \sp{h^*} \parens{
        \frac{S \log\parens{\tfrac{SAT}\delta}}{N_{t_k}(X_t)}
    }^{\frac 14}
    +
    \sqrt{\Var(p(X_t), h^*)}
\end{align*}
where $(*)$ is obtained by applying \cref{lemma:variance-approximation} and $(\dagger)$ holds with probability $1 - \delta$ by applying Weissman's inequality, see \cref{lemma:uniform-weissman}.
All together, with probability $1 - 6\delta$, A is upper-bounded by:
\begin{equation*}
    \text{A}
    \le
    \sqrt{
        \frac{2\Var(p(X_t), h^*) \log\parens{\tfrac{SA T}\delta}}{N_{t_k}(X_t)}
    }
    + 2 \text{A}_2
    + 
    \underbrace{
        \sp{h^*} \sqrt{
            \frac{
                2 \log\parens{\tfrac{SAT}\delta} \sqrt{S \log\tfrac{SAT}\delta}
            }{
                N_{t_k}(X_t)\sqrt{N_{t_k}(X_t)}
            }
        }
        + 
        \frac{3c_0 \log\parens{\tfrac{SAT}\delta}}{N_{t_k}(X_t)}
    }_{\text{A}_3(k,t)}
    .
\end{equation*}

\STEP{2}
The number of visits $N_k(X_t)$ is lower-bounded by $\frac 12 N_t(X_t)$ when $N_k(X_t) \ge 1$ by doubling trick (DT).
By summing over $t$ and $k$, we find that with probability $1-6\delta$,
\begin{align*}
    \sum_k (3k)
    & \le
    S A c_0 +
    \sum_k \sum_{t=t_k}^{t_{k+1}-1} 
    \mathbf 1_{N_{t_k}(X_t) \ge 1}
    \sqrt{
        \frac{2\Var(p(X_t), h^*) \log\parens{\tfrac{SA T}\delta}}{N_{t_k}(X_t)}
    }
    +
    \sum_k \sum_{t=t_k}^{t_{k+1}-1} 
    \mathbf 1_{N_{t_k}(X_t) \ge 1}
    (2\text{A}_2(k, t) + \text{A}_3(k,t))
    \\
    \text{(DT)} & \le 
    SA c_0 +
    2 \sum_k \sum_{t=t_k}^{t_{k+1}-1} 
    \mathbf 1_{N_{t_k}(X_t) \ge 1}
    \sqrt{
        \frac{2\Var(p(X_t), h^*) \log\parens{\tfrac{SA T}\delta}}{N_{t}(X_t)}
    }
    +
    \sum_k \sum_{t=t_k}^{t_{k+1}-1} 
    \mathbf 1_{N_{t_k}(X_t) \ge 1}
    (2\text{A}_2(k, t) + \text{A}_3(k,t))
    \\
    & \le 
    S A c_0 +
    4 \sqrt{2 S A \sumnl_{t=0}^{T-1} \Var(p(X_t), h^*) \log\parens{\tfrac{SAT}\delta}}
    +
    \sum_k \sum_{t=t_k}^{t_{k+1}-1} 
    \mathbf 1_{N_{t_k}(X_t) \ge 1}
    (2\text{A}_2(k, t) + \text{A}_3(k,t))
\end{align*}
where the last inequality is obtained with computations that are similar to those detailed in the proof of \cref{lemma:empirical-bias-error}.
We recognize the variance that we will leave as is.
We finish the proof by bounding the lower order terms $\text{A}_2$ and $\text{A}_3$.

\STEP{3}
We start with $\text{A}_2$.
We have:
\begin{align*}
    \sum_k \sum_{t=t_k}^{t_{k+1}-1} 
    \mathbf 1_{N_{t_k}(X_t) \ge 1}
    \text{A}_2(k,t)
    & :=
    \sum_k \sum_{t=t_k}^{t_{k+1}-1} 
    \mathbf 1_{N_{t_k}(X_t) \ge 1}
    \tfrac{
        \sqrt{
            16 c_0 \log\parens{\tfrac{S A T}\delta} S  \parens{
                (1+c_0) \parens{3 + 2 \sqrt{8T\log\parens{\tfrac 2\delta}} + 2 B(T)} 
            }
        }
    }{N_{t_k}(X_t)}
    \\
    \text{(DT)} 
    & \le 
    2 \sqrt{
        16 c_0 S \log\parens{\tfrac{S A T}\delta} \parens{
            (1+c_0) \parens{3 + 2 \sqrt{8T\log\parens{\tfrac 2\delta}} + 2 B(T)} 
        }
    } ~
    S A \log(T)
    \\
    & \le 
    8 (1 + c_0) S^{\frac 32} A \log^{\frac 32}\parens{\tfrac{SAT}\delta} \parens{
        2 + 4 T^{\frac 14} \log^{\frac 14}\parens{\tfrac{SAT}\delta}
        + 
        \sqrt{2 B(T)}
    }
    .
\end{align*}

\STEP{4}
We are left with $\text{A}_3$.
We have:
\begin{align*}
    \sum_k \sum_{t=t_k}^{t_{k+1}-1}
    \mathbf 1_{N_{t_k}(X_t) \ge 1}
    \text{A}_3(k, t)
    & :=
    \sum_k \sum_{t=t_k}^{t_{k+1}-1}
    \mathbf 1_{N_{t_k}(X_t) \ge 1}
    \parens{
        \sp{h^*} \sqrt{
            \frac{
                2 \log\parens{\tfrac{SAT}\delta} \sqrt{S \log\tfrac{SAT}\delta}
            }{
                N_{t_k}(X_t)\sqrt{N_{t_k}(X_t)}
            }
        }
        + 
        \frac{3c_0 \log\parens{\tfrac{SAT}\delta}}{N_{t_k}(X_t)}
    }
    \\
    \text{(DT)} & \le 
    \sum_k \sum_{t=t_k}^{t_{k+1}-1}
    \mathbf 1_{N_{t_k}(X_t) \ge 1}
    \parens{
        \sp{h^*} \sqrt{
            \frac{
                2 \log\parens{\tfrac{SAT}\delta} \sqrt{S \log\tfrac{SAT}\delta}
            }{
                N_{t_k}(X_t)\sqrt{N_{t_k}(X_t)}
            }
        }
        + 
        \frac{3c_0 \log\parens{\tfrac{SAT}\delta}}{N_{t_k}(X_t)}
    }
    \\
    & \le
    C \sp{h^*} S^{\frac 54} A T^{\frac 14} \log^{\frac 34}\parens{\tfrac{SAT}\delta} + 6 c_0 S A \log\parens{\tfrac{SA T}\delta} 
    \\
    & = \OH\parens{\sp{h^*} S^{\frac 54} A T^{\frac 14} \log\parens{\tfrac{SAT}\delta}}
    .
\end{align*}
This concludes the proof.
\qed

\subsection{Proof of \cref{lemma:second-order-error}, second order error}

Recall that by \cref{lemma:bias-confidence-region}, with probability $1 - 4\delta$, $h^* \in \Hc_{t_k}$ for all $k$, hence $\sp{\mathfrak{h}_k - h^*} \le 2 c_0$ for all $k$ on the same event.
Therefore, with probability $1 - 4 \delta$,
\begin{align*}
    \sum_k (4k)
    & :=
    2 c_0 S A +
    \sum_k \sum_{t=t_k}^{t_{k+1}-1}
    \mathbf 1_{N_{t_k}(X_t) \ge 1}
    \parens{ \hat{p}_k(S_t) - p_k(S_t)} \parens{\mathfrak{h}_k - h^*}
    \\
    & =
    2 c_0 S A +
    \sum_k \sum_{t=t_k}^{t_{k+1}-1} \sum_{s' \in \Sc} 
    \mathbf 1_{N_{t_k}(X_t) \ge 1}
    (\hat{p}_k(s'|S_t) - p_k(s'|S_t)) (\mathfrak{h}_k - h^*(s'))
    \\
    & \overset{(*)}\le 
    2 c_0 S A +
    2 \sum_k \sum_{t=t_k}^{t_{k+1}-1} \sum_{s' \in \Sc} 
    \mathbf 1_{N_{t_k}(X_t) \ge 1}
    (\hat{p}_k(s'|S_t) - p_k(s'|S_t)) d_{t_k}(s', S_t)
    \\
    & \overset{(\dagger)}\le
    2 c_0 S A +
    2 \sum_k \sum_{t=t_k}^{t_{k+1}-1} \sum_{s' \in \Sc} 
    \mathbf 1_{N_{t_k}(X_t) \ge 1}
    \parens{
        d_k (s', S_t) \sqrt{
            \frac{2 \hat{p}_k(s'|S_t) \log\parens{\tfrac{S^2AT}\delta}}{N_{t_k}(X_t)}
        }
        +
        3 d_k(s'|S_t) \frac{\log\parens{\tfrac{S^2AT}\delta}}{N_{t_k}(X_t)}
    }
    \\
    & \le
    2 c_0 S A +
    2 \sum_k \sum_{t=t_k}^{t_{k+1}-1} \sum_{s' \in \Sc}
    \mathbf 1_{N_{t_k}(X_t) \ge 1}
    \parens{
        \sqrt{c_0} \sqrt{
            \frac{2 \hat{p}_k(s'|S_t) d_k (s', S_t) \log\parens{\tfrac{S^2AT}\delta}}{N_{t_k}(X_t)}
        }
        +
        \frac{3 c_0 \log\parens{\tfrac{S^2AT}\delta}}{N_{t_k}(X_t)}
    }
    \\
    & \le
    2 c_0 S A +
    4 \sum_k \sum_{t=t_k}^{t_{k+1}-1} \sum_{s' \in \Sc}
    \mathbf 1_{N_{t_k}(X_t) \ge 1}
    \parens{
        \sqrt{c_0} \sqrt{
            \frac{2 \hat{p}_k(s'|S_t) d_k (s', S_t) \log\parens{\tfrac{S^2AT}\delta}}{N_{t}(X_t)}
        }
        +
        \frac{3 c_0 \log\parens{\tfrac{S^2AT}\delta}}{N_{t}(X_t)}
    }
\end{align*}
where $(*)$ uses that $h^* \in \Hc_{t_k}$, and $(\dagger)$ is obtained by applying the empirical Bernstein's inequality, see \cref{lemma:uniform-bernstein}, to $\hat{p}_k(s'|S_t) - p_k(s'|S_t)$, and holds with probability $1 - \delta$.
The rightmost term's sum is upper-bounded by:
\begin{equation*}
    4 \sum_k \sum_{t=t_k}^{t_{k+1}-1} \sum_{s' \in \Sc}
    \frac{3c_0 \log\parens{\tfrac{S^2AT}\delta}}{N_t(X_t)}
    \le 12 S^2 A \log(T) \log\parens{\tfrac{S^2 A T}\delta}
    .
\end{equation*}
For the other term, follow the line of the proof of \cref{lemma:optimism-overshoot} (term $\text{A}_2$).
We have with probability $1 - 5\delta$ ($4\delta$ of which is by invoking \cref{lemma:bias-confidence-region}):
\begin{align*}
    \hat{p}_k (s'|S_t) d_k(s', S_t)
    & = \frac{
        N_{t_k}(S_t, A_t, s') \parens{(1+ c_0) \parens{1 + \sqrt{8t_k \log\parens{\tfrac 2\delta}}} + 2 B_0(t_k)}
    }{N_{t_k}(S_t \toto s') N_{t_k}(X_t)}
    \\
    & \le
    \frac{
        \parens{
            (1+c_0) \parens{3 + 2 \sqrt{8T\log\parens{\tfrac 2\delta}} + 2 B(T)} 
        }
    }{
        N_{t_k}(X_t)
    }
    .
\end{align*}
Therefore,
\begin{equation*}
    \sqrt{c_0}
    \sqrt{
        \frac{2 \hat{p_k}(s'|S_t) d_{t_k}(s', S_t) \log\parens{\tfrac{S^2AT}\delta}}
        {N_t(X_t)}
    }
    \le
    \frac{
        4(1+c_0) \sqrt{
            \parens{3 + 2 \sqrt{8T\log\parens{\tfrac 2\delta}} + 2 B(T)} 
            \log\parens{\tfrac{S^2AT}\delta}
        }
    }{N_t(X_t)}
    .
\end{equation*}
Summing over $k$, $t$, $s'$, with probability $1 - 6\delta$, we have:
\begin{equation*}
    \sum_k (4k) \le 
    \begin{Bmatrix}
        16 S^2 A(1+c_0) \log^{\frac 12}\parens{\tfrac{S^2AT}{\delta}} \parens{
            \sqrt{2 B(T)} + 2 \parens{8 T \log\parens{\tfrac 2\delta}}^{\frac 14}
        }
        \\
        + 32 S^2 A \parens{
            \log(T) \log\parens{\tfrac{S^2 A T}{\delta}} + (1+c_0) \log^{\frac 12} \parens{\tfrac{S^2AT}\delta}
        }
    \end{Bmatrix}
\end{equation*}
This concludes the proof.
\qed

\clearpage
\section{Details on experiments}

\subsection{River swim}

Experiments are run on $n$-states river-swim.
Such MDPs are, despite their size, known to be hard to learn.
They consists in $n$ states aligned in a straight line with two playable actions {\sc right} and \textcolor{red}{\sc left} whose dynamics are given in the figure below.
Rewards are Bernoulli and null everywhere excepted for $r(s_n, \textsc{right}) = 0.95$ and $r(s_0, \text{\sc\color{red}left}) = 0.05$.

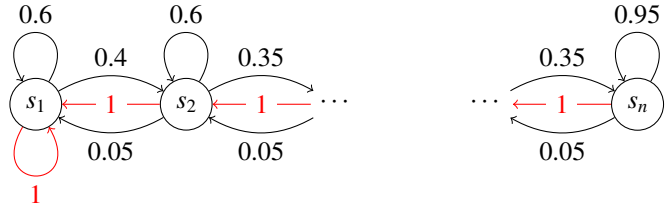
\begin{figure}[h]
\centering
\begin{tikzpicture}
    \node[circle, draw] (s0) at (0, 0) {$s_1$};
    \node[circle, draw] (s1) at (2, 0) {$s_2$};
    \node (dotl) at (4, 0) {$\cdots$};
    \node (dotr) at (6, 0) {$\cdots$};
    \node[circle, draw] (sn) at (8, 0) {$s_n$};

    \draw[->] (s0) to[in=90+30,out=90-30,min distance=1cm] node[above, midway] {$0.6$} (s0);
    \draw[->] (s0) to[bend left] node[midway, above] {$0.4$} (s1);

    \draw[->] (s1) to[in=90+30,out=90-30,min distance=1cm] node[above, midway] {$0.6$} (s1);
    \draw[->] (s1) to[bend left] node[midway, above] {$0.35$} (dotl);
    \draw[->] (s1) to[bend left] node[midway, below] {$0.05$} (s0);
    \draw[->] (dotl) to[bend left] node[midway, below] {$0.05$} (s1);
    
    \draw[->] (dotr) to[bend left] node[midway, above] {$0.35$} (sn);
    \draw[->] (sn) to[in=90+30,out=90-30,min distance=1cm] node[above, midway] {$0.95$} (sn);
    \draw[->] (sn) to[bend left] node[midway, below] {$0.05$} (dotr);

    \draw[->, color=red] (s0) to[in=270+30,out=270-30,min distance=1cm] node[below, midway] {$1$} (s0);
    \draw[->, color=red] (s1) to node[midway, fill=white] {$1$} (s0);
    \draw[->, color=red] (dotl) to node[midway, fill=white] {$1$} (s1);
    \draw[->, color=red] (sn) to node[midway, fill=white] {$1$} (dotr);
\end{tikzpicture}
\caption{
\label{figure:riverswim}
    The kernel of a $n$-state river-swim.
}
\end{figure}

\paragraph{$3$-state river-swim.} 
The gain is $g^* \approx 0.82$ and $h^* \approx (-4.28, -2.24, 0.4)$.

\paragraph{$5$-state river-swim.} 
The gain is $g^* \approx 0.82$ and $h^* \approx (-9.62, -7.58, -4.96, -2.27, 0.45)$.

\ifx\FALSE
\clearpage
Beyond this, this is mostly technical discussions.

\section{Paper list}\label{app:paperlist}

The list below is specific to papers about regret minimization in average-reward MDPs.
There are also many lines of work that add structure of the MDP/reward mechanism, but it seems irrelevant to our paper.
There are other papers \cite{burnetas_optimal_1997,pesquerel_imed-rl_2022} (and many others) that focus on the model dependent setting (regret guarantees scaling as $\log(T)$) and it seems a bit off-subject.

\begin{itemize}
    \item \cite{auer_logarithmic_2006} UCRL, first version
    
    \item \cite{auer_near-optimal_2009} The well-known UCRL2, regret lower-bound $\sqrt{D S A T}$, regret upper-bound $D S \sqrt{A T}$
    
    \item \cite{bartlett_regal_2009} REGAL, intractable, regret $\sp{h^*} S \sqrt{AT}$, their lower bound is wrong (see \cite{osband_lower_2016})

    \item \cite{filippi_optimism_2010} KL-UCRL, confidence region with KL-semiballs, regret upper-bound $D S \sqrt{A T}$

    \item \cite{osband_more_2013} LazyPSRL, bayesian, claim $D S \sqrt{AT}$ regret, but proof is wrong \cite{ouyang_learning_2017,osband_posterior_2016}

    \item \cite{gopalan_thompson_2015} TSMDP, bayesian, emphasis on parametrized MDP spaces, regret $DS \sqrt{A T}$

    \item \cite{osband_posterior_2016} Technical note showing LazyPSRL is wrong, further discussions

    \item \cite{osband_lower_2016} Discussion on regret lower-bound, points that REGAL's lower bound's proof is wrong

    \item \cite{ouyang_learning_2017} TSDE, bayesian, $D S \sqrt{A T}$ regret

    \item \cite{agrawal_optimistic_2023} Optimistic PSRL, $D S \sqrt{A T}$ regret, there is also a 2017 NeurIPS version of this paper with a smaller regret bound (so probably incorrect)
    
    \item \cite{talebi_variance-aware_2018} KL-UCRL, regret upper-bound improved to $S \sqrt{D A T}$

    \item \cite{fruit_efcient_2018} SCAL, span truncation, regret $S \sqrt{c A T}$ where $c$ is prior information on bias span

    \item \cite{fruit_improved_2020} UCRL2-B, UCRL2 with Bernstein confidence regions, regret $S \sqrt{D A T}$

    \item \cite{zhang_regret_2019} EBF, intractable, regret $\sqrt{c S A T}$ where $c$ is prior information on bias span, regret $\sqrt{D S A T}$ with initial phase learning the diameter

    \item \cite{bourel_tightening_2020} UCRL3, UCRL2 with cocktail inequalities, regret $S \sqrt{D A T}$, relies on a hacked version of EVI

    \item \cite{tossou_near-optimal_2019} UCRL-V, still preprint after 5 years, proof is impossible to understand, regret claimed $\sqrt{D \log(D) S A T}$

\end{itemize}
\fi

\clearpage
\section{Standard concentration inequalities}

\begin{lemma}[Azuma's inequality, \cite{azuma_weighted_1967}]
\label{lemma:azuma}
    Let $(U_t)_{t \ge 0}$ a martingale difference sequence such that $\sp{U_t} \le c$ a.s., i.e., there exists $a_t \in \R$ such that $a_t \le U_t \le a_t + c$ a.s.
    Then, for all $\delta > 0$,
    \begin{equation*}
        \Pr \parens{
            \sumnl_{t=0}^{T-1} U_t \ge c \sqrt{\tfrac 12 T \log\parens{\tfrac 1\delta}}
        }
        \le
        \delta
        .
    \end{equation*}
\end{lemma}

\begin{lemma}[Freedman's inequality, \cite{zhang_almost_2020}]
\label{lemma:additive-freedman}
    Let $(U_t)_{t \ge 0}$ a martingale difference sequence such that $\abs{U_t} \le c$ a.s., and denote its conditional variance $V_t := \E[U_t^2|\Fc_{t-1}]$.
    Then, for all $\delta > 0$,
    \begin{equation*}
        \Pr \parens{
            \exists T' \le T:
            \sumnl_{t=0}^{T'-1} U_t 
            \ge 
            \sqrt{
                2 \sumnl_{t=0}^{T'-1} V_t \log\parens{\tfrac T\delta}
            }
            + 4 c \log\parens{\tfrac T\delta}
        }
        \le
        \delta
        .
    \end{equation*}
\end{lemma}

\begin{lemma}[Time-uniform Azuma, \cite{bourel_tightening_2020}]
\label{lemma:time-uniform-azuma}
    Let $(U_t)$ a martingale difference sequence such that, for all $\lambda \in \R$, $\E[\exp(\lambda U_t)|U_1, \ldots, U_{t-1}] \le \exp(\frac{\lambda^2\sigma^2}2)$.
    Then:
    \begin{equation*}
        \forall \delta > 0,
        \quad
        \Pr\parens{
            \exists n \ge 1,
            \quad
            \parens{\sumnl_{k=1}^n U_k}^2 
            \ge
            n \sigma^2 \parens{1 + \tfrac 1n} \log\parens{\tfrac{\!\!\sqrt{1+n}}\delta}
        }
        \le 
        \delta
        .
    \end{equation*}
\end{lemma}

\begin{lemma}[Time-uniform Weissman]
\label{lemma:uniform-weissman}
    Let $q$ a distribution over $\set{1, \ldots, d}$.
    Let $(U_t)$ a sequence of i.i.d.~random variables of distribution $q$.
    Then:
    \begin{equation*}
        \forall \delta > 0,
        \quad
        \Pr\parens{
            \exists n \ge 1,
            \norm{
                \sumnl_{i=1}^n \parens{e_{U_i} - q}
            }_1^2
            \ge 
            n d \log\parens{\tfrac{2 \sqrt{1+n}}{\delta}}
        }
        \le 
        \delta
        .
    \end{equation*}
\end{lemma}
\begin{proof}
\newcommand{\dotProduct}[2]{\left\langle#1,#2\right\rangle}
    Remark that $\norm{\sum_{k=1}^n (e_{U_k} - q)}_1 = \max_{v \in \set{-1,1}^d} \sum_{k=1}^n \dotProduct{e_{U_k} - q}{v}$.
    Let $W_k^v := \dotProduct{e_{U_k} - q}v$.
    Remark that for each $v \in \set{-1,1}^d$, $(W_k^v)$ is a family of i.i.d.~random variables with $- \dotProduct qv \le W_k^v \le 1 - \dotProduct qv$, so $\E[\exp(\lambda W_k^v)] \le \exp(\tfrac{\lambda^2}8)$ by Hoeffding's Lemma.
    By \cref{lemma:time-uniform-azuma}, we have:
    \begin{align}
        \nonumber
        \Pr \parens{
            \exists n \ge 1,
            \norm{\sum_{k=1}^n (e_{U_k} - q)}_1
            \ge
            \!\!\sqrt{
                nd \log \parens{\tfrac{2\!\!\sqrt{1+n}}\delta}
            }
        }
        & =
        \Pr \parens{
            \exists v \in \set{-1,1}^d,
            \exists n,
            \sum_{k=1}^n W_k^v 
            \ge 
            \!\!\sqrt{
                nd \log \parens{\tfrac{2\!\!\sqrt{1+n}}\delta}
            }
        }
        \\ \nonumber
        & \le
        \sum_{v \in \set{-1,1}^d}
        \Pr \parens{
            \exists n,
            \sum_{k=1}^n W_k^v 
            \ge 
            \!\!\sqrt{
                nd \log \parens{\tfrac{2\!\!\sqrt{1+n}}\delta}
            }
        }
        \\ \nonumber
        & \le 
        \sum_{v \in \set{-1,1}^d}
        \Pr \parens{
            \exists n,
            \sum_{k=1}^n W_k^v 
            \ge 
            \!\!\sqrt{
                \tfrac 12 n \parens{1 + \tfrac 1n} \log \parens{\tfrac{\!\!\sqrt{1+n}}{2^{-d}\delta}}
            }
        }
        \\ \nonumber
        & \le 2^d \cdot 2^d \delta = \delta.
    \end{align}
    This concludes the proof.
\end{proof}

\begin{lemma}[Time-uniform Empirical Bernstein]
\label{lemma:uniform-bernstein}
    Let $(U_k)_{k \ge 1}$ a martingale difference sequence such that $\sp{U_n} \le c$ a.s., let $\hat{U}_n := \frac 1n \sum_{k=1}^n U_k$ the empirical mean and $\hat{V}_n := \frac 1n \sum_{k=1}^n (U_k - \hat{U}_n)^2$ the population variance.
    Then, 
    \begin{equation*}
        \forall \delta > 0,
        \forall T > 0,
        \quad
        \Pr \parens{
            \exists t \le T,
            \sumnl_{i=1}^t U_i
            \ge 
            \sqrt{2 t \hat{V}_t \log\parens{\tfrac {3T}\delta}}
            +
            3c \log\parens{\tfrac {3T}\delta}
        }
        \le \delta
        .
    \end{equation*}
\end{lemma}
\begin{proof}
    This is obtained with a union bound on the values of $n \le T$, then applying \cref{lemma:empirical-bernstein}.
\end{proof}

\begin{lemma}[Time-uniform Empirical Likelihoods, \cite{jonsson2020planning}]
\label{lemma:uniform-kl}
    Let $q$ a distribution on $\set{1, \ldots, d}$.
    Let $(U_t)$ a sequence of i.i.d.~random variables of distribution $q$.
    Then:
    \begin{equation*}
        \forall \delta > 0,
        \quad
        \Pr\parens{
            \exists n \ge 1,
            n \KL(\hat{q}_n||q) > \log\parens{\tfrac 1\delta} + (d-1) \log\parens{e \parens{1 + \tfrac n{d-1}}}
        }
        \le
        \delta
        .
    \end{equation*}
\end{lemma}

\begin{lemma}[Empirical Bernstein inequality, \cite{audibert_explorationexploitation_2009}]
\label{lemma:empirical-bernstein}
    Let $(U_k)_{k \ge 1}$ a martingale difference sequence such that $\sp{U_n} \le c$ a.s., let $\hat{U}_n := \frac 1n \sum_{k=1}^n U_k$ the empirical mean and $\hat{V}_n := \frac 1n \sum_{k=1}^n (U_k - \hat{U}_n)^2$ the population variance.
    Then, 
    \begin{equation*}
        \forall \delta > 0,
        \forall n \ge 1,
        \quad
        \Pr \parens{
            \sumnl_{k=1}^n U_k 
            \ge 
            \sqrt{2 n \hat{V}_n \log\parens{\tfrac 3\delta}}
            +
            3c \log\parens{\tfrac 3\delta}
        }
        \le \delta
        .
    \end{equation*}
\end{lemma}

\begin{lemma}[Bennett's inequality, \cite{audibert_explorationexploitation_2009}]
\label{lemma:bennett}
    Let $(U_t)_{t \ge 0}$ a martingale difference sequence such that $\abs{U_t} \le c$ a.s., and denote its conditional variance $V_t := \E[U_t^2|\Fc_{t-1}]$.
    Then, 
    \begin{equation*}
        \forall \delta > 0,
        \forall n \ge 1,
        \quad
        \Pr \parens{
            \exists k \le n,
            \sumnl_{i=1}^k U_i
            \ge 
            \sqrt{2 \sumnl_{i=1}^n V_i \log\parens{\tfrac 1\delta}}
            + \tfrac 13 c \log\parens{\tfrac 1\delta}
        }
        \le
        \delta
        .
    \end{equation*}
\end{lemma}

\begin{lemma}[Lemma~3 of \cite{zhang_sharper_2023}]
\label{lemma:upper-conditional-expectation}
    Let $(U_t)$ be a sequence of random variables such that $0 \le U_t \le c$ a.s., and let $\Fc_t := \sigma(U_0, U_1, \ldots, U_{t-1})$.
    Then:
    \begin{align*}
        & \forall \delta > 0,
        \quad
        \Pr\parens{
            \exists T \ge 0,
            \sumnl_{t=0}^{T-1} U_t
            \ge 
            3 \sumnl_{t=0}^{T-1} \E[U_t|\Fc_{t-1}]
            + c \log\parens{\tfrac 1\delta}
        }
        \le 
        \delta; 
        \\
        & \forall \delta > 0,
        \quad
        \Pr\parens{
            \exists T \ge 0,
            \sumnl_{t=0}^{T-1} \E[U_t|\Fc_{t-1}]
            \ge 
            3 \sumnl_{t=0}^{T-1} U_t
            + c \log\parens{\tfrac 1\delta}
        }
        \le 
        \delta.
    \end{align*}
\end{lemma}

\end{document}